\documentclass[pmlr]{jmlr}

\usepackage{longtable}

\usepackage{booktabs}
\usepackage[load-configurations=version-1]{siunitx} 
\usepackage[export]{adjustbox}%
\usepackage{algorithm}
\usepackage{algorithmic}
\usepackage{appendix}
\usepackage{color}
\usepackage{caption}
\usepackage{microtype}      
\usepackage{changepage}
\usepackage{graphicx,caption}
\usepackage{hyperref}
\usepackage{longtable}
\makeatletter
\def\set@curr@file#1{\def\@curr@file{#1}} 
\makeatother

\renewcommand \thepart{}
\renewcommand \partname{}

\theorembodyfont{\upshape}
\theoremheaderfont{\scshape}
\theorempostheader{:}
\theoremsep{\newline}

\jmlrvolume{149}
\jmlryear{2021}
\jmlrworkshop{Machine Learning for Healthcare}

\title[Power Constrained Bandits]{Power Constrained Bandits}

\author{\Name{Jiayu Yao}
       \Email{jiy328@g.harvard.edu}\\ 
       \addr SEAS, Harvard University\\
       Cambridge, MA, USA
       \AND
       \Name{Emma Brunskill}
       \Email{ebrun@cs.stanford.edu}\\ 
       \addr CS Department, Stanford University\\
       Stanford, CA, USA
       \AND
       \Name{Weiwei Pan}
       \Email{weiweipan@g.harvard.edu}\\ 
       \addr SEAS, Harvard University\\
       Cambridge, MA, USA
       \AND
       \Name{Susan Murphy}
       \Email{samurphy11@gmail.com }\\ 
       \addr SEAS, Harvard University\\
       Cambridge, MA, USA
       \AND
       \Name{Finale Doshi-Velez}
       \Email{finale@seas.harvard.edu}\\ 
       \addr SEAS, Harvard University\\
       Cambridge, MA, USA}

\begin{document}

\maketitle

\begin{abstract}
  Contextual bandits often provide simple and effective personalization in decision making problems, making them popular tools to deliver personalized interventions in mobile health as well as other health applications. However, when bandits are deployed in the context of a scientific study---e.g. a clinical trial to test if a mobile health intervention is effective---the aim is not only to personalize for an individual, but also to determine, with sufficient statistical power, whether or not the system's intervention is effective.  It is essential to assess the effectiveness of the intervention before broader deployment for better resource allocation.
  The two objectives are often deployed under different model assumptions, making it hard to determine how achieving the personalization and statistical power affect each other. In this work, we develop general meta-algorithms to modify existing algorithms such that sufficient power is guaranteed while still improving each user's well-being. We also demonstrate that our meta-algorithms are robust to various model mis-specifications possibly appearing in statistical studies, thus providing a valuable tool to study designers. 
\end{abstract}
\section{Introduction}
Mobile health applications are gaining more popularity due to easy access to smartphones and wearable devices. Mobile health applications can increase patients' information access, improve patients' communication with clinicians and help with self-monitoring. In mobile health applications, much of the initial research and development is done via clinical studies.  In these safety-critical applications, it is crucial to determine whether or not a treatment has an effect on the health of the patient (i.e. whether or not such an effect exists). This property is known as \emph{power} in the statistical literature: the probability of detecting an effect if it exists. A currently popular study design for assessing the treatment effect is the \emph{micro-randomized} trial \cite{liao2015micro, klasnja2015microrandomized}, in which an automated agent interacts in parallel with a number of individuals over a number of times.  At each interaction point, the intervention (or lack of intervention), is chosen according to some apriori determined probability.  This type of design allows the designer to observe the pattern of initial excitement/novelty effect followed by some disengagement that one would observe in a real deployment. The fact that each intervention is randomized also allows for rigorous statistical analysis to quantify the treatment effect.  However, it is also true that certain interventions may be more effective in certain contexts for certain people, and this knowledge may not be captured in apriori randomization probabilities.  Thus, another important goal in mobile health is to personalize these randomized probabilities to each user.

Contextual bandits provide an attractive tool for personalization in mobile health studies.  They represent a middle ground between basic multi-arm bandits, which ignore the intervention contexts, and full Markov Decision Processes (MDPs), which may be hard to learn given limited data.  
In this work, we are interested in meeting the dual objective in mobile health where we not only want to personalize actions for the users, but we also want to guarantee the ability to detect whether an intervention has an effect (if the effect exists) for the study designers.  Such situations arise frequently in mobile health studies: imagine a mobile app that will help patients manage their mental illness by delivering reminders to self-monitor their mental state.  In this case, not only may we want to personalize reminders, but we also want to measure the marginal effect of reminders on self-monitoring.  Quantifying these effects is often essential for downstream scientific analysis and development.

Currently, there exist algorithms that either have principled bounds on regret (e.g.~\cite{abbasiyadkori2011,agrawal2012thompson,Krishnamurthy2018}), which largely come from the Reinforcement Learning (RL) community, or aim to rigorously determine an effect (e.g. micro-randomized trials \cite{liao2015micro, klasnja2015microrandomized,Kramer2019}), which have been a focus in the experimental design community. Practical implementation of these algorithms  often results in tensions in mobile health applications: for regret minimization, one may make assumptions that are likely not true, but close enough to result in fast personalization. However, for treatment effect analysis, one must be able to make strong statistical claims in the face of a potentially non-stationary user---e.g. one who is initially excited by the novelty of a new app, and then disengages---as well as highly stochastic, hidden aspects of the environment---e.g. if the user has a deadline looming, or starts watching a new television series.  It is not obvious whether an algorithm that does a decent job of personalization under one set of assumptions would guarantee desired power under more general assumptions.

In this work, we \emph{both} rigorously guarantee that a trial will be sufficiently powered to provide inference about treatment effects (produce generalizable knowledge about a population of users) \emph{and} minimize regret (improve each user's well-being).  In minimizing regret, each user represents a different task; the task is performed separately on the entire sample of users.  Finally, mobile health studies and trials are expensive as each trial might be long. Thus not only must one be sufficiently powered, one must also leave open the option for post-hoc analyses via off-policy evaluation techniques; the latter implies that all action probabilities must be bounded away from 0 or 1.

\paragraph{Generalizable Insights for Machine Learning in the Context of Healthcare}
We provide important tools for study designers in mobile health to achieve good personalization and power at the same time.  Specifically, we introduce a novel meta-algorithm that 
can make simple adjustments to a variety of popular regret minimization algorithms 
such that sufficient power is guaranteed \emph{and} we get optimal regret per user with respect to an oracle that selects from a class of power-preserving policies. 
The wrapper algorithm only makes slight changes to the original algorithms and works by selectively sharing the information with them. Although our focus in this paper is on mobile health, our analysis and methods apply to many settings where personalization and power are equally prioritized.

\paragraph{Structure} In Section~\ref{sec:background}, we provide necessary technical tools for this work. In Section~\ref{sec:power} and~\ref{sec:regret}, we provide theoretical analyses of our methods. In Section~\ref{sec:experiments}, we describe all experiment details and demonstrate our approaches on a realistic mobile health simulator based on HeartSteps~\cite{liao2015micro}, a mobile health app designed to encourage users' physical activities (we focus on simulations because real studies are expensive and demonstration of power estimation requires running a large number of studies to compute the proportion of times one correctly detects a treatment effect).

\section{Related Work}
Micro-randomized trial (MRT), which can be used to determine whether a treatment effect exists in a time-varying environment, is a popular method in mobile health to inform the development of system intervetions~\cite{li2020micro,bell2020notifications,info:doi/10.2196/15033,liao2015micro}. For example,~\citet{li2020micro} used MRT to promote long term engagement of users in mobile health to help data collection.~\citet{bell2020notifications} used MRT to assess if in Drink Less, a behavior change app that helps users reduce alcohol consumption, sending a message at night increases behavioral engagement. However, in these studies, the randomized probabilities are fixed and the treatment plans are not personalized for users.

There is a body of works focusing on ways to quantify properties of various arms of a bandit. Some works estimate the means of all arms \citep{carpentier2011upper} while others focus on best-arm identification to find the best treatment with confidence \citep{audibert2010best}.  Best-arm identification has been applied to both stochastic and adversarial settings   \citep{abbasiyadkori2018, LattimoreSzepesvari2019}.  However, these algorithms typically personalize little if at all, and thus can result in high regret.  

Other works focus on minimizing regret without considering testing hypotheses related to treatment effectiveness.  While there exists a long history of optimizing bandits in RL \citep{abbasiyadkori2011,agrawal2012thompson}, perhaps most relevant are more recent works that can achieve optimal first order regret rates in highly stochastic, even adversarial settings \citep{LattimoreSzepesvari2019,Krishnamurthy2018,greenewald2017action}. Our approach also provides power guarantees in those challenging settings without significant increase in regret.

Finally, other works consider different simultaneous objectives. \citet{erraqabi2017} consider arm value estimation jointly with regret minimization. 
\citet{nie2017adaptively,deshpande, hadad2019confidence} consider how to accurately estimate the means or provide confidence intervals with data collected via adaptive sampling algorithms.  At a high level, most similar to this work is  that of~\citet{williamson2017bayesian,villar2015multi,degenne2019}. All of them assume multi-arm bandits while for regret minimization, we consider contextual bandits and for statistical analysis, we assume very general settings common in the mobile health where the environments can be non-stationary and highly stochastic.
\citeauthor{williamson2017bayesian,villar2015multi}  consider  the task of assigning treatments to $N$ individuals so as to minimize regret (i.e., maximize success rate).  They consider heuristic alternatives to improve power but not guarantee it while our work provides theoretical guarantees for a stated power.~\citeauthor{degenne2019} consider best arm identification with regret minimization with application to A/B testing. They studied one particular algorithm while we develop several meta-algorithms that allow us to adapt a broad range of existing algorithms.

To our knowledge, we are the first to consider the two following tasks: a sequential decision problem per user with the goal to minimize regret during the study while guaranteeing the power to detect a marginal (across the users) effect after the study is over. We guarantee the latter in a \emph{non-stationary} and \emph{stochastic} setting.

\section{Technical Preliminaries: Notation, Model, and Statistical Setting}
\label{sec:background}

In this section, we lay out the formal notations and assumptions for our work.  We then develop our methods in Sections~\ref{sec:power} and~\ref{sec:regret} before moving on to the results in the context of a mobile health simulator.  A critical point in all of the following is that it is quite common for study designers to consider two different sets of assumptions when designing their intervention algorithms and conducting treatment effect analyses.  When it comes to maximizing personalization for each user, designers may make stronger assumptions---e.g. use a model with fewer parameters---that allow for faster exploration and learning.  However, for the statistical analysis of the treatment effect, the study designers will want to ensure that their study is sufficiently powered even if the environment is stochastic, non-stationary, and future contexts can depend on past ones---all of which are common in mobile health and other applications.  Here and in Section~\ref{sec:power}, we will consider these very general settings for our power analyses.  In Section ~\ref{sec:regret}, we will consider a variety of additional assumptions that might be made by the regret minimization algorithms.  For example, Action-Centered Thompson Sampling~\citep{greenewald2017action} and Semi-Parametric Contextual Bandit~\citep{Krishnamurthy2018} assume that the treatment effect only depends on the current context  while our setting for power guarantees allows it to be a function of full history.  We also allow correlated reward noise  across time.

\paragraph{Basic Notation}
We consider a collection of histories $\{H_{nT}\}_{n=1}^N$ consisting of $N$ users, each with $T$ steps, where $H_{nt} = (C_{n0},A_{n0}, R_{n0},C_{n1},A_{n1}, R_{n1} \dots, C_{nt})$, $t\le T$ is the history of user $n$ up to time step $t$;
$C_{nt}$ denotes the context of user $n$ at time step $t$, $A_{nt}\in \{0,1\}$ denotes the binary action (no intervention and intervention), and $R_{nt}$ denotes the reward. The potential rewards are $(R_{nt}(0), R_{nt}(1))$.  The reward $R_{nt}$ is a composite of the potential rewards and the action, $A_{nt}$: $R_{nt}=R_{nt}(A_{nt})$.
For each user, a contextual bandit algorithm uses a policy $\pi_t$ which is a function constructed from the user's prior data $H_{n,t-1}, A_{n,t-1}, R_{n,t-1}$, in order to select action $A_{nt}$ based on the current context $C_{nt}$ (i.e. $P(A_{nt}=1)=\pi_{t}(C_{nt})$). We write the policy $\pi_{t}(C_{nt})$ as $\pi_{nt}$ for short in the following text. 

In this work, we will require policies to have action probabilities in some $[\pi_{\min},\pi_{\max}]$ bounded away from 0 and 1.  In mobile health where clinical trials are often expensive, this policy class is preferred---and often required---by scientists who wish to preserve their ability to perform unspecified secondary analyses \citep{philip2016, su2019} and causal inference analyses \citep{boruvka2018assessing}. 
We also run the algorithm for each user separately. Although it is possible to analyze the treatment effect with adaptively collected data~\cite{nie2017adaptively,deshpande, hadad2019confidence}, in mobile health, correctly accounting for treatment effect when combining data over users is nontrivial since users may enter the study at different times. Furthermore, some works have found that for online detection and prediction, user-specific algorithms work better than population-based algorithms \citep{dallery2013single,korinek2018adaptive,albers2017personalized}.

\paragraph{Preliminaries: Environment and Notation for Statistical Analyses}
In the contextual bandits literature, linear models are often preferred because they are well understood theoretically and easy to implement. However, in real life, linear models are often insufficient to model rewards accurately, and domain scientists wish to make as few assumptions as practically possible when testing for treatment effects.

In this work, we consider a semiparametric linear contextual bandit setting, which provides a middle ground between linear models and fully flexible models. In this setting, the reward function is decomposed into an action-dependent linear treatment effect term, which preserves nice theoretical properties for rigorous statistical analyses, and an action-independent marginal reward term, which constructs a reward model accurately. 

For the treatment effect, we assume it satisfies \begin{equation}
\mathbb{E}[R_{nt}(1)| H_{nt}]-\mathbb{E}[R_{nt}(0)| H_{nt}]=Z_t^{\intercal}(H_{nt})\delta_0, 
\label{eqn:true_reward}
\end{equation}
where 
$Z_t(H_{nt})$ is a set of features that are a known function of the history $H_{nt}$ and $\delta_0$ is a vector encodes the information of treatment effect.  Importantly, the feature vector $Z_t(H_{nt})$ is independent of the present action, $A_{nt}$, but may depend on prior actions.  We assume that an expert defines what features of a history may be important for the reward but make \emph{no} assumptions about how the history itself evolves.  
We assume the histories $\{H_{nt}\}_{n=1}^N$ are independent and identically distributed as we run algorithms on each user separately. 
However, there  may be dependencies across time within a specific user. Finally, we assume that 
the variance of potential rewards is finite ($\mathrm{Var}[R_{nt}(a)|H_{nt}]< \infty$ for $a\in\{0,1\}$ and $t=1,\dots, T$). We denote the marginal reward over treatments, $\mathbb{E}[R_{nt}| H_{nt}]$,  by $\gamma_t(H_{nt})$, which can be a complex non-linear function of the history $H_{nt}$. We discuss how to approximate $\gamma_t(H_{nt})$ later. 
 In the following text, we  write the features
$Z_{t}(H_{nt})$ as $Z_{nt}$ and the marginal reward $\gamma_{t}(H_{nt})$ as $\gamma_{nt}$ for short. In fact, the reward function can be written as,
 \[\mathbb{E}[R_{nt}|A_{nt}, H_{nt}]=\gamma_{nt}+(A_{nt}-\pi_{nt})Z^\intercal_{nt}\delta_0 \quad \text{ (Appendix~\ref{apdx:pf_reward_fcn})}.\]

\paragraph{Preliminaries: Hypothesis Testing.}
In statistics, hypothesis testing is the act of testing an assumption about the population based on observations collected from an experiment. In this work, we are interested in testing if there exists a treatment effect. Our goal is to test between the null hypothesis $H_0$, which proposes there is no treatment effect 
($H_0: \delta_0=0$), and the alternate hypothesis $H_1$, which proposes there is a treatment effect ($H_1: \delta_0\neq 0$). Hypothesis testing is often analyzed in terms of \emph{Type 1 error} and \emph{power}. The Type 1 error is the probability of finding a treatment effect when there is no effect ($P(\text{Reject } H_0|\ H_0 \text{ is True})$), and the power is the probability of detecting a treatment effect when an effect exists ($P(\text{Reject } H_0|\ H_1 \text{ is True})$). Prior to data collection, power analysis is used to compute the number of samples needed to achieve a particular level of power (if an effect exists).

\paragraph{Preliminaries: Test Statistic.}
 To identify if we can reject the null hypothesis, we need to construct a test statistic that allows us to compare the sample data with what is expected under the null hypothesis.
Drawing on one used in multiple micro-randomized trials in mobile health~\citep{liao2015micro,boruvka2018assessing,klasnja2019,Bidargaddi2018}, we construct a test statistic that requires minimal assumptions to guarantee the desired Type 1 error and the desired power.  The construction  assumes the treatment effect model in Equation~\ref{eqn:true_reward}. Next we construct a \lq\lq working model\rq\rq\ for the marginal reward $\gamma_{nt}$:
\begin{equation}
\mathbb{E}[R_{nt}|H_{nt}]=\gamma_{nt}=B_{nt}^\intercal\gamma_0, 
\label{eqn:working_model}
\end{equation}
for some vector $\gamma_0$ and
$B_{nt}$, which is a feature vector provided by experts  constructed from the history $H_{nt}$  and is different from $Z_{nt}$. 

Let $\theta=\begin{bmatrix}\gamma\\\delta\end{bmatrix}$. Our test statistics $\hat\theta=\begin{bmatrix}\hat\gamma\\\hat\delta\end{bmatrix}$  will minimize
\begin{equation}
\label{eqn:loss_delta}
 L(\theta) = \sum_{n=1}^N\sum_{t=1}^T\frac{\left( R_{nt}-X^\intercal_{nt}\theta\right)^2}{\pi_{nt}(1-\pi_{nt})}
\end{equation}
where $X_{nt} = \begin{bmatrix}B_{nt}\\(A_{nt}-\pi_{nt})Z_{nt}\end{bmatrix}\in\mathcal{R}^{(p+q)\times 1}$, and $p,q$ are the dimensions of $B_{nt},Z_{nt}$ respectively. 
Setting $\partial L(\theta)/\partial \theta = 0$ gives the solution for $\hat\theta$:
\begin{align}\small
\begin{split}
\hat\theta = &\left(\frac{1}{N}\sum_{n=1}^N\sum_{t=1}^T\frac{X_{nt}X_{nt}^\intercal}{\pi_{nt}(1-\pi_{nt})}\right)^{-1} \left(\frac{1}{N}\sum_{n=1}^N\sum_{t=1}^T\frac{R_{nt}X_{nt}}{\pi_{nt}(1-\pi_{nt})}\right)
\end{split}\label{eqn:theta_hat}
\end{align}

Since we are mainly interested in detecting the treatment effect, we focus on properties of $\hat\delta$, which is the estimator of $\delta_0$.  The loss function in Equation~\ref{eqn:loss_delta} centers the action by $A_{nt}-\pi_{nt}$. This results in an unbiased estimator of $\delta_0$ even when the model in Equation~\ref{eqn:working_model} is false~\citep{boruvka2018assessing}. The asymptotic distribution of $\sqrt{N}(\hat\delta-\delta_0)$ is as follows:
\begin{theorem}
Under the assumptions in this section, and the assumption that matrices $\mathbb{E}[\sum_{t=1}^TZ_{nt}Z_{nt}^\intercal]$, $\mathbb{E}\left[\sum_{t=1}^T\frac{B_{nt}B_{nt}^\intercal}{\pi_{nt}(1-\pi_{nt})}\right]$ are invertible, the distribution of $\sqrt{N}(\hat\delta-\delta_0)$ converges, as $N$ increases, to  a normal distribution with $0$ mean and covariance $\Sigma_\delta = QW^{-1}Q$, where  $Q=\mathbb{E}\left[\sum_{t=1}^TZ_{nt}Z_{nt}^\intercal\right]^{-1}$, and

\begin{align*}
W=\mathbb{E}\bigg[&\sum_{t=1}^T\frac{(R_{nt}-X_{nt}^\intercal\theta^*)(A_{nt}-\pi_{nt})Z_{nt}}{\pi_{nt}(1-\pi_{nt})}\times\sum_{t=1}^T\frac{(R_{nt}-X_{nt}^\intercal\theta^*)(A_{nt}-\pi_{nt})Z_{nt}^\intercal}{\pi_{nt}(1-\pi_{nt})}\bigg],
\end{align*} 
where $\small{ \gamma^*=\mathbb{E}\left[\sum_{t=1}^T\frac{B_{nt}B_{nt}^\intercal}{\pi_{nt}(1-\pi_{nt})}\right]^{-1}\mathbb{E}\left[\sum_{t=1}^T\frac{B_{nt}R_{nt}}{\pi_{nt}(1-\pi_{nt})}\right]}$ and $\small{\theta^*=\begin{bmatrix}\delta_0\\\gamma^*\end{bmatrix}}$.
\label{thm:theorem1}
\end{theorem}
\begin{proof}
The proof is a minor adaptation of \citet{boruvka2018assessing} (Appendix Section \ref{pf:apdx_thm1}).
\end{proof}

The covariance matrix, $\Sigma_{\delta}$, can be estimated from the data using standard methods. Denote the estimator of $\Sigma_{\delta}$ by $\hat\Sigma_\delta$ (See  Section~\ref{subsec:settings}, for  $\hat\Sigma_{\delta}$). Under the null hypothesis $\delta_0=0$, the statistic $N\hat\delta^\intercal{\hat\Sigma_\delta}^{-1}\hat\delta$ asymptotically follows a $\chi^2_p$ where $p$ is the number of parameters in $\delta_0$. Under the alternate hypothesis $\delta_0=\delta$, $N\hat\delta^\intercal{\hat\Sigma_\delta}^{-1}\hat\delta$ has an asymptotic 
non-central $\chi^2_p$ distribution with degrees of freedom $p$ and non-centrality parameter $c_N =N\delta{\Sigma_{\delta}}^{-1}\delta $. The Type 1 error is  the percentage of times that the null hypothesis is incorrectly rejected; power is the percentage of times that the null hypothesis is correctly rejected.

\section{Power Constrained Bandits}
\label{sec:power}

In clinical studies of mobile health where the number of trials is often limited, if the amount of exploration, which is controlled by the intervention probability $\pi_{nt}$ is insufficient, we won't be able to determine the treatment effect. That is, to guarantee sufficient power, each treatment option needs to be tried at least some minimal number of times.
In this section, we develop a set of constraints on the randomized probability of the intervention $\pi_{nt}$ which guarantees sufficient power.

We start by proving the intuition that sufficient power requires a intervention probability $\pi_{nt}$ that ensures each option is tried enough times: for a fixed randomization probability $\pi_{nt}=\pi\in(0,1)$, for all $n,\ t$,  there exists a $\pi_{\min}$ and a $\pi_{\max}$ ($\pi_{\min} \leq \pi_{\max}$) such that when $\pi$ is $\pi_{\min}$ or $\pi_{\max}$, the experiment is sufficiently powered.  Conceptually, if the intervention probability $\pi_{nt}$ is too close to 0 or 1, then we will not see one of the alternatives often enough to detect an effect of the intervention.  

\begin{theorem}
Let $\epsilon_{nt}=R_{nt}-X^\intercal_{nt} \theta^*$ where $\theta^*$ is defined in Theorem~\ref{thm:theorem1}.  Assume that the working model in Equation~\ref{eqn:working_model} is correct.  Further assume that  $\mathbb{E}[\epsilon_{nt}|A_{nt},H_{nt}]=0$ and $Var(\epsilon_{nt}|H_{nt},A_{nt})=\sigma^2$.
Let $\alpha_0$ be the desired Type 1 error and $1-\beta_0$ be the desired power.  Set
\begin{align*}
\pi_{\text{min}}=\frac{1-\sqrt{1-4\triangle}}{2},\
\pi_{\text{max}}=\frac{1+\sqrt{1-4\triangle}}{2},\
\triangle=\frac{\sigma^2c_{\beta_0}}{N\delta^\intercal_0 \mathbb{E}\bigg[\sum_{t=1}^TZ_{nt}Z_{nt}^\intercal\bigg]\delta_0}.
\end{align*}
We choose $c_{\beta_0}$ such that $1-\Phi_{p;c_{\beta_0}}(\Phi_{p}^{-1}(1-\alpha_0))=\beta_0$, where $\Phi_{p;c_{\beta_0}}$ denotes the cdf of a non-central $\chi^2$ distribution with d.f. $p$ and non-central parameter $c_{\beta_0}$, and $\Phi_{p}^{-1}$ denotes the inverse cdf of a $\chi^2$ distribution with d.f. $p$.
For a given trial with  $N$ subjects each over $T$ time units, if the randomization probability is fixed at $\pi_{nt}=\pi_{\text{min}}$ or $\pi_{\text{max}}$, 
the  resulting Type 1 error converges to $\alpha_0$ as $N \xrightarrow{} \infty$ and the resulting power converges to $1-\beta_0$ as $N \xrightarrow{} \infty$.
\label{thm:theorem2}
\end{theorem}
\begin{proof} (Sketch) The rejection region for $H_0:\delta_0=0$ is
$
\{N\hat\delta^\intercal{\hat\Sigma_\delta}^{-1}\hat\delta>\Phi^{-1}_{p}(1-\alpha_0)\},
$
which results in the 
Type 1 error of
\begin{equation*}
    \alpha_0 = \Phi_{p}(\Phi_{p}^{-1}(1-\alpha_0)),
\end{equation*}
and the power of
\begin{equation}
    1-\beta_0 = 1-\Phi_{p;c_N}(\Phi_{p}^{-1}(1-\alpha_0))
\label{eqn:power}
\end{equation}
where $c_N = N\delta^\intercal_0\Sigma^{-1}_{\delta}\delta_0$.  The formula for $\Sigma_{\delta}$ is in Theorem~\ref{thm:theorem1}, thus we only need to solve for $\pi_{\min},\pi_{\max}$ when we substitute the expression for $\Sigma_{\delta}$ in $c_N$ (full analysis in Appendix \ref{pf:apdx_thm2}).
\end{proof} 

Violations of assumptions listed in Theorem~\ref{thm:theorem2} have an effect on the robustness of the power guarantee. For example, 
although for the test statistic defined in Theorem~\ref{thm:theorem1} to possess the desired Type 1 error, we do not need the working model in Equation~\ref{eqn:working_model} to be correct,  the choice of $B_{nt}$ can have an effect on the robustness of power guarantee (Appendix \ref{pf:apdx_thm4}). Calculations in Theorem~\ref{thm:theorem2} also requires a correct treatment effect model.
In some cases, such as in the work of \citet{liao2015micro}, $Z_{nt}$ may be available in advance of the study.  In other cases, the study designer will need to specify a set of plausible models and determining the power for some fixed randomization probability will require finding the worst-case $\mathbb{E}[\sum_t Z_{nt} Z^\intercal_{nt}]$. If
the average treatment effect, $\frac{1}{T}\mathbb{E}[\sum_{t=1}^T Z_{nt}^\intercal \delta_0]$, is overestimated, it will result in lower power ($\triangle$ increases) because more exploration is needed.
Additionally, if the noise variance, $\sigma^2$, is underestimated, the resulting power will also be lower (since $\triangle$ increases) because less exploration is required in a less noisy environment.  
In Section~\ref{subsec:results}, we show that our power guarantees are robust to these possible violations: There is still a reasonable proportion of times that the treatment effect (if it exists) can be detected. 

Next, we prove that as long as each randomization probability $\pi_{nt} \in [\pi_{\min}, \pi_{\max}]$, the power constraint will be met.  Our proof holds for \emph{any} selection strategy for $\pi_{nt}$, \emph{including} ones where the policy is adversarially chosen to minimize power based on the subject's history $H_{nt}$.  Having the condition across myriad ways of choosing $\pi_{nt}$ is essential to guaranteeing power for any contextual bandit algorithm that can be made to produce clipped probabilities.

\begin{theorem}
With the same set of assumptions in Theorem~\ref{thm:theorem2}, given $\pi_{\text{min}},\pi_{\text{max}}$ we solved for above, if for all $n$ and all $t$ we have that $\pi_{nt}\in[\pi_{\text{min}},\pi_{\text{max}}]$, then the resulting power will converge to a value no smaller than $1-\beta_0$ as $N\xrightarrow{} \infty$.
\label{thm:theorem3}
\end{theorem}
\begin{proof} (Sketch)
The right hand side of Equation ~\ref{eqn:power} is monotonically increasing with respect to $c_N$. The resulting power will be no smaller than $1-\beta_0$ as long as $c_N\geq c_{\beta_0}$. This holds when $\pi_{nt}\in[\pi_{\text{min}},\pi_{\text{max}}]$. The full proof is in Appendix~\ref{pf:apdx_thm3}.
\end{proof}

\section{Regret with Power-Constrained Bandits}
\label{sec:regret}

When running a mobile health study, the study designer may already know that certain regret minimization algorithms with certain assumptions will work well in their particular domain.  These assumptions often come from knowledge about the domain and the designer's experience from prior studies.  
Our contribution in this section is to provide very general ways for study designers to take their preferred contextual bandit algorithm and adapt it such that (a) one can perform sufficiently powered analyses of the treatment effect under very general assumptions and (b) the contextual bandit algorithm retains its original regret guarantees (among the set of algorithms that give sufficient powers).  

Said more formally, in Section~\ref{sec:power}, we provided an algorithm-agnostic way to guarantee a study's power constraints were met, in the very general setting described in Section~\ref{sec:background}. In practice, to facilitate personalized treatment design, developers often use bandit algorithms that make more specific environment assumptions than power analyses do.  Now, we consider, with the power constraints, how well we can do with respect to regret \emph{under the bandit algorithm's environment assumptions}. Because we  are constrained to policies that guarantee a certain amount of exploration, our goal is to preserve regret rates, now with respect to a clipped oracle, i.e. an oracle whose action probabilities $\pi_{nt}$ lie within $\pi_{\min}$ and $\pi_{\max}$. 
We first present some specific algorithms in which we can preserve regret rates with respect to a clipped oracle by simply clipping the action selection probability to lie in $[\pi_{\min}, \pi_{\max}]$. We then present very general wrapper algorithms with formal analyses that allow us to adapt a large class of existing algorithms while preserving regret rates.

We formally define the regret as,
\begin{equation}
\label{eqn:regret_rate}
\text{reg} =\mathbb{E}\left[\sum_{t=1}^T\max_{\substack{a, \pi^*}} \mathbb{E}[R_{nt}|A_{nt}=a,H_{nt}]\right]-\mathbb{E}\left[\sum_{t=1}^T R_{nt}\right]
\end{equation}
where $a\in\{0,1\},\pi^*\in[0,1]$,
and the regret with respect to clipped oracle as 
\begin{equation}
\label{eqn:regret_rate_wrt_clipped_oracle}
\text{reg}_\text{c} =\mathbb{E}\left[\sum_{t=1}^T\max_{\substack{a, \pi^*}} \mathbb{E}[R_{nt}|A_{nt}=a,H_{nt}]\right]-\mathbb{E}\left[\sum_{t=1}^T R_{nt}\right]
\end{equation}
where $a\in\{0,1\},\pi^*\in[\pi_{\min}, \pi_{\max}]$.
\subsection{Regret Rates of Specific Algorithms with Probability Clipping}
\label{subsec:regret_wt_clipping}
Before getting into the very general case (Section~\ref{subsec:regret_wrapper}), we note that in some cases, one can simply clip action probabilities and still achieve optimal regret with respect to a clipped oracle.  For example, Action-Centered Thompson Sampling (ACTS~\citep{greenewald2017action}) and Semi-Parametric Contextual Bandits (BOSE~\citep{Krishnamurthy2018}) have optimal first order regret with respect to a clipped oracle if one clips probabilities.  Both algorithms perform in non-stationary, adversarial settings where the features and rewards are a function of the current context $C_{nt}$ (unlike our full history $H_{nt}$). BOSE further assumes the noise term is action independent. Neither algorithms consider power; using our probabilities will result in optimal regret and satisfy the required power guarantees at the same time.

Other cases are more subtle but still work: for example, we can prove that clipped Linear Stochastic Bandits (OFUL) preserves regret with respect to a clipped oracle (the proof involves ensuring optimism under the constraint, see Appendix~\ref{subsec:apdx_alg_regret_bound}).

\subsection{Regret Rate of General Power-Preserving Wrapper Algorithms}
\label{subsec:regret_wrapper}
The above cases require a case-by-case analysis to determine if clipping probabilities would preserve regret rates (with respect to a clipped oracle).  Now we describe how to adapt a wide variety of bandit algorithms in a way that (a) guarantees sufficient power and (b) preserves regret rates with respect to a clipped oracle. 

We first present the main meta-algorithm, data dropping, where information is selectively shared with the algorithm. 
The key to guaranteeing good regret with this wrapper for a broad range of input algorithms $\mathcal{A}$ is to ensure that the input algorithm $\mathcal{A}$ only sees samples that match the data it would observe if \emph{it} was making all decisions.
 Denote the action probability given by a bandit algorithm $\mathcal{A}$ as $\pi_\mathcal{A}(C_{nt})$.  The algorithm works as follows:
\paragraph{Meta-Algorithm: Selective Data Dropping.}
\begin{enumerate}
    \item Produce $\pi_\mathcal{A}(C_{nt})$ as before. If sampling $A_{nt} \sim \texttt{Bern}(\pi_\mathcal{A}(C_{nt}))$ would have produced the same action as sampling $A'_{nt} \sim \texttt{Bern}(\texttt{clip}(\pi_\mathcal{A}(C_{nt})))$ (see detailed algorithm description in Appendix~\ref{sec:apdx_data_dropping} as to how to do this efficiently), then perform $A_{nt}$; else perform $A'_{nt}$.
    \item The algorithm $\mathcal{A}$ stores the tuple $C_{nt} , A_{nt}, R_{nt}$ \emph{if} $A_{nt}$ was performed; else it stores nothing from that interaction.
    \item The scientist \emph{always} stores the tuple $C_{nt} , A'_{nt}, R_{nt}$
\end{enumerate}
\begin{theorem}
\label{thm:data_dropping}
Given input $\pi_{\min},\pi_{\max}$ and a contextual bandit algorithm $\mathcal{A}$. Assume algorithm $\mathcal{A}$ has a regret bound $\mathcal{R}(T)$ and that one of the following holds for the setting $\mathcal{B}$: (1) under $\mathcal{B}$ the data generating process for each context is independent of history, or (2) under $\mathcal{B}$ the context depends on the history, and the bound $\mathcal{R}$ for algorithm $\mathcal{A}$ is robust to an adversarial choice of context.

Then our wrapper algorithm will (1) return a dataset that  satisfies the desired power constraints under the data generation process of Section ~\ref{sec:background} and (2) has expected regret no larger than $ \mathcal{R}(\pi_{\text{max}}T) + (1-\pi_{\text{max}})T$ if assumptions of $\mathcal{B}$ are satisfied in the true environment.
\end{theorem}
\begin{proof}(Sketch)
The key to guaranteeing good regret with this wrapper for a broad range of input algorithms $\mathcal{A}$ is in deciding what information we share with the algorithm. The context-action-reward tuple from that action is only shared with the input algorithm $\mathcal{A}$ if $\mathcal{A}$ would have also made that same decision. This process ensures that the input algorithm $\mathcal{A}$ only
sees samples that match the data it would observe if it was making all decisions. Hence, the environment $\Omega$ remains closed when data are dropped and the expected regret rate is no worse than $R(\pi_{\max}T)$ with respect to a clipped oracle. The full proof is in Appendix Section~\ref{sec:apdx_data_dropping}.
\end{proof}
The data dropping strategy can be applied to two general classes of algorithms described in Theorem~\ref{thm:data_dropping} (e.g. OFUL belongs to setting (1), ACTS and BOSE belong to setting(2)).  It is simple to implement and gives good regret rates (Section~\ref{subsec:results}). 
In addition to data dropping, there are alternative ways to adapt algorithms and still preserve the regret rates with respect to a clipped oracle. Next, we present another simple meta-algorithm, action flipping, which encourages exploration by taking the action output by any algorithm and flipping it with some probability. While action flipping has nice asymptotic properties, in Section~\ref{subsec:results}, we will see that it can result in extra power and  high regret due to over-exploration and extra stochasticity of the agent's perceived environment. 
\paragraph{Meta-Algorithm: Action-Flipping.} 
The pseudocode is given as follows:
\begin{enumerate} 
    \item Given current context $C_{nt}$, algorithm $\mathcal{A}$ produces action probabilities $\pi_\mathcal{A}(C_{nt})$
    \item  Sample $A_{nt} \sim \texttt{Bern}(\pi_\mathcal{A}(C_{nt}))$.
    \item If $A_{nt}=1$, sample  $A_{nt}^\prime \sim \texttt{Bern}(\pi_{\max})$.  If $A_{nt}=0$, sample  $A_{nt}^\prime \sim \texttt{Bern}(\pi_{\min})$.
    \item We perform $A'_{nt}$ and receive reward $R_{nt}$.
    \item The algorithm $\mathcal{A}$ stores the tuple $C_{nt} , A_{nt}, R_{nt}$.  (Note that if $A_{nt}$ and $A'_{nt}$ are different, then, unbeknownst to the algorithm $\mathcal{A}$, a different action was actually performed.)
    \item The scientist stores the tuple $C_{nt} , A'_{nt} , R_{nt}$ for their analysis.
\end{enumerate}
Let $A'_{nt} = G_(A_{nt})$ denote the stochastic transformation by which the wrapper above transforms the action $A_{nt}$ from algorithm $\mathcal{A}$ to the new action $A'_{nt}$.  
Suppose that the input algorithm $\mathcal{A}$ had an regret rate $\mathcal{R}(T)$ for a set of environments $\Omega$ (e.g. assumptions on distributions of $\{C_{nt}, R_{nt}(0), R_{nt}(1)\}_{t=1}^T $).  We give conditions under which the altered algorithm $\mathcal{A}$, as described above, will achieve the same rate against a clipped oracle: 
\begin{theorem}
\label{thm:wrapper}
Given $\pi_{\min},\ \pi_{\max}$ and a contextual bandit algorithm $\mathcal{A}$, 
assume that algorithm $\mathcal{A}$ has expected regret  $\mathcal{R}(T)$ for any environment in  $\Omega$, 
with respect to an oracle $\mathcal{O}$.
If there exists an environment  in
$\Omega$ such that the potential rewards, $R'_{nt}(a) = R_{nt}(G(a))$, for $a\in\{0,1\}$,
then the wrapper algorithm will  (1) return a data set that  satisfies the desired power constraints 
and (2) have expected regret no larger than $\mathcal{R}(T)$ with respect to a clipped oracle $\mathcal{O}'$. 
\end{theorem}
\begin{proof} (Sketch)
Our wrapper algorithm makes the input algorithm $\mathcal{A}$ believe that the environment is more stochastic than it is.  If algorithm $\mathcal{A}$ achieves some rate in this more stochastic environment, then it will be optimal with respect to the clipped oracle. Full proof in Appendix Section~\ref{sec:apdx_action_flip}.
\end{proof}
There exists many environments $\Omega$ which are closed under the reward transformation above, including ~\citet{abbasiyadkori2011,agrawal2012thompson,langford2007epoch}.
In Appendix~\ref{sec:apdx_action_flip}, we describe a large number of settings, including stochastic contextual bandits and adversarial contextual bandits, in which this wrapper could be used.

\section{Experiments \& Results}
\label{sec:experiments}
In clinical studies, power analyses are often conducted before the data collection process to help the scientists to determine the smallest sample size that is needed in order to detect a certain level of treatment effect. To estimate the power accurately, multiple runs of a study are needed to compute the proportion of times a treatment effect is detected (if it exists). Collecting preliminary data for this process would often be prohibitively expensive and thus simulations are often used for power analyses.
In this work, we demonstrate the properties of our power-constrained bandits on a realistic mobile health simulator\footnote{Our code is public at \url{https://github.com/dtak/power-constrained-bandits-public}.}.

\subsection{Realistic Mobile Health Simulator} To demonstrate our approaches on real life tasks, we utilize a mobile health simulator that was introduced in \cite{liao2015micro} and was  
motivated by the HeartSteps mobile health application. HeartSteps aims to encourage physical activities in users by sending suggestions for a walk tailored to the user's current context, such as user's location and current events based on the user's calendar. 
The suggestions will be sent during  morning commute, mid-day, 
mid-afternoon, evening commute, and post-dinner times, which encourages the user to take a walk in the next few hours. Our mobile health simulator mimics the data generating process of HeartSteps. In this task, we aim to detect a certain amount of treatment effect (how much more physically active users become on average) as well as increase physical activity for each user as much as possible.

 In real studies, the number of users and study length will be provided by domain experts. In our simulations, we chose a sample size that is close to real life and is large enough so that the power constraint will be met under maximal exploration (when $\pi$=0.5).  Specifically, we collected $N=20$ users and
 each simulated user $n$ participates for $90$ days. The action $A_{nt}=1$ represents a message is delivered while $A_{nt}=0$ represents not, and the reward $R_{nt}$ represents the square root of the step count at day $t$. The marginal reward, $\gamma_{nt}$, decreases linearly over time as people engage more at the start of the study. The feature vector $Z_{nt}$ is created by experts such that the treatment effect $Z_{nt}^\intercal \delta_0$ starts small at day $0$, as people have not developed the habits of increasing physical activity, then peaks at day $45$, and decays to $0$ at day $90$ as people disengage. The noise  $\epsilon_{nt}$ follows an AR(1) process. 
 We generated $1,000$ simulated data based on a desired standard error level. Generating multiple datasets corresponds to  running a specific study multiple times, which allows us to calculate how often---if one could run a study multiple times---one would correctly detect the treatment effect. The desired Type 1 error is set to $\alpha_0=0.05$ and the desired power to $1-\beta_0=0.8$.
See simulation details in Appendix~\ref{subsec:apdx_mobile_environments}.

\subsection{Test Statistics, Baselines and Metrics}
\label{subsec:settings}
\paragraph{Test Statistic Calculation}
To calculate the test statistic $N\hat\delta^\intercal{\hat\Sigma_\delta}^{-1}\hat\delta$, $\hat\delta$ and $\hat\Sigma_\delta$ are needed. 
For the $s^{th}$ simulation dataset, ${\hat\delta}^{(s)}$ can be obtained from ${\hat\theta}^{(s)}$ where ${\hat\theta}^{(s)}$ is estimated with Equation~\ref{eqn:theta_hat}. With all simulated datasets,  $\hat\Sigma_\delta$ can be obtained from $\hat\Sigma_{\theta}=\begin{bmatrix}\hat\Sigma_{\gamma}&\hat\Sigma_{\gamma\delta}\\ \hat\Sigma_{\delta\gamma}&\hat\Sigma_{\delta}\end{bmatrix}$ where $\hat\Sigma_\theta$ is estimated with
\begin{equation}\scriptsize
\begin{aligned}
\hat\Sigma_\theta &= \bigg(\frac{1}{N}\sum_{n=1}^N\sum_{t=1}^T\frac{X_{nt}X_{nt}^\intercal}{\pi_{nt}(1-\pi_{nt})}\bigg)^{-1}
\left(\frac{1}{N}\sum_{n=1}^N\sum_{t=1}^T\frac{\hat\epsilon_{nt}X_{nt}}{\pi_{nt}(1-\pi_{nt})}\right)\\
&\times\quad\left(\frac{1}{N}\sum_{n=1}^N\sum_{t=1}^T\frac{\hat\epsilon_{nt}X_{nt}^\intercal}{\pi_{nt}(1-\pi_{nt})}\right)
\bigg(\frac{1}{N}\sum_{n=1}^N\sum_{t=1}^T\frac{X_{nt}X_{nt}^\intercal}{\pi_{nt}(1-\pi_{nt})}\bigg)^{-1}
\end{aligned}
\label{eqn:sigmahat}
\end{equation}
where $\hat\epsilon_{nt}=R_{nt}-X^\intercal_{nt}\hat\theta$. Equation~\ref{eqn:sigmahat} is derived in Appendix Section~\ref{pf:apdx_thm1}.
The test statistics $\{N{\hat\delta}^{(s)\intercal}{\hat{\Sigma}_\delta}^{-1}{\hat\delta}^{(s)}\}_{s=1}^{1000}$ follow the distribution  in Section~\ref{sec:background}. 
\paragraph{Baselines}
To our knowledge, bandit algorithms with power guarantees are novel.  Thus, we compare our power-preserving strategies applied to various algorithms focused on minimizing regret: ACTS, BOSE, and linear Upper Confidence Bound (linUCB~\citep{chu2011contextual}, which is similar to OFUL but simpler to implement and more commonly used in practice).  We also include the performance of a Fixed Policy ($\pi_{nt}=0.5$ for all $n,t$), a clipped (power-preserving) oracle, 
and standard (non-power preserving) oracle (details in Appendix~\ref{sec:apdx_algo}). 

\paragraph{Metrics}
 For each algorithm, we compute the resulting Type 1 error, the resulting power (under correct and incorrect specifications of various model assumptions in Section~\ref{sec:background}), the regret with respect to a standard oracle (Equation~\ref{eqn:regret_rate}), the regret with respect to a clipped oracle(Equation~\ref{eqn:regret_rate_wrt_clipped_oracle}), 
and the average return $\left(\text{AR} =\mathbb{E}\left[\sum_{t=1}^T R_{nt}\right]\right)$.

\paragraph{Hyperparameters}
All the algorithms require hyperparameters, which are selected by maximizing the average return.
The same parameter values are used in the adapted and non-adapted versions of the algorithms. (All hyperparameter settings in Appendix ~\ref{apdx:subsec_hyperparameter}).

\subsection{Results}
\label{subsec:results}

\begin{figure}[h]
    \centering
    \subfigure[Type 1 error (Mobile Health)]{
    \includegraphics[width=0.38\textwidth,valign=t]{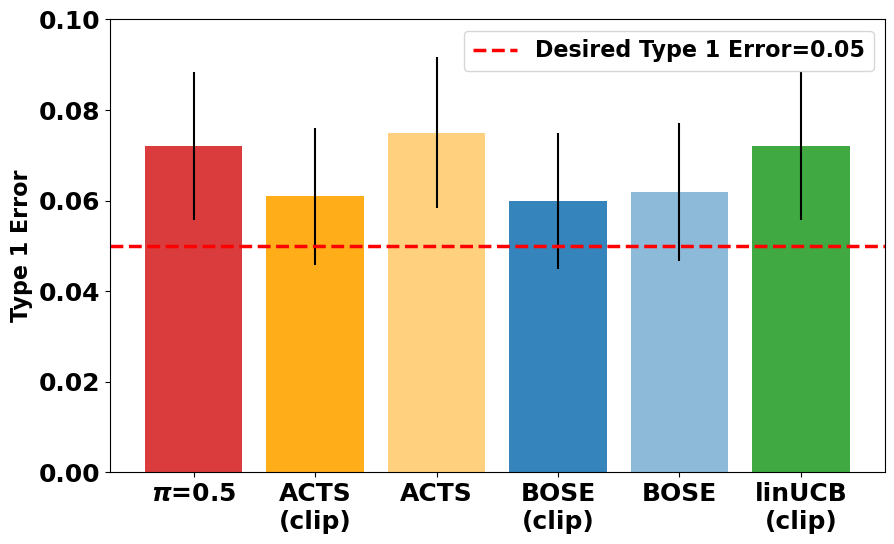}
    \label{fig:mb_type_i}
    }
    \centering
    \subfigure[AR v.s. Power (Mobile Health)]{
    \includegraphics[width=0.38\textwidth,valign=t]{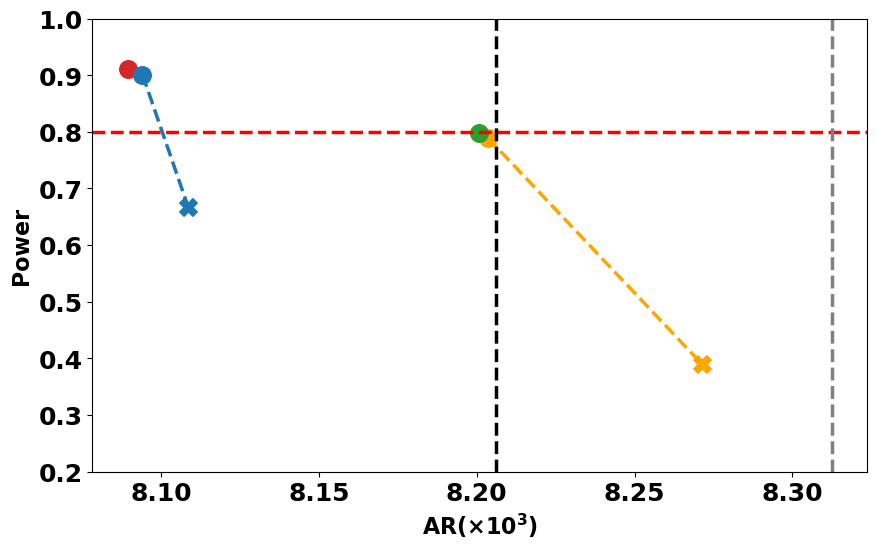}
    \includegraphics[width=0.15\textwidth,valign=t]{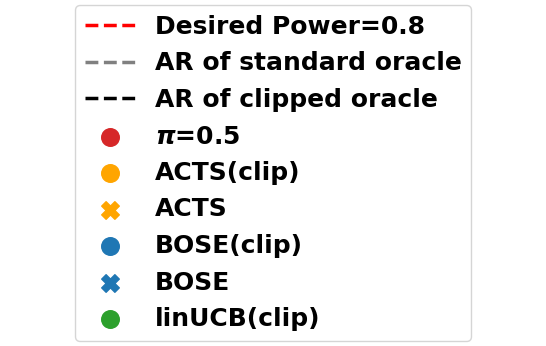}
     \label{fig:mb_regret_vs_power}
     }
     \caption{(a)Type 1 error with 95\% Confidence Interval: We denote estimated Type 1 error as ${\hat\alpha}_0$. 95\% C.I. = 2$\sqrt{{\hat\alpha}_0(1-{\hat\alpha}_0)/S}$ where
$S = 1000$. The red dashed line denotes the desired Type 1 error and the black bar denotes 95\% C.I.. 
We see that some Type 1 errors are slightly higher than 0.5. (b)Average Return v.s. Resulting power: $x$-axis denotes average return and $y$-axis denotes the resulting power. Clipping preserves the power. Power tends to decrease as average return increases. BOSE has the best power with the worst average return. ACTS and linUCB perform similarly in term of power and average return. }
\end{figure}

\textbf{When there is no treatment effect, we recover the correct Type 1 error.} Before power analysis, a basic but critical question is whether we achieve the correct Type 1 error when there is no treatment effect. We have shown in Theorem~\ref{thm:theorem3} that Type 1 error will be trivially guaranteed when the null hypothesis is true.
In Figure~\ref{fig:mb_type_i}, we see that when there is no treatment effect (the messages delivered fail to encourage the user for more physical activity), some Type 1 errors are slightly higher than $0.05$.  This makes sense as the estimated covariance $\hat{\Sigma}_\delta$ is biased downwards due to sample size~\citep{mancl2001covariance}; if needed, this could be controlled by various adjustments or by using critical values based on Hotelling's $T^2$ distribution instead of $\chi^2$ distribution.
\begin{figure}[h]
    \centering
    \subfigure[Mis-estimated Treatment Effect Size]{
    \includegraphics[width=0.32\textwidth,valign=t]{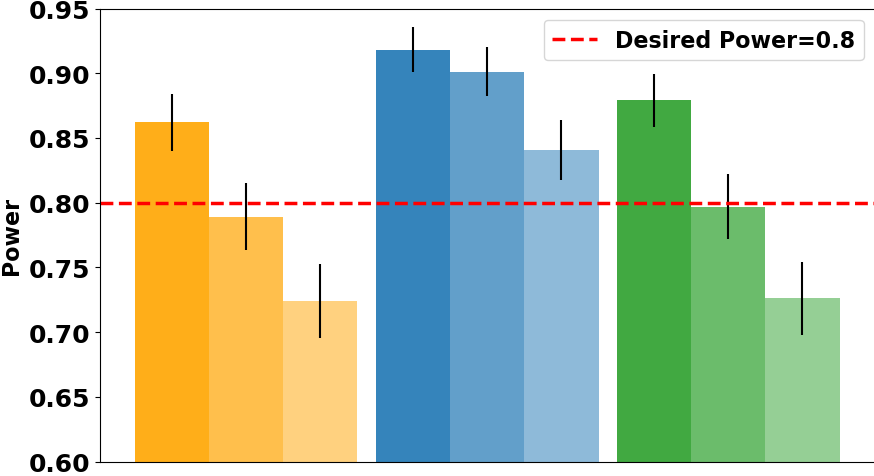}
    \includegraphics[width=0.15\textwidth,valign=t]{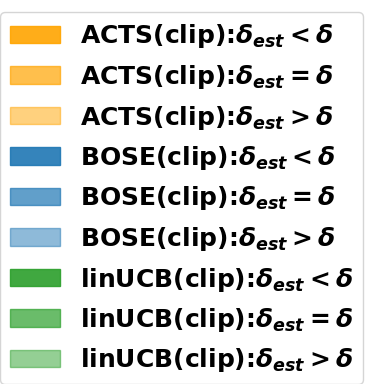}
    \label{fig:mb_treatment}
    }
     \subfigure[Mis-specified Noise Model]{
    \includegraphics[width=0.32\textwidth,valign=t]{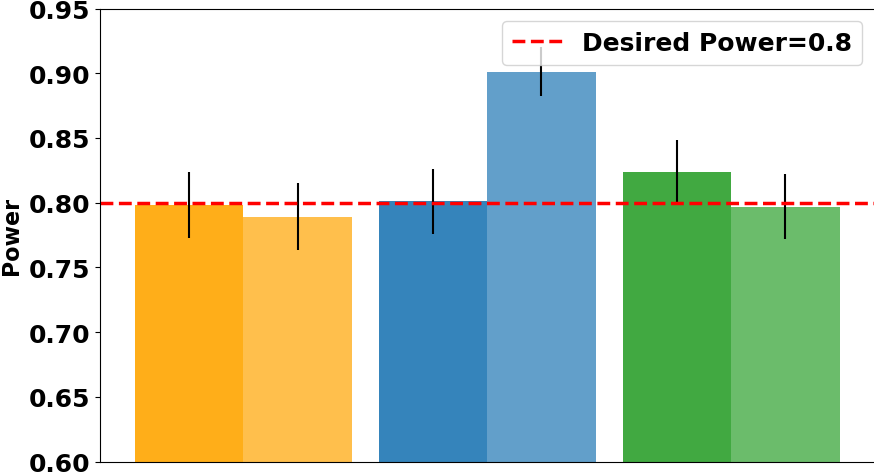}
    \includegraphics[width=0.15\textwidth,valign=t]{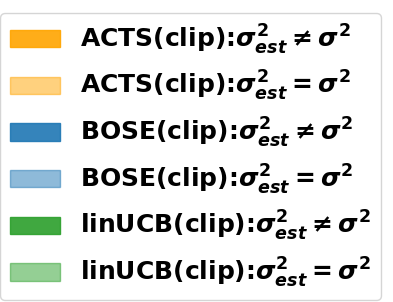}
    \label{fig:mb_noise}
    }
    \subfigure[ Mis-specified Marginal Reward Model]{
    \includegraphics[width=0.32\textwidth,valign=t]{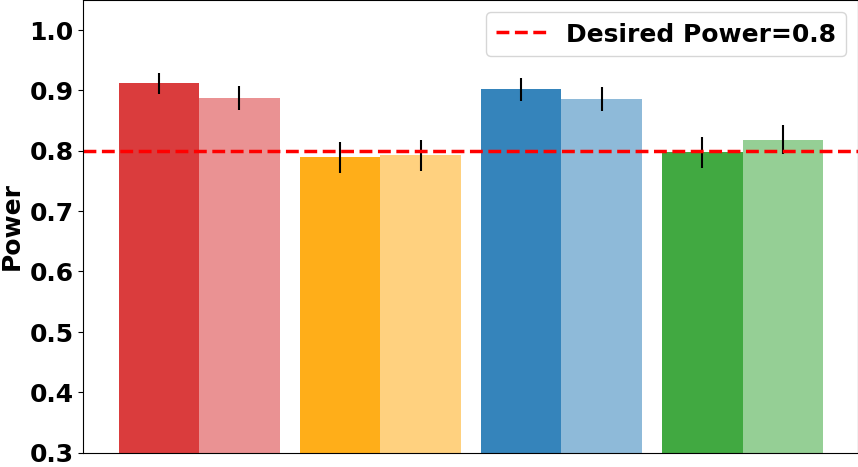}
    \includegraphics[width=0.15\textwidth,valign=t]{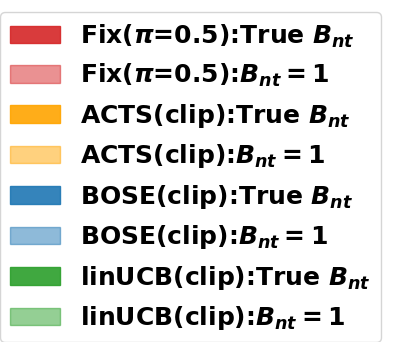}
    \label{fig:mb_marginal_reward_approx}
    }
    \subfigure[Mis-specified Treatment Effect Model]{
    \includegraphics[width=0.32\textwidth,valign=t]{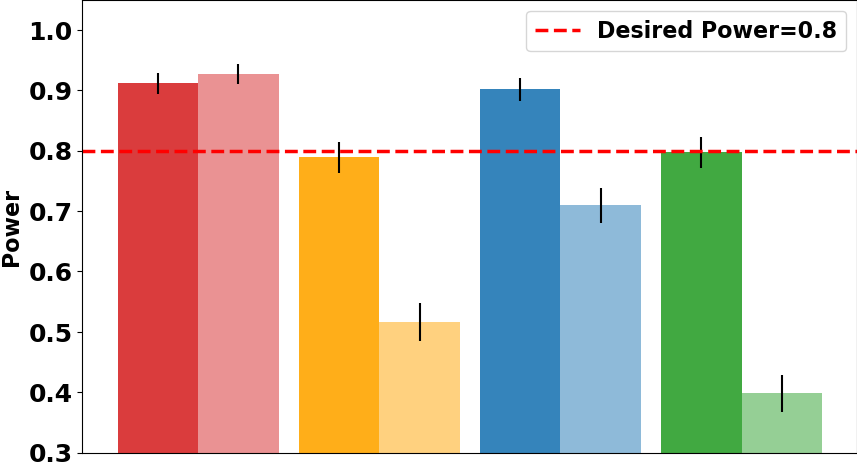}
    \includegraphics[width=0.15\textwidth,valign=t]{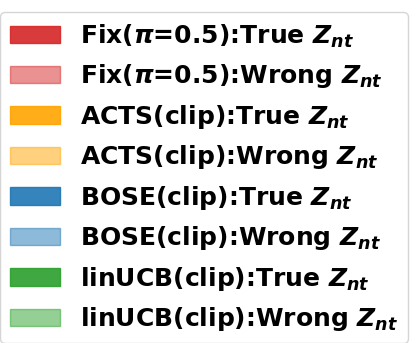}
    \label{fig:mb_treatment_model_approx}
    }
    \caption{Robustness of power guarantees: (a)Effect of mis-estimated treatment effect size on power: When $Z_t\delta_{est}<Z_t\delta_0$, power is higher and when $Z_t\delta_{est}>Z_t\delta_0$, power is lower. Clipped BOSE is the most robust algorithm of all. (b)Effect of mis-specified noise model on power: All algorithms are robust to the specific noise mis-specification in mobile health with all bars meeting the red dashed line. (c)Effect of marginal reward model mis-specification: All the algorithms are robust to marginal reward mis-specification with fixed policy decrease the most. (d)Effect of treatment effect model mis-specification on power: excluding a key feature can have a large impact on the power with clipped linUCB dropping to around $0.4$.}
    \label{fig:mb_robustness test}
\end{figure}

\textbf{When there is a treatment effect, we recover the correct power if we guessed the effect size correctly.} From Figure~\ref{fig:mb_regret_vs_power}, we see that, without clipping, the desired power cannot be achieved  while clipped algorithms recover the correct power (All crosses are below the red line while all circles are above).  
Fixed Policy ($\pi =0.5$) achieves the highest power because the exploration is maximal. Clipped BOSE performs similarly to Fixed Policy. 
For both clipped ACTS and clipped linUCB, the power is approximately $0.80$. Our test statistic relies on a stochastic policy (Theorem~\ref{thm:theorem1}) and is thus not compatible with linUCB's deterministic policy. 

\textbf{There can be a trade-off between regret and the resulting power.}
Figure~\ref{fig:mb_regret_vs_power} also shows  that the average return often increases as the power  decreases overall. For example, Fixed Policy ($\pi =0.5$) gives us the highest power but the lowest average return. Without probability clipping, ACTS and BOSE achieve higher average return but result in less power. For clipped BOSE, the decrease in average return is not significant: the users take
around 100 steps less on each day in average.

\textbf{The power is reasonably robust to a variety of
model mis-specifications, e.g. mis-estimated treatment effect size, mis-estimated noise level, mis-specified marginal reward model (Equation~\ref{eqn:working_model}) and treatment effect model  (Equation~\ref{eqn:true_reward})}. 

\emph{Treatment effect size mis-specification.} We tested when the estimated treatment effect is larger and  smaller than the true treatment effect (The message encourages the users to have more or less physical activities than they truly do). As expected from Theorem~\ref{thm:theorem2}, Figure~\ref{fig:mb_treatment} shows that underestimation results in more exploration, and thus higher power while 
overestimation results in less exploration and lower power.

\emph{Noise Model Mis-specification.} We test the robustness of power against mis-estimated noise variance. For this experiment, we set up the simulator in a way to mimic the data pattern that during the weekend, the user's behavior has more stochasticity due to less motivation. Specifically, we let the noise variance of the weekend to be $1.5^2$ larger than that of the weekdays. The estimated variance is calculated using the average variance over time $\sigma^2_{est}=\frac{1}{T}\sum_{t=1}^T \sigma^2_t$. Figure~\ref{fig:mb_noise} shows that all algorithms are robust to this specific noise mis-specification with all bars meeting the desired power.

\emph{Marginal Reward Model Mis-specification.} Marginal reward mis-specification will also affect the power. In this case, we can prove that when the marginal reward model is mis-specified, the resulting power will decrease (Appendix \ref{pf:apdx_thm4}). The amount of decrease in power, however, may vary, and experimentally we confirm that the effect is insignificant. For this experiment, we approximate the marginal reward, which starts at a large value and decays to $0$ linearly over time, as a constant. From Figure~\ref{fig:mb_marginal_reward_approx}, we see that in this case, all algorithms perform robustly with the heights of the bars remain almost the same.

\emph{Treatment Effect Model Mis-specification.} To see the effect of mis-specified treatment effect models, we consider the case where the constructed feature space is smaller than the true feature space (i.e. experts mistakenly exclude some relevant features). For this experiment, we drop the last dimension of the feature vector $Z_{nt}$ provided by the experts. For mobile health, it turns out that excluding a key feature can have a big effect: In Figure~\ref{fig:mb_treatment_model_approx}, the power of clipped linUCB drops to
around $0.4$.

\begin{figure}[h]
    \centering
    \includegraphics[scale=0.32,valign=t]{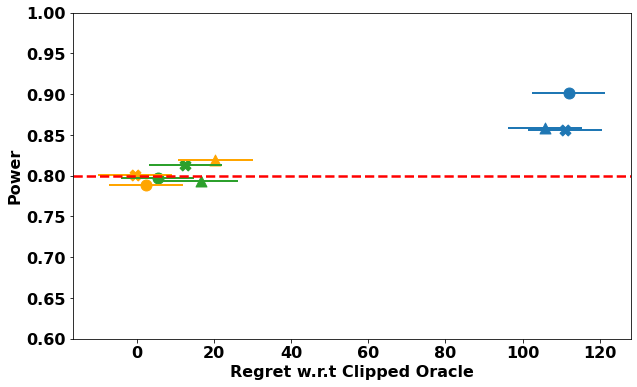}
    \includegraphics[scale=0.4,valign=t]{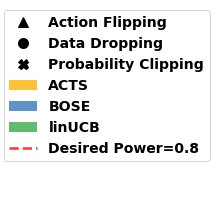}
    \caption{Regret w.r.t clipped oracle v.s. Resulting power for different wrapper algorithms: $x$-axis is regret with respect to clipped oracle and $y$-axis is the resulting power. In mobile health simulator, probability clipping, data dropping and action flipping perform similarly.}
     \label{fig:mb_regret_vs_power_wrapper}
\end{figure}

\textbf{Different algorithms have different regrets, but all still converge as expected with respect to the clipped oracle.} 
Based on Figure~\ref{fig:mb_regret_vs_power}, overall, the regret of clipped algorithms with respect to a clipped oracle is on the same scale as the regret of non-clipped algorithms with respect to a non-clipped oracle (The distance between crosses (x) and the grey dashed line, and the distance between circles (o) and the black dashed line are similar). Our results support the claims in Section~\ref{subsec:regret_wt_clipping} that for specific algorithms we tested, clipping preserves regret rates with respect to the clipped oracle.

\textbf{All wrapper algorithms achieve good regret rate with slightly different trade-offs given the situation}. Figure~\ref{fig:mb_regret_vs_power_wrapper} shows that, for BOSE and linUCB algorithms, all three strategies perform similarly in terms of power and regret. For ACTS, action flipping results in highest regret and highest power due to more exploration and environment stochasticity.

\subsection{Additional Benchmark Environments}
To show the generality of our approach, we also test our algorithms in standard semiparametric and adversarial semiparametric settings whose results are included in Appendix~\ref{subsec:res_benchmark}. In general, the strong results on the mobile health simulator still hold: (1) When the model is correctly specified, we can recover the correct Type 1 error and the correct power with probability clipping. (2) In general, we see a trade-off between average return and resulting power. (3) Our approaches are robust to various model mis-specifications possibly showing up in the clinical studies. (4) Our adapted algorithms are able to retain their original regret guarantees with respect to the clipped oracle. Additionally, similar to Figure~\ref{fig:mb_regret_vs_power_wrapper}, we see that action flipping can perform badly in terms 
of regret comparing to the other two meta-algorithms.

\section{Directions for Extensions}
Our work provides a very general approach to adapting existing contextual bandit algorithms to guarantee sufficient study power in the kinds of very general settings necessary for mobile health studies, while enabling effective personalization.  In this section, we sketch a few extensions to apply our approach to an even broader range of applications.

\paragraph{Extensions to Markov Decision Processes (MDPs)} While we focus on bandits in this work, in some mobile health applications, it might be more reasonable to assume that data are generated from an MDP, where the current state depends on the previous state. For example, in mobile apps for self-management of diabetes, the food intake will have an affect on the patient's glucose level in the next hour.   
Since our power guarantees allow the feature $Z_{nt}$ to be a function of the full history $H_{nt}$, our results in Section~\ref{sec:power} give us the power to identify marginal treatment effects \emph{even if the environment is an MDP}.  The action flipping strategy of Section~\ref{subsec:regret_wrapper} yields the following corollary to Theorem~\ref{thm:wrapper} (proof in Appendix~\ref{subsec:apdx_mdp}):

\begin{corollary}
Given $\pi_{\min},\ \pi_{\max}$ and an MDP algorithm $\mathcal{A}$, 
assume that algorithm $\mathcal{A}$ has an expected regret  $\mathcal{R}(T)$ for any MDP environment in  $\Omega$, 
with respect to an oracle $\mathcal{O}$.
Under stochastic transformation $G$, if there exists an environment  in
$\Omega$ that contains the new transition probability function:
$$P_{s,s'}^{'a}=\left(\pi^a_{\min}\pi^{1-a}_{\max}P_{s,s'}^0+\pi^{1-a}_{\min}\pi^a_{\max} P_{s,s'}^1\right),$$
then the wrapper algorithm will  (1) return a data set that  satisfies the desired power constraints 
and (2) have expected regret no larger than $\mathcal{R}(T)$ with respect to a clipped oracle $\mathcal{O}'$. 
\end{corollary}

\paragraph{Extensions to multiple actions} Although our work focuses on the binary action case, there might be multiple treatment options in some mobile health applications: for example, in HeartSteps, there can be different suggestion messages tailored to different contexts. Our work can be extended to multiple actions easily.  Given $K$ arms, now solve for $\pi_{k,\min},\ \pi_{k,\max}$ for $k\in K$ where $Z_{nt}^\intercal \delta_k$ now represents the treatment effect of action $k$. At each trial, we now have a linear programming problem where the expected reward is maximized  subject to the constraints $\pi_{k,\min}\leq \pi_k \leq \pi_{k,\max}$ and $\sum_{k}^K \pi_k = 1$. The sample size $N$ and trajectory length $T$ needs to be sufficiently large such that a feasible solution exists.

\paragraph{Power for Secondary Analyses} If potential secondary analyses are known, one can seamlessly apply our methods to guarantee power for multiple analyses by considering the minimum $\pi_{\max}$ and maximum $\pi_{\min}$. 

\section{Discussion \& Conclusion}
We describe a general approach for an important need in mobile health: ensuring that studies are sufficiently powered while also personalizing treatment plans for users. We provide regret bounds for specific algorithms; we also provide wrapper algorithms which guarantee that power constraints are met without significant regret increase for a broad class of learning algorithms. With HeartSteps, we show that our wrapper algorithms meet the power guarantees while managing to increase the users' physical activity levels to a large extent. We also show that our approaches are robust to various model mis-specifications possibly appearing in clinical studies. To demonstrate that our work can be applied to more general settings, we also test a couple of benchmark environments. 
In general, we find out that our strong results still hold in benchmark environments.

Finally, in this work we assume that the clipping probabilities remain fixed over time, allowing one to maintain the same regret bound with respect to a clipped oracle for a broad range of algorithms. However, stronger regret bounds may be possible if one considers adaptive clipping strategies and this would be an interesting direction for future research.

\acks{Research reported in this work was supported by the National Institute Of Health grants  P41EB028242, R01 AA023187 and U01 CA229437. The content is solely the responsibility of the authors and does not necessarily represent the official views of the National Institutes of Health. FDV and JY acknowledge support from NSF RI1718306. EB acknowledges support from an NSF CAREER award. WP is support by IACS, Harvard. }

\bibliography{ref}

\newpage
\appendix
\section{Proofs}
\subsection{Reward Function}
\label{apdx:pf_reward_fcn}
In the main text Section~\ref{sec:background}, we state that the reward function can be decomposed into an action-independent marginal reward term and an action-dependent linear treatment effect term. In fact,
\[\mathbb{E}[R_{nt}|A_{nt}, H_{nt}]=\gamma_{nt}+(A_{nt}-\pi_{nt})Z^\intercal_{nt}\delta_0.\]
We show that this is true. Note that the marginal reward $\gamma_{nt}$ is the expected rewards over treatment
 \begin{align*}
     \gamma_{nt}=\mathbb{E}[R_{nt}| H_{nt}]&= \pi_{nt}\mathbb{E}[R_{nt}(1)| H_{nt}]+  (1- \pi_{nt})\mathbb{E}[R_{nt}(0)| H_{nt}]\\
     &= \mathbb{E}[R_{nt}(0)| H_{nt}]+ \pi_{nt}\left(\mathbb{E}[R_{nt}(1)| H_{nt}]-\mathbb{E}[R_{nt}(0)| H_{nt}]\right)\\
     &=\mathbb{E}[R_{nt}(0)| H_{nt}]+\pi_{nt}Z^\intercal_{nt}\delta_0
 \end{align*}
where the last equality comes by the definition of the treatment effect (Equation~\ref{eqn:true_reward} in main text). This implies \[\mathbb{E}[R_{nt}(0)|H_{nt}]=\gamma_{nt}-\pi_{nt}Z^\intercal_{nt}\delta_0\]
 
 Further,
 \begin{align*}
     \mathbb{E}[R_{nt}|A_{nt}, H_{nt}]&= A_{nt}\mathbb{E}[R_{nt}(1)| H_{nt}]+(1-A_{nt})\mathbb{E}[R_{nt}(0)|H_{nt}]\\
     &= A_{nt}\left(\mathbb{E}[R_{nt}(1)| H_{nt}]-\mathbb{E}[R_{nt}(0)|H_{nt}]\right)+ \mathbb{E}[R_{nt}(0)|H_{nt}]\\
     &= A_{nt}Z^{\intercal}_{nt}\delta_0+ \mathbb{E}[R_{nt}(0)|H_{nt}]\\
     &=A_{nt}Z^{\intercal}_{nt}\delta_0+\gamma_{nt}-\pi_{nt}Z^\intercal_{nt}\delta_0\\
     &=\gamma_{nt}+(A_{nt}-\pi_{nt})Z^{\intercal}_{nt}\delta_0
 \end{align*}

\subsection{Proof of Theorem 1}
\label{pf:apdx_thm1}
\begin{theorem}[Restate of Theorem~\ref{thm:theorem1}]
\label{thm:apdx_thm1}
Under the assumptions in main text Section~\ref{sec:background}, and the assumption that matrices $\mathbb{E}[\sum_{t=1}^TZ_{nt}Z_{nt}^\intercal]$, $\mathbb{E}\left[\sum_{t=1}^T\frac{B_{nt}B_{nt}^\intercal}{\pi_{nt}(1-\pi_{nt})}\right]$ are invertible, the distribution of $\sqrt{N}(\hat\delta-\delta_0)$ converges, as $N$ increases, to  a normal distribution with $0$ mean and covariance $\Sigma_\delta = QW^{-1}Q$, where  $Q=\mathbb{E}\left[\sum_{t=1}^TZ_{nt}Z_{nt}^\intercal\right]^{-1}$, and

\begin{align*}
W=\mathbb{E}\bigg[&\sum_{t=1}^T\frac{(R_{nt}-X_{nt}^\intercal\theta^*)(A_{nt}-\pi_{nt})Z_{nt}}{\pi_{nt}(1-\pi_{nt})}\ \sum_{t=1}^T\frac{(R_{nt}-X_{nt}^\intercal\theta^*)(A_{nt}-\pi_{nt})Z_{nt}^\intercal}{\pi_{nt}(1-\pi_{nt})}\bigg],
\end{align*} 
where $\small{ \gamma^*=\mathbb{E}\left[\sum_{t=1}^T\frac{B_{nt}B_{nt}^\intercal}{\pi_{nt}(1-\pi_{nt})}\right]^{-1}\mathbb{E}\left[\sum_{t=1}^T\frac{B_{nt}R_{nt}}{\pi_{nt}(1-\pi_{nt})}\right]}$ and $\small{\theta^*=\begin{bmatrix}\delta_0\\\gamma^*\end{bmatrix}}$.
\end{theorem}
Our proof is a minor adaptation of \citet{boruvka2018assessing}. 
\begin{proof}
Note that since the time series, $n=1,\ldots, N$ are independent and identically distributed, $Q, W, \gamma^*$ do not depend on $n$. 
Suppose the marginal reward is approximated as
\begin{equation}
    \mathbb{E}[R_{nt}|H_{nt}] = B^\intercal_{nt}\gamma_0
\label{eqn:apdx_marginal_reward}
\end{equation}

Let $\theta=\begin{bmatrix}\gamma\\ \delta\end{bmatrix}$, $X_{nt} = \begin{bmatrix}B_{nt}\\(A_{nt}-\pi_{nt})Z_{nt}\end{bmatrix}\in\mathbb{R}^{(q+p)\times 1}$, where $q,\ p$ are the dimensions of $B_{nt},Z_{nt}$ respectively. Note that $X_{nt}$ is random because $B_{nt},A_{nt},\pi_{nt},Z_{nt}$ depend on random history. The test statistics $\hat\theta=\begin{bmatrix}\hat\gamma\\ \hat\delta\end{bmatrix}$ is obtained by minimizing the loss,
\begin{equation*}
L(\theta) = \frac{1}{N}\sum_{n=1}^N\sum_{t=1}^T\frac{\left( R_{nt}-X_{nt}^\intercal\theta\right)^2}{\pi_{nt}(1-\pi_{nt})}
\end{equation*}
By solving $\frac{\partial L}{\partial \theta}=0$, we have the solution for $\hat\theta$
\begin{equation*}
\hat\theta_N = \left(\frac{1}{N}\sum_{n=1}^N\sum_{t=1}^T\frac{X_{nt}X_{nt}^\intercal}{\pi_{nt}(1-\pi_{nt})}\right)^{-1}\left(\frac{1}{N}\sum_{n=1}^N\sum_{t=1}^T\frac{R_{nt}X_{nt}}{\pi_{nt}(1-\pi_{nt})}\right)
\end{equation*}
where $\hat\theta_N$ denotes the estimate of $\theta$ with $N$ samples. We drop the subscript $N$ in the following text for short notation.
Using the weak law of large numbers and the continuous mapping theorem we have that $\hat\theta$ converges in probability, as $N\to\infty$ to  $\theta^*=\begin{bmatrix}\gamma^*\\ \delta^*\end{bmatrix}$ where 
\begin{equation*}
\theta^* = \left(\mathbb{E}\left[\sum_{t=1}^T\frac{X_{nt}X_{nt}^\intercal}{\pi_{nt}(1-\pi_{nt})}\right]\right)^{-1}\left(\mathbb{E}\left[\sum_{t=1}^T\frac{R_{nt}X_{nt}}{\pi_{nt}(1-\pi_{nt})}\right]\right).
\end{equation*}
Note that our goal is to show that $\delta^*=\delta_0$ and $\gamma^*$ is given by the statement in the theorem.  One can do this directly using the above definition for $\theta^* $ or by noting that  
 that $\mathbb{E}[\frac{\partial L}{\partial \theta}]\rvert_{\theta=\theta^*}=0$.  We use the latter approach here.  Recall all the time series are independent and identical; thus
\begin{equation}
\mathbb{E}\left.\left[\frac{\partial L}{\partial \theta}\right]\right\vert_{\theta=\theta^*} = \mathbb{E}\bigg[\sum_{t=1}^T\frac{R_{nt}-B_{nt}^\intercal\gamma^*-(A_{nt}-\pi_{nt})Z_{nt}^\intercal\delta^*}{\pi_{nt}(1-\pi_{nt})}\begin{bmatrix}B_{nt}\\(A_{nt}-\pi_{nt})Z_{nt}\end{bmatrix}\bigg]=0
\label{eqn:apdx_eqn1}
\end{equation}
We first focus on the part with $(A_{nt}-\pi_{nt})Z_{nt}$ which is related to $\delta^*$
\begin{equation*}\mathbb{E}\bigg[\sum_{t=1}^T\frac{R_{nt}-B_{nt}^\intercal\gamma^*-(A_{nt}-\pi_{nt})Z_{nt}^\intercal\delta^*}{\pi_{nt}(1-\pi_{nt})}(A_{nt}-\pi_{nt})Z_{nt}\bigg]=0\end{equation*}
Note that given the history $H_{nt}$, the current $A_{nt}$ is independent of the features $B_{nt},Z_{nt}$. Thus, for all $n,t,$ 
\begin{align*}
\mathbb{E}\bigg[\frac{-B_{nt}^\intercal\gamma^*(A_{nt}-\pi_{nt})Z_{nt}}{\pi_{nt}(1-\pi_{nt})}\bigg]&=\mathbb{E}\left[-B_{nt}^\intercal\gamma^*\mathbb{E}\bigg[\frac{A_{nt}-\pi_{nt}}{\pi_{nt}(1-\pi_{nt})}\middle|H_{nt}\right]Z_{nt}\bigg]\\
&=\mathbb{E}\left[-B_{nt}^\intercal\cdot 0 \cdot Z_{nt}\right]=0
\end{align*} which leaves us with 
\begin{equation*}
\mathbb{E}\bigg[\sum_{t=1}^T\frac{R_{nt}-(A_{nt}-\pi_{nt})Z_{nt}^\intercal\delta^*}{\pi_{nt}(1-\pi_{nt})}(A_{nt}-\pi_{nt})Z_{nt}\bigg]=0.
\end{equation*}
We then rewrite the reward $R_{nt}$ as $R_{nt}(0)+[R_{nt}(1)-R_{nt}(0)]A_{nt}$. Note for all $n,t,$ 
\begin{equation*}
    \mathbb{E}\left[\frac{R_{nt}(0)(A_{nt}-\pi_{nt})Z_{nt}}{\pi_{nt}(1-\pi_{nt})}\right]=\mathbb{E}\left[R_{nt}(0)\mathbb{E}\left[\frac{A_{nt}-\pi_{nt}}{\pi_{nt}(1-\pi_{nt})}\middle|H_{nt}\right]Z_{nt}\right]=0.
\end{equation*}
Thus, we only need to consider, 
\begin{equation}\mathbb{E}\bigg[\sum_{t=1}^T\frac{[R_{nt}(1)-R_{nt}(0)]A_{nt}-(A_{nt}-\pi_{nt})Z_{nt}^\intercal\delta^*}{\pi_{nt}(1-\pi_{nt})}(A_{nt}-\pi_{nt})Z_{nt}\bigg]=0
\label{eqn:apdx_eqn2}
\end{equation}
We observe that for all $n,t,$
\begin{equation}
\mathbb{E}\left[\frac{[R_{nt}(1)-R_{nt}(0)]\pi_{nt}}{\pi_{nt}(1-\pi_{nt})}(A_{nt}-\pi_{nt})Z_{nt}\right]=0.
\label{eqn:apdx_eqn3}
\end{equation}
Subtracting Equation~\ref{eqn:apdx_eqn3} from Equation~\ref{eqn:apdx_eqn2}, we obtain
\begin{align*}
\mathbb{E}\left[\sum_{t=1}^T\frac{[R_{nt}(1)-R_{nt}(0)](A_{nt}-\pi_{nt})-(A_{nt}-\pi_{nt})Z_{nt}^\intercal\delta^*}{\pi_{nt}(1-\pi_{nt})}(A_{nt}-\pi_{nt})Z_{nt}\right]&=0\\
\mathbb{E}\left[\sum_{t=1}^T\frac{[R_{nt}(1)-R_{nt}(0)-Z_{nt}^\intercal\delta^*](A_{nt}-\pi_{nt})^2Z_{nt}}{\pi_{nt}(1-\pi_{nt})}\right]&=0
\end{align*}
Since that given the history $H_{nt}$, the present action $A_{nt}$ is independent of $R_{nt}(0),R_{nt}(1),Z_{nt}$, we know
\begin{equation*}
\mathbb{E}\left[\frac{(A_{nt}-\pi_{nt})^2}{\pi_{nt}(1-\pi_{nt})}\middle| H_{nt}\right]=1.
\end{equation*}
Now, we are only left with
\begin{equation*}
\mathbb{E}\left[\sum_{t=1}^T(R_{nt}(1)-R_{nt}(0)-Z_{nt}^\intercal\delta^*)Z_{nt}\right]=0
\end{equation*}
Solve for $\delta^*$, by Equation~\ref{eqn:true_reward} in the main paper ($\mathbb{E}\left[R_{nt}(1)-R_{nt}(0)|H_{nt}\right]=Z_{nt}^\intercal \delta_0$), we can see
\begin{equation*}
\mathbb{E}\left[\sum_{t=1}^TZ_{nt}Z_{nt}^\intercal\right](\delta_0-\delta^*)=0\ \Rightarrow\ \delta^*=\delta_0.
\end{equation*}
Similarly, we can solve for $\gamma^*$. Focus on the part related to $\gamma^*$ in Equation~\ref{eqn:apdx_eqn1}, we have
\begin{equation*}
\mathbb{E}\bigg[\sum_{t=1}^T\frac{R_{nt}-B_{nt}^\intercal\gamma^*-(A_{nt}-\pi_{nt})Z_{nt}^\intercal\delta^*}{\pi_{nt}(1-\pi_{nt})} B_{nt}\bigg]=0.
\end{equation*}
Since for all $n,t$, $
\mathbb{E}\left[\frac{(A_{nt}-\pi_{nt})}{\pi_{nt}(1-\pi_{nt})}\middle|H_{nt}\right]=0$, we have
\begin{equation*}
\mathbb{E}\bigg[\sum_{t=1}^T\frac{(R_{nt}-B_{nt}^\intercal\gamma^*)B_{nt}}{\pi_{nt}(1-\pi_{nt})}\bigg]=0.
\end{equation*}Hence,
\begin{equation*}
\gamma^*=\left(\mathbb{E}\bigg[\sum_{t=1}^T \frac{B_{nt}B_{nt}^\intercal}{\pi_{nt}(1-\pi_{nt})}\bigg]\right)^{-1}
\mathbb{E}\bigg[\sum_{t=1}^T \frac{B_{nt}R_{nt}}{\pi_{nt}(1-\pi_{nt})}\bigg].
\end{equation*}
Thus $\delta^*=\delta_0$ and $\gamma^*$ is given by the theorem statement.

From the above, we have proved that as $N\rightarrow \infty$, $\delta^*=\delta_0$. Therefore, the distribution of $\sqrt{N}(\hat{\delta}-\delta_0)$ converges, as $N$ increases, to a normal distribution with zero mean. We still need to show that the covariance matrix $\Sigma_\delta$ is indeed $QW^{-1}Q$ where
$Q=\mathbb{E}\left[\sum_{t=1}^TZ_{nt}Z_{nt}^\intercal\right]^{-1}$, and
\begin{align*}
W=\mathbb{E}\bigg[&\sum_{t=1}^T\frac{(R_{nt}-X_{nt}^\intercal\theta^*)(A_{nt}-\pi_{nt})Z_{nt}}{\pi_{nt}(1-\pi_{nt})}\ \sum_{t=1}^T\frac{(R_{nt}-X_{nt}^\intercal\theta^*)(A_{nt}-\pi_{nt})Z_{nt}^\intercal}{\pi_{nt}(1-\pi_{nt})}\bigg],
\end{align*} 
To derive the covariance matrix of $\sqrt{N}(\hat{\delta}-\delta_0)$, we first derive the covariance matrix of $\sqrt{N}(\hat{\theta}-\theta^*)$, denoted as $\Sigma_\theta$. Since  
$\Sigma_\theta=\begin{bmatrix}
\Sigma_\gamma\ \ \ \Sigma_{\gamma\theta}\\
\Sigma_{\theta\gamma}\ \Sigma_\theta
\end{bmatrix}$, we can simply extract $\Sigma_\delta$ from  $\Sigma_\theta$.

We provide a sketch of the derivation below, starting with the following useful formulas about the loss and the expected loss at the optimal values of $\theta$:
\begin{enumerate}
\item $\frac{\partial L}{\partial{\hat\theta}}=\frac{1}{N}\sum_{n=1}^N\sum_{t=1}^T \frac{R_{nt}-X_{nt}^\intercal\hat\theta}{\pi_{nt}(1-\pi_nt)}X_{nt}=0$
\item $\mathbb{E}\Big[\frac{\partial L}{\partial{\theta}}\Big]_{\theta=\theta^*}=\mathbb{E}\bigg[\sum_{t=1}^T \frac{R_{nt}-X_{nt}^\intercal\theta^*}{\pi_{nt}(1-\pi_{nt})}X_{nt}\bigg]=0$
\end{enumerate}
We can combine the two formulas above to get the following equality: 
\begin{equation}0=\underbrace{\frac{\partial L}{\partial{\hat\theta}}-\mathbb{E}\Big[\frac{\partial L}{\partial{\theta}}\Big]_{\theta=\hat\theta}}_{\text{Term 1}}+\underbrace{\mathbb{E}\Big[\frac{\partial L}{\partial{\theta}}\Big]_{\theta=\hat\theta}-\mathbb{E}\Big[\frac{\partial L}{\partial{\theta}}\Big]_{\theta=\theta^*}}_{\text{Term 2}}
\label{eqn:apdx_eqn4}
\end{equation}

We first focus on Term 2.  This term can be expanded as
\begin{align*}
&\mathbb{E}\bigg[\sum_{t=1}^T \frac{R_{nt}-X_{nt}^\intercal\hat\theta}{\pi_{nt}(1-\pi_{nt})}X_{nt}\bigg]-\mathbb{E}\bigg[\sum_{t=1}^T \frac{R_{nt}-X_{nt}^\intercal\theta^*}{\pi_{nt}(1-\pi_{nt})}X_{nt}\bigg]\\
=&\mathbb{E}\bigg[\sum_{t=1}^T\frac{1}{\pi_{nt}(1-\pi_{nt})}\begin{bmatrix}B_{nt}B_{nt}^\intercal&B_{nt}Z_{nt}^\intercal(A_{nt}-\pi_{nt})\\B_{nt}^\intercal Z_{nt}(A_{nt}-\pi_{nt})&Z_{nt}Z_{nt}^\intercal(A_{nt}-\pi_{nt})^2\end{bmatrix}\bigg](\theta^*-\hat\theta)
\end{align*}
Note cross terms inside the matrix are $0$ and $\mathbb{E}\left[\sum_{t=1}^T\frac{Z_{nt}Z_{nt}^\intercal(A_{nt}-\pi_{nt})^2}{\pi_{nt}(1-\pi_{nt})}\right]=\mathbb{E}\left[\sum_{t=1}^T Z_{nt}Z_{nt}^\intercal\right]$.\newline
We have
\begin{equation*}
\text{Term 2}=-\mathbb{E}\left[\sum_{t=1}^T\begin{bmatrix}\frac{B_{nt}B_{nt}^\intercal}{\pi_{nt}(1-\pi_{nt})}&0\\ 0&Z_{nt}Z_{nt}^\intercal\end{bmatrix}\right](\hat\theta-\theta^*).
\end{equation*}
We now look at Term 1. Define
\begin{equation}
u_N(\theta)= \frac{1}{N}\sum_{n=1}^N
\sum_{t=1}^T\frac{R_{nt}-X_{nt}^\intercal\theta}{\pi_{nt}(1-\pi_{nt})}X_{nt}-\mathbb{E}\bigg[\sum_{t=1}^T\frac{R_{nt}-X_{nt}^\intercal\theta}{\pi_{nt}(1-\pi_{nt})}X_{nt}\bigg]
\label{eqn:def_u}
\end{equation}
and note that Term 1 is $u_N(\hat\theta)$ estimated with $N$ samples. We again drop $N$ for short. 

Plugging $\hat\theta$, $\theta^*$ into $u_N$ gives us the following fact
\begin{eqnarray*}
u(\hat\theta)-u(\theta^*)=
-\left( \frac{1}{N}\sum_{n=1}^N
\sum_{t=1}^T\frac{X_{nt}X_{nt}^\intercal}{\pi_{nt}(1-\pi_{nt})}-\mathbb{E}\bigg[\sum_{t=1}^T\frac{X_{nt}X_{nt}^\intercal}{\pi_{nt}(1-\pi_{nt})}\bigg]\right)(\hat\theta-\theta^*)
\end{eqnarray*}
Denote the right hand side of the equation as $-v(\hat\theta-\theta^*)$. Now Term 1 can be written as
\begin{eqnarray*}
\text{Term 1}=u(\hat\theta)=-v(\hat\theta-\theta^*)+u(\theta^*)
\end{eqnarray*}
Plugging Term 1 and Term 2 back into Equation~\ref{eqn:apdx_eqn4} gives us
\begin{equation*}
\mathbb{E}\left[\sum_{t=1}^T\begin{bmatrix}\frac{B_{nt}B_{nt}^\intercal}{\pi_{nt}(1-\pi_{nt})}&0\\ 0&Z_{nt}Z_{nt}^\intercal\end{bmatrix}+v \right](\hat\theta-\theta^*)=u(\theta^*).
\end{equation*}
where by the weak law of large numbers $v$ converges in probability to $0$. Therefore, as $N$ increases, we have 
\[\sqrt{N}(\hat\theta-\theta^*)=\mathbb{E}\left[\sum_{t=1}^T\begin{bmatrix}\frac{B_{nt}B_{nt}^\intercal}{\pi_{nt}(1-\pi_{nt})}&0\\ 0&Z_{nt}Z_{nt}^\intercal\end{bmatrix}\right]^{-1}\sqrt{N}u(\theta^*)\]
Note $\mathbb{E}[u(\theta^*)]=0$ based on the definition in Equation~\ref{eqn:def_u}. Apply central limit theorem on $\sqrt{N}u(\theta^*)$; that is as $N\to\infty$, $\sqrt{N}u(\theta^*)$ converges in distribution to $\mathcal{N}(0,\Sigma)$,
where 
\begin{align*}
\Sigma=\mathbb{E}\bigg[&\sum_{t=1}^T\frac{(R_{nt}-X_{nt}^\intercal\theta^*)X_{nt}}{\pi_{nt}(1-\pi_{nt})}\sum_{t=1}^T\frac{(R_{nt}-X_{nt}^\intercal\theta^*)X_{nt}^\intercal}{\pi_{nt}(1-\pi_{nt})}\bigg]
\end{align*} 
By linear transformation of a multivariate Gaussian, we can convert this covariance on $\sqrt{N}u(\theta^*)$ back to the desired covariance on $\sqrt{N}(\hat\theta-\theta^*)$:
\[\Sigma_\theta=\mathbb{E}\left[\sum_{t=1}^T\begin{bmatrix}\frac{B_{nt}B_{nt}^\intercal}{\pi_{nt}(1-\pi_{nt})}&0\\ 0&Z_{nt}Z_{nt}^\intercal\end{bmatrix}\right]^{-1} \text{\Large$ \Sigma $}\ \ \mathbb{E}\left[\sum_{t=1}^T\begin{bmatrix}\frac{B_{nt}B_{nt}^\intercal}{\pi_{nt}(1-\pi_{nt})}&0\\ 0&Z_{nt}Z_{nt}^\intercal\end{bmatrix}\right]^{-1}\]
Recall that $\Sigma_\delta$ is the lower right matrix of $\Sigma_\theta$.
Denote the lower right matrix of $\Sigma$ by $W$. Then
\begin{align*}
W=\mathbb{E}\bigg[&\sum_{t=1}^T\frac{(R_{nt}-X_{nt}^\intercal\theta^*)(A_{nt}-\pi_{nt})Z_{nt}}{\pi_{nt}(1-\pi_{nt})}\sum_{t=1}^T\frac{(R_{nt}-X_{nt}^\intercal\theta^*)(A_{nt}-\pi_{nt})Z_{nt}^\intercal}{\pi_{nt}(1-\pi_{nt})}\bigg].
\end{align*} 
Therefore, we have
$\sqrt{N}(\hat\delta-\delta_0)\sim\mathcal{N}\left(0,\Sigma_\delta\right)$ where $\Sigma_\delta= \left(\mathbb{E}\Big[\sum_{t=1}^T Z_{nt}Z_{nt}^\intercal\Big]\right)^{-1}W\left(\mathbb{E}\Big[\sum_{t=1}^T Z_{nt}Z_{nt}^\intercal\Big]\right)^{-1}$. We can estimate $\Sigma_\theta$ by putting in sample averages and plugging in $\hat\theta$ as $\theta^*$.

Under the null hypothesis $H_0:\delta_0 = 0$, $N\hat\delta{\hat\Sigma_\delta}^{-1}\hat\delta$ asymptotically follows ${\chi}^2$ with degree of freedom $p$. Under the alternate hypothesis $H_1: \delta_0 = \delta$, $N\hat\delta{\hat\Sigma_\delta}^{-1}\hat\delta$ asymptotically follows a non-central $\chi^2$ with degree of freedom $p$ and non-central parameter $c_N=N(\delta^\intercal{\Sigma_\delta}^{-1}\delta)$. 
\end{proof}

\subsection{Proof of Theorem 2}\label{pf:apdx_thm2}
\begin{theorem}[Restate of Theorem~\ref{thm:theorem2}]
\label{thm:apdx_thm2}
Let $\epsilon_{nt}=R_{nt}-X^\intercal_{nt} \theta^*$ where $\theta^*$ is defined in Theorem~\ref{thm:theorem1}.  Assume that the working model in Equation~\ref{eqn:working_model} is correct.  Further assume that  $\mathbb{E}[\epsilon_{nt}|A_{nt},H_{nt}]=0$ and $Var(\epsilon_{nt}|H_{nt},A_{nt})=\sigma^2$.
Let $\alpha_0$ be the desired Type 1 error and $1-\beta_0$ be the desired power.  Set
\begin{align*}
&\pi_{\text{min}}=\frac{1-\sqrt{1-4\triangle}}{2},
\pi_{\text{max}}=\frac{1+\sqrt{1-4\triangle}}{2}, \\
&\triangle=\frac{\sigma^2c_{\beta_0}}{N\delta^\intercal_0 \mathbb{E}\bigg[\sum_{t=1}^TZ_{nt}Z_{nt}^\intercal\bigg]\delta_0}.
\end{align*}
We choose $c_{\beta_0}$ such that $1-\Phi_{p;c_{\beta_0}}(\Phi_{p}^{-1}(1-\alpha_0))=\beta_0$, where $\Phi_{p;c_{\beta_0}}$ denotes the cdf of a non-central $\chi^2$ distribution with d.f. $p$ and non-central parameter $c_{\beta_0}$, and $\Phi_{p}^{-1}$ denotes the inverse cdf of a $\chi^2$ distribution with d.f. $p$.
For a given trial with  $N$ subjects each over $T$ time units, if the randomization probability is fixed at $\pi_{nt}=\pi_{\text{min}}$ or $\pi_{\text{max}}$, 
the  resulting Type 1 error converges to $\alpha_0$ as $N \xrightarrow{} \infty$ and the resulting power converges to $1-\beta_0$ as $N \xrightarrow{} \infty$.
\end{theorem}

\begin{proof}
According to Section~\ref{pf:apdx_thm1}, under $H_0$, $N\hat\delta\hat\Sigma^{-1}\hat\delta$ will asymptotically follows a ${\chi}^2$ with degree of freedom $p$. The rejection region for $H_0:\delta_0=0$ is
$
\{N\hat\delta^\intercal{\hat\Sigma_\delta}^{-1}\hat\delta>\Phi^{-1}_{p}(1-\alpha_0)\}
$, thus resulting in an expected
Type 1 error of
\begin{equation*}
    \alpha_0 = \Phi_{p}(\Phi_{p}^{-1}(1-\alpha_0)),
\end{equation*}
Under $H_1$, $N\hat\delta\hat\Sigma^{-1}\hat\delta$ will asymptotically follows a non-central ${\chi}^2$ with degree of freedom $p$ and non-central parameter $c_N=N(\delta^\intercal_0{\Sigma_\delta}^{-1}\delta_0)$, which results in an expected power of,
\begin{equation}
1-\Phi_{p;c_N}(\Phi^{-1}_{p}(1-\alpha_0))
\label{eqn:apdx_power}
\end{equation}

Note function~\ref{eqn:apdx_power} is monotonically increasing w.r.t $c_N$. If we want the desired power to be asymptotically $1-\beta_0$, we need $c_N=N\delta^\intercal_0\Sigma^{-1}_{\delta}\delta_0= c_{\beta_0}$, where $\Sigma_\delta$ is the term that involves $\pi_{nt}$. To solve for $\pi_{\min},\ \pi_{\max}$, we first simplify $\Sigma_\delta$ with some additional assumptions in the following Remarks.

\begin{remark}
\label{rmk:apdx_rmk1}  Let $\tilde{\epsilon}_{nt}=R_{nt}-(A_{nt}-\pi_{nt})Z^\intercal_{nt}\delta_0-\gamma_{nt}$.
We make the further assumption that 
$ \mathbb{E}[\tilde{\epsilon}_{nt}|A_{nt},H_{nt}]=0$ and that $Var(\tilde{\epsilon}_{nt}|A_{nt},H_{nt})=\sigma^2$.   Then $W$ can be further simplified as
\begin{align*}\small
\begin{split}
W = &\mathbb{E}\bigg[\sum_{t=1}^T\frac{\sigma^2}{\pi_{nt}(1-\pi_{nt})}Z_{nt}Z_{nt}^\intercal\bigg]+\mathbb{E}\bigg[\sum_{t=1}^T\frac{(\gamma_{nt}
-B_{nt}^\intercal\gamma^*)^2Z_{nt}Z_{nt}^\intercal}{\pi_{nt}(1-\pi_{nt})}\bigg],
\end{split}
\end{align*} 
\end{remark}

\begin{proof}[Proof of Remark~\ref{rmk:apdx_rmk1}]
Since in any cross term, 
\begin{enumerate}
\item $\mathbb{E}[A_{nt}-\pi_{nt}|H_{nt}]=0$,  
\item $Z_{nt}, Z_{nt'}, B_{nt}, B_{nt'}, \gamma_{nt}, \gamma_{nt'}, \tilde{\epsilon}_{nt'}, A_{nt'}, \pi_{nt'}$ are all  determined by $H_{nt}$ when $t'<t$,
\item  and  $\mathbb{E}[\tilde{\epsilon}_{nt}|A_{nt}, H_{nt}]=0$.
\end{enumerate}
Note that $R_{nt}-X^\intercal_{nt}\theta^*=R_{nt}-B^\intercal_{nt}\gamma^*-(A_{nt}-\pi_{nt})Z^\intercal_{nt}\delta_0+\gamma_{nt}-\gamma_{nt}=\tilde{\epsilon}_{nt}+\gamma_{nt}-B^\intercal_{nt}\gamma^*$, we can simply $W$ in Theorem~\ref{thm:apdx_thm1} to
\begin{align*}
W=\mathbb{E}\bigg[&\sum_{t=1}^T\frac{(R_{nt}-X_{nt}^\intercal\theta^*)^2(A_{nt}-\pi_{nt})^2Z_{nt}Z_{nt}^\intercal}{\pi_{nt}^2(1-\pi_{nt})^2}\bigg].\\
\end{align*}

Recall $Var(\tilde{\epsilon}_{nt}|A_{nt},H_{nt})=\sigma^2$.  Then,
\begin{align}
W&=\mathbb{E}\bigg[\sum_{t=1}^T\frac{\tilde{\epsilon}_{nt}^2(A_{nt}-\pi_{nt})^2}{\pi_{nt}^2(1-\pi_{nt})^2}Z_{nt}Z_{nt}^\intercal\bigg]+\mathbb{E}\bigg[\sum_{t=1}^T\frac{(\gamma_{nt}-B_{nt}^\intercal\gamma^*
)^2(A_{nt}-\pi_{nt})^2}{\pi_{nt}^2(1-\pi_{nt})^2}Z_{nt}Z_{nt}^\intercal\bigg]\notag\\
&=\mathbb{E}\bigg[\sum_{t=1}^T\frac{\sigma^2}{\pi_{nt}(1-\pi_{nt})}Z_{nt}Z_{nt}^\intercal\bigg]+
\underbrace{\mathbb{E}\bigg[\sum_{t=1}^T\frac{(\gamma_{nt}
-B_{nt}^\intercal\gamma^*)^2Z_{nt}Z_{nt}^\intercal}{\pi_{nt}(1-\pi_{nt})}\bigg]}_{\text{Term 2}}.
\label{eqn:apdx_sigma_u}
\end{align}
\newline
Assuming the  assumptions in the Remark~\ref{rmk:apdx_rmk1}, we  have $\sqrt{N}(\hat\delta-\delta_0)\sim\mathcal{N}(0,\Sigma_\delta)$ where $\Sigma_\delta$ now simplifies to
\begin{equation}
\Sigma_\delta=\mathbb{E}\left[\sum_t Z_{nt}Z_{nt}^\intercal\right]^{-1} W' \mathbb{E}\left[\sum_t Z_{nt}Z_{nt}^\intercal\right]^{-1}.
\end{equation}
where $W'$ is given in Equation~\ref{eqn:apdx_sigma_u}.
\end{proof}

We now show that when the working model of the marginal reward is correct (i.e. $\gamma_{nt}=B^\intercal_{nt}\gamma_0$), Term 2 in Equation~\ref{eqn:apdx_sigma_u} goes to $0$.

\begin{remark}
\label{rmk:apdx_rmk2}
With the same set of assumptions in Remark~\ref{rmk:apdx_rmk1}, suppose the working model of the marginal reward in Equation~\ref{eqn:apdx_marginal_reward}
is correct, then $\Sigma_{\delta}$ can be further simplified to
\begin{equation}
\Sigma_\delta=\mathbb{E}\left[\sum_t Z_{nt}Z_{nt}^\intercal\right]^{-1}\mathbb{E}\bigg[\sum_{t=1}^T\frac{\sigma^2}{\pi_{nt}(1-\pi_{nt})}Z_{nt}Z_{nt}^\intercal\bigg] \mathbb{E}\left[\sum_t Z_{nt}Z_{nt}^\intercal\right]^{-1}.
\label{eqn:apdx_sigma_delta}
\end{equation}
\end{remark}

\begin{proof}[Proof of Remark~\ref{rmk:apdx_rmk2}]
We first show that when the working model of the marginal reward is correct, $\gamma^*=\gamma_0$. Recall that
\begin{equation*}
\gamma^*=\left(\mathbb{E}\bigg[\sum_{t=1}^T \frac{B_{nt}B_{nt}^\intercal}{\pi_{nt}(1-\pi_{nt})}\bigg]\right)^{-1}
\mathbb{E}\bigg[\sum_{t=1}^T \frac{B_{nt}R_{nt}}{\pi_{nt}(1-\pi_{nt})}\bigg],
\end{equation*}
and by definition of $\tilde{\epsilon}_{nt}$, $R_{nt}=\gamma_{nt}+(A_{nt}-\pi_{nt})Z_{nt}^\intercal\delta_0 + \tilde{\epsilon}_{nt}$ and $\mathbb{E}[R_{nt}|H_{nt},A_{nt}]= \gamma_{nt}+(A_{nt}-\pi_{nt})Z_{nt}^\intercal\delta_0$.  Thus,
\begin{eqnarray*}
\gamma^*&=&\left(\mathbb{E}\bigg[\sum_{t=1}^T \frac{B_{nt}B_{nt}^\intercal}{\pi_{nt}(1-\pi_{nt})}\bigg]\right)^{-1}
\mathbb{E}\bigg[\sum_{t=1}^T  \frac{B_{nt}(\gamma_{nt}+ (A_{nt}-\pi_{nt})Z_{nt}^\intercal\delta_0+\tilde{\epsilon}_{nt})}{\pi_{nt}(1-\pi_{nt})}\bigg]\\
&=&\left(\mathbb{E}\bigg[\sum_{t=1}^T \frac{B_{nt}B_{nt}^\intercal}{\pi_{nt}(1-\pi_{nt})}\bigg]\right)^{-1}
\mathbb{E}\bigg[\sum_{t=1}^T  \frac{B_{nt}(\gamma_{nt}+\tilde{\epsilon}_{nt})}{\pi_{nt}(1-\pi_{nt})}\bigg]
\end{eqnarray*}
where the last equality holds because 
of fact 1 listed in the proof of Remark~\ref{rmk:apdx_rmk1}. Given the assumption that $\mathbb{E}[\tilde{\epsilon}_{nt}|A_{nt},H_{nt}]=0$, then for or all $n,t,$
\begin{equation*}\mathbb{E}\left[\frac{\tilde{\epsilon}_{nt}B_{nt}}{\pi_{nt}(1-\pi_{nt})}\right]=\mathbb{E}\left[\mathbb{E}\left[\tilde{\epsilon}_{nt}| H_{nt}, A_{nt}\right]\frac{B_{nt}}{\pi_{nt}(1-\pi_{nt})}\right]=0
\end{equation*}
\begin{equation*}
\text{and}\quad\gamma^*=\left(\mathbb{E}\bigg[\sum_{t=1}^T \frac{B_{nt}B_{nt}^\intercal}{\pi_{nt}(1-\pi_{nt})}\bigg]\right)^{-1}
\mathbb{E}\bigg[\sum_{t=1}^T  \frac{B_{nt}\gamma_{nt}}{\pi_{nt}(1-\pi_{nt})}\bigg].
\end{equation*}
When  the working model in Equation~\ref{eqn:apdx_marginal_reward} is true, we have $\gamma_{nt}=B_{nt}\gamma_0$
 and thus
\begin{equation*}
\gamma^*=\left(\mathbb{E}\bigg[\sum_{t=1}^T \frac{B_{nt}B_{nt}^\intercal}{\pi_{nt}(1-\pi_{nt})}\bigg]\right)^{-1}
\mathbb{E}\bigg[\sum_{t=1}^T \frac{B_{nt}B_{nt}}{\pi_{nt}(1-\pi_{nt})}\gamma_0\bigg]=\gamma_0.
\end{equation*}
Recall that 
\[W'=\mathbb{E}\bigg[\sum_{t=1}^T\frac{\sigma^2}{\pi_{nt}(1-\pi_{nt})}Z_{nt}Z_{nt}^\intercal\bigg]+
\underbrace{\mathbb{E}\bigg[\sum_{t=1}^T\frac{(\gamma_{nt}
-B_{nt}^\intercal\gamma^*)^2Z_{nt}Z_{nt}^\intercal}{\pi_{nt}(1-\pi_{nt})}\bigg]}_{\text{Term 2}}.\]
Given that $\gamma^*=\gamma_0$, we have $\gamma_{nt}=B_{nt}\gamma_0=B_{nt}\gamma^*$. Thus Term 2 of $W'$ is equal to $0$ 
(When the working model is false, later we will show that Term 2 is positive semidefinite and $\hat\delta$ will likely have inflated covariance matrix).  Assuming the working model is correct and assuming the  assumptions in the Remark, we  simply have
$\Sigma_\delta$ stated in the Remark
\end{proof}

We now proceed with the Proof of Theorem~\ref{thm:apdx_thm2}. When the working model is correct, we observe that $\epsilon=\tilde{\epsilon}$. The assumptions that $ \mathbb{E}[\epsilon_{nt}|A_{nt},H_{nt}]=0$ and that $Var(\epsilon_{nt}|A_{nt},H_{nt})=\sigma^2$ follows from the assumptions in Remark~\ref{rmk:apdx_rmk1}.

Suppose the patient is given treatment with a fixed probability at every trial. i.e. $p(A_{nt}=1)=\pi$, with $\Sigma_\delta$ derived in Remark~\ref{rmk:apdx_rmk2},  we then have
\begin{align}
c_N =& c_{\beta_0}\notag\\
N(\delta_0^\intercal{\Sigma_{\delta}}^{-1}\delta_0) =& c_{\beta_0}\notag\\
N\delta_0^\intercal\mathbb{E}\left[\sum_{t=1}^{T} Z_{nt}Z^\intercal_{nt}\right]\mathbb{E}\left[\sum_{t=1}^{T}Z_{nt}{Z^\intercal_{nt}}\frac{\sigma^2}{\pi_{nt}(1-\pi_{nt})}\right]^{-1}\mathbb{E}\left[\sum_{t=1}^{T} Z_{nt}Z^\intercal_{nt}\right]\delta_0 =& c_{\beta_0}\notag\\
\frac{N\pi(1-\pi)}{\sigma^2}\delta^\intercal_0\mathbb{E}\left[\sum_{t=1}^{T} Z_{nt}Z_{nt}^\intercal\right]\delta_0 =& c_{\beta_0}\notag\\
\pi(1-\pi) =& \triangle,
\label{eqn:apdx_quadtratic}
\end{align}
where $\triangle$ is given by the statement in the theorem. Solving the quadratic function~\ref{eqn:apdx_quadtratic} gives us $\pi=\frac{1\pm\sqrt{1-4\triangle}}{2}$ and the theorem is proved. 
We let $\pi_{\min}=\frac{1-\sqrt{1-4\triangle}}{2}$ and $\pi_{\max}=\frac{1+\sqrt{1-4\triangle}}{2}$. Note that $\pi_{\min}$ and $\pi_{\max}$ are symmetric to $0.5$. Also note that $N$ needs to be sufficiently large so that there exists a root for function~\ref{eqn:apdx_quadtratic}.
\end{proof}
\subsection{Proof Theorem 3}
\label{pf:apdx_thm3}
\begin{theorem}[Restate of Theorem~\ref{thm:theorem3}]
\label{thm:apdx_thm3}
Given the values of $\pi_{\min},\pi_{\max}$ we solved in Theorem~\ref{thm:apdx_thm2}, if for all $n$ and all $t$ we have that $\pi_{nt}\in[\pi_{\min},\pi_{\max}]$, then the resulting power will converge to a value no smaller than $1-\beta_0$ as $N \xrightarrow{} \infty$.
\end{theorem}
\begin{proof}
Function~\ref{eqn:apdx_power} is monotonically increasing w.r.t $c_N$. Hence, to ensure the resulting power is no smaller than $1-\beta_0$, we just need

\begin{equation*}
c_N=N\delta_0^\intercal\mathbb{E}\left[\sum_{t=1}^{T} Z_{nt}Z^\intercal_{nt}\right]\mathbb{E}\left[\sum_{t=1}^{T}Z_{nt}{Z^\intercal_{nt}}\frac{\sigma^2}{\pi_{nt}(1-\pi_{nt})}\right]^{-1}\mathbb{E}\left[\sum_{t=1}^{T} Z_{nt}Z^\intercal_{nt}\right]\delta_0 \geq c_{\beta_0}.    
\end{equation*}

We rewrite some of the terms for notation simplicity. Let 
$b=\mathbb{E}\left[\sum_{t=1}^{T} Z_{nt}Z^\intercal_{nt}\right]\delta_0$. Note $b$ is a vector and $b\in \mathcal{R}^{p\times 1}$, where $p$ is the dimension of $Z_{nt}$. Let $V = \mathbb{E}\left[\sum_{t=1}^{T}Z_{nt}{Z^\intercal_{nt}} a_{nt}\right]$ where $a_{nt} = \frac{1}{\pi_{nt}(1-\pi_{nt})}$. Hence, we have $c_N(a_{nt}) = \frac{N}{\sigma^2} b^\intercal V^{-1} b$
\begin{align*}
\frac{\partial c_N}{\partial a_{nt}}&=tr\left(\left(\frac{\partial c_N}{\partial V^{-1}}\right)^\intercal \frac{\partial V^{-1}}{\partial a_{nt}}\right)\\
&=\frac{N}{\sigma^2}tr(bb^\intercal\times-V^{-1}\frac{dV}{da_{nt}}V^{-1})\\
&=\frac{N}{\sigma^2}tr(-bb^\intercal V^{-1}\mathbb{E}[Z_{nt}Z^\intercal_{nt}] V^{-1})\\
&=-\frac{N}{\sigma^2}(b^\intercal V^{-1})\mathbb{E}[Z_{nt}Z^\intercal_{nt}]( V^{-1} b)\\
\end{align*}

Since $Z_{nt}Z^\intercal_{nt}$ is semi-positive definite, $\mathbb{E}[Z_{nt}Z^\intercal_{nt}]$ is semi-positive definite. Thus $\frac{\partial c_N}{\partial a_{nt}}\leq 0$ and $c_N$ is non-increasing w.r.t $a_{nt}$. As long as we have
\begin{equation*}
\frac{1}{\pi_{nt}(1-\pi_{nt})}\leq \frac{1}{\pi_{\min}(1-\pi_{\min})}\quad \text{and}\quad \frac{1}{\pi_{nt}(1-\pi_{nt})}=\frac{1}{\pi_{\max}(1-\pi_{\max})},\end{equation*} 
we will have that $c_N\geq c_{\beta_0}$.\newline
Since for all $n, t$ and $\ \pi_{nt}\in[\pi_{\min},\pi_{\max}]$, we have
\begin{equation*}
\pi_{nt}(1-\pi_{nt})\geq \pi_{\min}(1-\pi_{\min}) = \pi_{\max}(1-\pi_{\max}),
\end{equation*}
and hence
\begin{equation*}
\frac{1}{\pi_{nt}(1-\pi_{nt})} \leq \frac{1}{\pi_{\min}(1-\pi_{\min})}=\frac{1}{\pi_{\max}(1-\pi_{\max})}.
\end{equation*}
Thus, $c_N\geq c_{\beta_0}$. The power constraint will be met. 
\end{proof}

\subsection{The Effect of Model Mis-specification on Power}
\label{pf:apdx_thm4}
\begin{corollary}
\label{cor:corollary_1}
When the marginal reward structure is incorrect ( $B_{nt}\gamma_0\neq \gamma_{nt}$), the resulting power will converge to a value less than the desired power $1-\beta_0$ as $N \xrightarrow{} \infty$.
\end{corollary}

\begin{proof}
When the construction model of the marginal reward is not correct, the estimator $\hat\gamma$ will be biased and now Term 2 in $W'$ (Equation~\ref{eqn:apdx_sigma_u}) is non-zero. Using the same notation in Section~\ref{pf:apdx_thm3}, $c_N=\frac{N}{\sigma^2}b^\intercal V^{'-1}b$, we now have

\begin{equation*}
V'=\mathbb{E}\left[\sum_{t=1}^TZ_{nt}Z^\intercal_{nt}a_{nt}(1+c_{nt})\right],\ \text{where}\
a_{nt} = \frac{1}{\pi_{nt}(1-\pi_{nt})}\text{ and }\ c_{nt}=\frac{(\gamma_{nt}-B_{nt}^\intercal\gamma^*)^2}{\sigma^2}
\end{equation*}
Following similar derivation in Section~\ref{pf:apdx_thm3}, we have
\begin{equation*}
\frac{\partial c_N}{\partial c_{nt}}=-\frac{N}{\sigma^2}(b^\intercal \Sigma^{-1})\mathbb{E}[Z_{nt}Z^\intercal_{nt}a_{nt}]( \Sigma^{-1} b)
\end{equation*}
Since $a_{nt}>0$, $\frac{\partial c_N}{\partial c_{nt}}<0$. Thus $c_N$ is monotonically decreasing w.r.t $c_{nt}$. Hence, when the reward mean structure is incorrect, the noncentral parameter $c_N$ will decrease and thus, power will be less than $1-\beta_0$.
\end{proof}

\subsection{Regret Bound of Specific Algorithms}
\label{subsec:apdx_alg_regret_bound}
In the main text Section~\ref{sec:regret}, we mentioned that there exists specific algorithms in which the regret rates with respect to a clipped oracle can be preserved by simply clipping the action selection probability to lie within $[\pi_{\min}, \pi_{\max}]$. Below, we list three specific algorithms, describe their environment assumptions and provide a proof sketch that the regret rates are preserved. 

\textbf{Action-Centered Thompson Sampling (ACTS).}  ACTS~\citep{greenewald2017action} already has optimal first order regret with respect to a clipped oracle in non-stationary, adversarial settings where the features and reward are a function of current context $C_{nt}$ (rather than the history $H_{nt}$). They do not consider power; using our probabilities will result in optimal regret and satisfy required power guarantees.

\textbf{Semi-Parametric Contextual Bandits (BOSE).}
BOSE~\citep{Krishnamurthy2018} has optimal first order regret with respect to a standard oracle in a non-stationary, adversarial setting.  Like ACTS, features and rewards are functions of the current context $C_{nt}$. They further assume noise term is action independent. In the two action case, BOSE will select actions with probability 0.5 or with probability 0 or 1. With probability clipping, the regret bound remains unaffected and the details are provided in Section 3.3 of~\citep{Krishnamurthy2018}. 

\textbf{A More Subtle Case: Linear Stochastic Bandits (OFUL).}
Finally, consider the OFUL algorithm of~\citet{abbasiyadkori2011} which considers a linear assumption on the entire mean reward that  $\mathbb{E}[R_{nt}|A_{nt}=a] =x_{t,a}^T\theta$ for features $(x_{t,0},x_{t,1})$. We prove that with probability clipping, OFUL will maintain the same regret rate with respect to a clipped oracle. 

The clipped OFUL algorithm is given in Algorithm~\ref{alg:apdx_OFUL}.
The proof below is separate for each subject; thus for simplicity we drop the subscript $n$ (e.g. use $R_t$ instead of $R_{nt}$). We also only assume that $0<\pi_{\min}\le\pi_{\max}<1$, that is, we do not require the sum, $\pi_{\min}+\pi_{\max}=1$.  As we have binary actions, we can write \citeauthor{abbasiyadkori2011}'s decision set as $D_t= \{x_{t,0}, x_{t,1}\}$; the second subscript denotes the binary action and $x$ denotes a feature vector for each action.  
To adapt OFUL to accommodate the clipped constraint, we will make a slight change to ensure optimism under the  constraint. Specifically, the criterion $x_{t,a}^\intercal\theta$ is replaced by  
$\ell_t(a,\theta)=\mathbb{E}[x_{t,A_t^c}^\intercal\theta| A_t=a]$ where $A_t^c\sim$ Bernoulli($\pi_{\text{max}}^{a}\pi_{\text{min}}^{1-a}$).
Construction of the confidence set remains the same.  

\begin{algorithm}[htb]
  \caption{Clipped OFUL (Optimism in the Face of Uncertainty)}
  \label{alg:apdx_OFUL}
\begin{algorithmic}[1]
\STATE Input: $\pi_{\text{max}}, \pi_{\text{min}}$ 
\FOR{$t = 1,2,\dots,T$}
\STATE Observe context features for each possible action: $\{x_{t,1}, x_{t,0}\}$
\STATE $(A_t,\tilde\theta_t)=\arg\max_{(a,\theta)\in \{0,1\} \times C_{t-1}} \ell_t(a,\theta)$ 
\STATE Play $A^c_t\sim$ Bernoulli ($\pi_{\text{max}}^{A_t}\pi_{\text{min}}^{1-A_t}$) and observe reward $R_t(A^c_t)$
\STATE Update confidence set $C_t$
\ENDFOR
\end{algorithmic}
\end{algorithm}

\begin{proof}
Clipped OFUL uses a two-step procedure to select the (binary) action in $D_t$. It first selects an optimistic $A_t$ in step 4. However, instead of implementing $A_t$, it implements  action $A^c_t$ where $A^c_t\sim\texttt{Bern}(\pi_{max}^a\pi_{min}^{1-a})$ given  $A_t=a$.   This means that $X_t$ in \cite{abbasiyadkori2011} becomes $x_{t,A^c_t}$ in clipped OFUL.  

We use notations and assumptions similar to \cite{abbasiyadkori2011}.
Let $\{F_t\}_{t\ge 1}$ be a filtration, the error terms, $\{\eta_t\}_{t\ge 1}$ be a real-valued stochastic process, the features, $\{X_t\}_{t\ge 1}$ be a $\mathbb{R}^d$-valued stochastic process. $\eta_t$ is $F_t$ measurable and  $X_t$ is $F_{t-1}$ measurable. Further assume that  $||X_t||_2\le L$ for a constant $L$. Define $V=\lambda I\in \mathbb{R}^{d\times d}$ with $\lambda\ge 1$.  
The observed reward is assumed to satisfy 
$$R_t=\theta^\intercal_* X_t +\eta_t$$ 
for an unknown $\theta_*\in \mathbb{R}^d$.  The error term  $\eta_t$ is assumed to be conditionally $\sigma$-sub-Gaussian for a finite positive constant $\sigma$. This implies that $\mathbb{E}[\eta_t|F_{t-1}]=0$ and $Var[\eta_t|F_{t-1}]\le \sigma^2$.  The coefficient satisfies $||\theta_*||_2\le S$ for a constant $S$. Lastly assume that $|\max\{\theta_*^\intercal x_{t,1},\theta_*^\intercal x_{t,0}\}|\le 1$.  

Under these assumptions, Theorems 1,  2, Lemma 11 of~\cite{abbasiyadkori2011} as well as their proofs remain the same with $X_t$ defined as $x_{t,A^c_t}$.  Theorem 2 concerns construction of the confidence set.  Neither Theorems 1, 2  or Lemma 11 concern the definition of the regret and only Theorem 3 and its proof need be altered to be valid for clipped OFUL with the regret against a clipped oracle.  

Define $$\ell_t(a,\theta)= a[\pi_{\max}\theta^\intercal x_{t,1} + (1-\pi_{\max})\theta^\intercal x_{t,0}] +(1-a)[\pi_{\min}\theta^\intercal x_{t,1} + (1-\pi_{\min})\theta^\intercal x_{t,0}].$$  
Below it will be  useful to note that $\ell_t(a,\theta)=\mathbb{E}[\theta^\intercal x_{t,A^c_t}|A_t=a, F_{t-1}]$. 

First we define the clipped oracle.
Recall the oracle action is  $A^*_t=\arg\max_a\theta_*x_{t,a}$.  It is easy to see that $A^*_t=\arg\max_a \mathbb{E}[\theta_*^\intercal x_{t,A^{c*}_t}|A_t^*=a, F_{t-1}]$  for $A_t^{c*}\sim$ Bernoulli$(\pi_{\max}^a\pi_{\min}^{1-a})$.    The clipped oracle action is $A^{c*}_t$. 
 Note that $\mathbb{E}[\theta_*^\intercal x_{t,A^{c*}_t}|A_t^*=a, F_{t-1}]=\ell_t(a,\theta_*)$.  So just as  $A^*_t$ maximizes $\ell_t(a,\theta_*)$, in clipped OFUL the optimistic action, $A_t$, similarly provides an $\arg\max$ of  $\ell_t(a,\theta)$; see line 4 in Algorithm~\ref{alg:apdx_OFUL}.
 
The time $t$ regret  against the clipped oracle is given by
 $r_t=\ell_t(A_t^*,\theta_*)-\ell_t(A_t,\theta_*)$.  In the proof to follow it is useful to note that $r_t$ can also be written as $r_t=  \mathbb{E}[\theta_*^\intercal x_{t,A^{c*}_t}|A_t^*, F_{t-1}]-\mathbb{E}[\theta_*^\intercal x_{t,A^{c}_t}|A_t, F_{t-1}]$. 
  In the following we provide an upper bound on the expected regret, $\mathbb{E}\left[\sum_{t=1}^nr_t\right]$. 
 \begin{eqnarray*}
 r_t&=& \ell_t(A_t^*,\theta_*)-\ell_t(A_t,\theta_*)\\
 &\le& \ell_t(A_t,\tilde\theta_t)-\ell_t(A_t,\theta_*)\text{ (by line 4 in Alg.~\ref{alg:apdx_OFUL})}\\
 &=&\mathbb{E}[\tilde\theta_t^\intercal x_{t,A^{c}_t}|A_t, F_{t-1}]-\mathbb{E}[\theta_*^\intercal x_{t,A^{c}_t}|A_t, F_{t-1}]\text{ (by line 5 in Alg.~\ref{alg:apdx_OFUL})}\\
 &=&  \mathbb{E}[ (\tilde\theta_t-\theta_*)^\intercal x_{t,A^{c}_t}|A_t, F_{t-1}].
 \end{eqnarray*}
 
Thus we have that 
 \[\mathbb{E}[r_t]\le \mathbb{E}[ (\tilde\theta_t-\theta_*)^\intercal x_{t,A^{c}_t}]= \mathbb{E}[ (\tilde\theta_t-\theta_*)^\intercal X_t]\]
with the second equality holding due to  the definition of $X_t$.
The  proof of  Theorem 3 in  \citet{abbasiyadkori2011} provides a high probability upper bound on $(\tilde\theta_t-\theta_*)^\intercal X_t$.  In particular the proof shows that 
with probability at least $(1-\delta)$, for all $n\ge 1$,
\begin{eqnarray*}
\sum_{t=1}^n (\tilde\theta_t-\theta_*)^\intercal X_t &\le& 4\sqrt{nd\log(\lambda +nL/d)}\biggl(\lambda^{1/2}S + R\sqrt{2\log(1/\delta)+d\log(1+nL/(\lambda d))}\biggr)\\
&\le&  4\sqrt{nd\log(\lambda +nL/d)}\biggl(\lambda^{1/2}S + R\sqrt{2\log(1/\delta)}+R\sqrt{d\log(1+nL/(\lambda d))}\biggr)
\end{eqnarray*}
since for $x>0$, $\sqrt{1+x}\le 1+\sqrt{x}$.  

Let $a_n=4\sqrt{nd\log(\lambda +nL/d)}$, $b_n=\lambda^{1/2}S + R \sqrt{d\log(1+nL/(\lambda d))}$ and $c=R\sqrt{2}$.  We have 
 $P\left[\sum_{t=1}^n (\tilde\theta_t-\theta_*)^\intercal X_t \ge a_n(b_n+c\sqrt{\log(1/\delta)}\right]\le \delta$.  Let $v=a_n\left(b_n+c\sqrt{\log(1/\delta)}\right)$ then solving for $\delta$ one obtains $\delta=\exp{\left\{-\left(v-b_na_n\right)^2/(a_n c)^2\right\}}$.   Thus $P\left[\sum_{t=1}^n (\tilde\theta_t-\theta_*)^\intercal X_t \ge v\right]\le \exp{\left\{-\left(v-b_na_n\right)^2/(a_n c)^2\right\}}$.
 
 Recall that for any random variable, $Y$, $\mathbb{E}[Y]\le \int_0^\infty P[Y>u] du$.  
 Thus 
 \begin{eqnarray*}
 \mathbb{E}\left[\sum_{t=1}^n r_t\right]&=&\mathbb{E}\left[\sum_{t=1}^n (\tilde\theta_t-\theta_*)^\intercal X_t \right]\\
 &\le&\int_0^\infty \exp{\left\{-\left(v-b_na_n\right)^2/(a_n c)^2\right\} }dv\\
 &\le&a_nc\sqrt{\pi}\\
 &=&4R\sqrt{2\pi nd\log(\lambda +nL/d)}.
 \end{eqnarray*}
 Thus the expected regret  up to time $n$ is of order $O(\sqrt{n})$ up to terms in $\log(n)$ for clipped OFUL.
 \end{proof}

\subsection{Data-Dropping Power-Preserving Wrapper Algorithm}
\label{sec:apdx_data_dropping}
 \begin{algorithm}[htb]
  \caption{Data-Dropping Power-Preserving Wrapper Algorithm}
  \label{alg:apdx_dropping}
\begin{algorithmic}[1]
\STATE Input: $\pi_{\min},\ \pi_{\max}$, Algorithm $\mathcal{A}$ 
\FOR{$t = 1,2,\dots$}
\STATE Observe context $C_t$ and outputs $\pi_{\mathcal{A}}(C_t)$ each action
\IF{$\pi_{\max}\leq_a \{\pi_{\mathcal{A}}(a)\} \leq \pi_{\max}$}
\STATE $A_t \sim \pi_{\mathcal{A}}$ \COMMENT{Sample action}
\STATE Observe $R_t$
\STATE Update Algorithm $\mathcal{A}$ with $(C_t,A_t,R_t)$
\ELSE
\STATE $u \sim \texttt{unif}(0,1)$
\STATE $A^*_t = \arg\max_a \pi_{\mathcal{A}}(a)$
\IF{$u \leq \pi_{\max}$ or $u > \max_a \{\pi_{\mathcal{A}}(a)\}$}
\IF{$u \leq  \pi_{\max}$}
\STATE $A_t = A^*_t$
\ELSE
\STATE $A_t = \arg\min \{\pi_{\mathcal{A}}(a)\}$
\ENDIF
\STATE Observe $R_t$
\STATE Update Algorithm $\mathcal{A}$ with $(C_t,A_t,R_t)$ \COMMENT{Both approaches agree on action}
\ELSE 
\STATE $A_t = \arg\min \{\pi_{\mathcal{A}}(a)\}$
\STATE Observe $R_t$ \COMMENT{Do not give data to $\mathcal{A}$}
\ENDIF
\ENDIF
\ENDFOR
\end{algorithmic}
\end{algorithm}
In this section, we give full analyses of the data-dropping wrapper algorithm which can also be used for power preserving purpose. The algorithm implementation is given in Algorithm~\ref{alg:apdx_dropping}.  The wrapper takes as input a contextual bandit algorithm $\mathcal{A}$ and pre-computed $\pi_{\min},\ \pi_{\max}\ (\pi_{\max}+\pi_{\min}=1)$ computed from Theorem~\ref{thm:apdx_thm2}. The input algorithm $\mathcal{A}$ can be stochastic or deterministic.  Conceptually, our wrapper operates as follows: for a given context, if the input algorithm $\mathcal{A}$ returns a probability distribution over choices that already satisfies $\pi_{\mathcal{A}} \in [\pi_{\min}, \pi_{\max}]$, then we sample the action according to $\pi_{\mathcal{A}}$. However, if the maximum probability of an action exceeds $\pi_{\max}$, then we sample that action according to $\pi_{\max}$.  

The key to guaranteeing good regret with this wrapper for a broad range of input algorithms $\mathcal{A}$ is in ensuring that the input algorithm $\mathcal{A}$ only sees samples that match the data it would observe if \emph{it} was making all decisions.
 Specifically, the sampling approach in lines 9-22 determines whether the action that was ultimately taken would have been taken absent the wrapper; the context-action-reward tuple from that action is only shared with the input algorithm $\mathcal{A}$ if $\mathcal{A}$ would have also made that same decision. 

Now, suppose that the input algorithm $\mathcal{A}$ was able to achieve some regret bound $\mathcal{R}(T)$ with respect to some setting $\mathcal{B}$ (which, as noted before, may be more specific than that in Section~\ref{sec:background} in main paper).  The wrapped version of input $\mathcal{A}$ by Algorithm~\ref{alg:apdx_dropping} will achieve the desired power bound by design; but what will be the impact on the regret?  We prove that as long as the setting $\mathcal{B}$ allows for data to be dropped, then an algorithm that incurs $\mathcal{R}$ regret in its original setting suffers at most $(1-\pi_{\max})$ linear regret in the clipped setting.
Specifically, if an algorithm $\mathcal{A}$ achieves an optimal rate $O(\sqrt{T})$ rate with respect to a standard oracle, its clipped version will achieve that optimal rate with respect to the clipped oracle.

\begin{theorem}[Restate of Theorem~\ref{thm:data_dropping}]
\label{thm:apdx_thm5}
Assume as input $\pi_{\max}$ and a contextual bandit algorithm $\mathcal{A}$. Assume algorithm $\mathcal{A}$ has a regret bound $\mathcal{R}(T)$ under one of the following assumptions on the setting $\mathcal{B}$: (1) $\mathcal{B}$ assumes that the data generating process for each context is independent of history, or (2) $\mathcal{B}$ assumes that the context depends on the history, and the bound $\mathcal{R}$ for algorithm $\mathcal{A}$ is robust to an adversarial choice of context.

Then our wrapper Algorithm~\ref{alg:apdx_dropping} will  (1) return a dataset that  satisfies the desired power constraints and (2) has expected regret no larger than $ \mathcal{R}(\pi_{\max}T) + (1-\pi_{\max})T$ if assumptions $\mathcal{B}$ are satisfied in the true environment.
\end{theorem}

\begin{proof}
\textbf{Satisfaction of power constraints:} By construction our wrapper algorithm ensures that the selected actions always satisfy the required power constraints. 

\textbf{Regret with respect to a clipped oracle:} Note that in the worst case, the input algorithm $\mathcal{A}$ deterministically selects actions $A_t$, which are discarded with probability $1-\pi_{\text{max}}$. Therefore if running in an environment satisfying the assumptions $\mathcal{B}$ of the input algorithm $\mathcal{A}$, our wrapper could suffer at most linear regret on $T(1-\pi_{\text{max}})$ points, and will incur the same regret as the algorithm $\mathcal{A}$ on the other points (which will appear to algorithm $\mathcal{A}$ as if these are the only points it has experienced). 

Note that since the wrapper algorithm does not provide all observed tuples to algorithm $\mathcal{A}$, this proof only works for assumptions $\mathcal{B}$ on the data generating process that assumes the contexts are independent of history, or in a setting in which $\mathcal{A}$ is robust to adversarially chosen contexts. 
\end{proof}

Essentially this result shows that one can get robust power guarantees while incurring a small linear loss in regret (recall that $\pi_{\max}$ will tend toward 1, and $\pi_{\min}$ toward 0, as $T$ gets large) if the setting affords additional structure commonly assumed in stochastic contextual bandit settings.  Because our wrapper is agnostic to the choice of input algorithm $\mathcal{A}$, up to these commonly assumed structures, we enable a designer to continue to use their favorite algorithm---perhaps one that has seemed to work well empirically in the domain of interest---and still get guarantees on the power.

\begin{corollary}
For algorithms $\mathcal{A}$ that satisfy the assumptions of Theorem~\ref{thm:apdx_thm5}, our wrapper algorithm will incur regret no worse than $O( \mathcal{R}(\pi_{max}T))$ with respect to a clipped oracle.
\end{corollary}
\begin{proof}
Recall that a clipped oracle policy takes the optimal action with probability $\pi_{\max}$ and the other action with probability $1-\pi_{\max}$.  By definition, any clipped oracle will suffer a regret of $(1-\pi_{\max})T$. Therefore relative to a clipped oracle, our wrapper algorithm will have a regret rate $O(\mathcal{R}(\pi_{\max}T))$ that matches the regret rate of the algorithm in its assumed setting when the true environment satisfies those assumptions. This holds for algorithms $\mathcal{A}$ satisfying the assumptions of Theorem~\ref{thm:apdx_thm5}.
\end{proof}

\subsection{Action Flipping Wrapper Algorithm}
\label{sec:apdx_action_flip}
In this section, we provide full analyses of the action flipping wrapper algorithm described in Section~\ref{subsec:regret_wrapper} in the main paper. We first prove that the wrapper algorithm can be applied to a large class of algorithms and achieves good regret rate with respect to a clipped oracle and then we listed common algorithms on which the wrapper algorithm can be used. The proof below will drop the subscript $n$ since the algorithm is for each user separately.

\paragraph{Meta-Algorithm: Action-Flipping}(Restated)
\begin{enumerate}
    \item Given current context $C_{t}$, algorithm $\mathcal{A}$ produces action probabilities $\pi_\mathcal{A}(C_{t})$
    \item  Sample $A_{t} \sim \texttt{Bern}(\pi_\mathcal{A}(C_{t}))$.
    \item If $A_{t}=1$, sample  $A_{t}^\prime \sim \texttt{Bern}(\pi_{\max})$.  If $A_{nt}=0$, sample  $A_{t}^\prime \sim \texttt{Bern}(\pi_{\min})$.
    \item We perform $A'_{t}$ and receive reward $R_{t}$.
    \item The algorithm $\mathcal{A}$ stores the tuple $C_{t} , A_{t}, R_{t}$.  (Note that if $A_{t}$ and $A^\prime_{t}$ are different, then, unbeknownst to the algorithm $\mathcal{A}$, a different action was actually performed.)
    \item The scientist stores the tuple $C_{t} , A^\prime_{t} , R_{t}$ for their analysis.
\end{enumerate}

\begin{theorem}[Restate of Theorem~\ref{thm:wrapper}]
\label{thm:apdx_thm4}
Given $\pi_{\min},\ \pi_{\max}$ and a contextual bandit algorithm $\mathcal{A}$, 
assume that algorithm $\mathcal{A}$ has expected regret  $\mathcal{R}(T)$ for any environment in  $\Omega$, 
with respect to an oracle $\mathcal{O}$.
If there exists an environment  in
$\Omega$ such that the potential rewards, $R'_{nt}(a) = R_{nt}(G(a))$ for $a\in\{0,1\}$,
then the wrapper algorithm will  (1) return a data set that  satisfies the desired power constraints 
and (2) have expected regret no larger than $\mathcal{R}(T)$ with respect to a clipped oracle $\mathcal{O}'$. 
\end{theorem}

\begin{proof}
\textbf{Satisfaction of power constraints:} 
Note that in step 6, we store the transformed action $A'_t$, thus we need to compute $\pi'_\mathcal{A}$. From step 3, we see that we can write the transformed probability $\pi'_\mathcal{A}$ as follows:
\begin{align}
\pi'_\mathcal{A} = p(A'_t=1)=\pi_\mathcal{A}\pi_{\max}+(1-\pi_\mathcal{A})\pi_{\min}.
\end{align}
Since $\pi_{\max}-\pi'_{\mathcal{A}} = (\pi_{\max}-\pi_{\min})(1-\pi_\mathcal{A})\geq0$ and $\pi'_{\mathcal{A}}-\pi_{\min} = (\pi_{\max}-\pi_{\min})\pi_\mathcal{A}\geq 0$, it follows that $\pi'_\mathcal{A}\in[\pi_{\min},\pi_{\max}]$. Thus, by Theorem~\ref{thm:apdx_thm3} ,the power constraint is met. 
\vskip0.2cm
\textbf{Regret with respect to a clipped oracle:} Under the wrapper algorithm, $A_t$ is transformed by the stochastic mapping $G$ and the potential rewards can be written as  $R'_t(a)=R_t(G(a))$ for $a\in\{0,1\}$. And by assumption there is an environment in 
$\Omega$ with these rewards.  Further algorithm $\mathcal{A}$ has regret rate no greater than $\mathcal{R}(T)$  with respect to an oracle $\mathcal{O}$ on the original environment.
The expected reward of an oracle on the new environment is  the same as  the expected reward of the wrapper algorithm applied to the oracle on the original environment, i.e. $\mathbb{E}[R_{t}(\mathcal{O}')]= \mathbb{E}[R_{t}(G(\mathcal{O})]$.
Thus, we can equivalently state that the algorithm resulting from transforming $A_t$ by $G$ has expected regret bound $\mathcal{R}(T)$ 
with respect to a clipped oracle $\mathcal{O}'$.
\end{proof}

For sure, we should ask what collections of environments $\Omega$ are closed under the reward transformation above. In the following, we characterize properties of $\Omega$ satisfying Theorem~\ref{thm:apdx_thm4}. 

\begin{lemma}
\label{lemma:apdx_lemma_1}
For a stochastic contextual bandit, the following environment class has the closure property assumed by Theorem~\ref{thm:apdx_thm4} under the action-transforming operation $G$ - that is, for all environments in  
$\Omega$, the potential rewards $\{R_t(1), R_t(0)\}$ 
transforms to $\{R_t(G(1)),R_t(G(0)\}$, which are still in $\Omega$:

\begin{enumerate}
\item $R_t(a)\leq L$, where $L$ is a constant.
\item $R_t-\mathbb{E}[R_t|A_t,C_t]$ is $\sigma$-sub-Gaussian
\end{enumerate}
\end{lemma}

\begin{proof}
Condition 1. above clearly holds for $R_t(G(a))$ as $G(a)\in\{0,1\}$.   
Now, under the stochastic mapping $G$ on actions,
 the new reward is 
 \begin{align*}
     R_t'=R_t(G(A_t))=&[A_tG(1)+(1-A_t)G(0)]R_t(1)\\ &+[A_t(1-G(1))+(1-A_t)(1-G(0))]R_t(0)
 \end{align*} and the new reward function is given by:
\begin{align*}
\mathbb{E}[R_t'|C_t, A_t]
=&
[A_t\pi_{\max}+(1-A_t)\pi_{\min}]\mathbb{E}[R_t(1)|C_t]\\
& +[A_t(1-\pi_{\max})+(1-A_t)(1-\pi_{\min})]\mathbb{E}[R_t(0)|C_t].
\end{align*}
Since $A_t, G(0), G(1)$ are binary, and the set of sub-Gaussian random variables is closed under finite summation, Condition 2.  still holds albeit with a different constant $\sigma$. 
\end{proof}

Next, we discuss how Lemma~\ref{lemma:apdx_lemma_1} applies to a set of common algorithms. In the derivations of regret bounds for these algorithms, in addition to the environmental assumptions outlined in Lemma~\ref{lemma:apdx_lemma_1}, each derivation makes further assumptions on the environment. We discuss how each set of assumptions is preserved under the closure operation defined by our stochastic  transformation $G$.

\begin{remark}
\label{rmk:apdx_rmk3}
LinUCB~\citep{abbasiyadkori2011},  SupLinUCB~\citep{chu2011contextual} , SupLinREL~\citep{auer2002using}  and TS~\citep{agrawal2012thompson} further assume that the reward takes the form of $\mathbb{E}[R_t(a)|C_{t,a}] = C_{t,a}^\intercal \theta$.  They assume that $\|\theta\|\leq S_1$, $\|C_{t,a}\|\leq S_2$. Thus, under $G$, $$\mathbb{E}[R_t(G(a))|C_t] = \pi^a_{\min}\pi^{1-a}_{\max}C_{t,0}^\intercal \theta+\pi^{1-a}_{\min}\pi^a_{\max}  C_{t,1}^\intercal \theta.$$ $\{\theta,\pi^a_{\min}\pi^{1-a}_{\max}C_{t,0}+\pi^{1-a}_{\min}\pi^a_{\max}  C_{t,1}\}$ are still bounded but possibly  with different constants. 

Differently, $\epsilon$-greedy~\citep{langford2007epoch} assumes the learner is given a set of hypothesis $\mathcal{H}$ where each hypothesis $h$ maps a context $C_t$ to an action $A_t$. The goal is to choose arms to compete with the best hypothesis in $\mathcal{H}$. They  assume that $(C_t,R_t)\sim P$ for some distribution $P$.  Under $G$ this remains true but now with a different distribution $(C'_t,R'_t)\sim P'$ under $G$.  
\citeauthor{langford2007epoch} derived the regret bounds when the hypothesis space is finite $|\mathcal{H}|=m$ with an unknown expected reward gap. Let $R(h)$ be the expected total reward under hypothesis $h$ and $\mathcal{H} = \{h_1,h_2,\dots,h_m\}$. Without loss of generality, they assume $R(h_1)\geq R(h_2)\geq \dots R(h_m)$ and $R(h_1)\geq R(h_2)+\triangle$ where $\triangle$ is the unknown expected reward gap, $\triangle>0$. Now, under $G$, the hypothesis space needs to change accordingly to $\mathcal{H'}$ where each new hypothesis $h'$ may map a context to actions different from before (Each new hypothesis needs to lie within the power-preserving policy class that we derived in Theorem~\ref{thm:apdx_thm3}); however, the hypothesis space size remains the same, $|\mathcal{H'}| = m$. And without loss of generality, we can reorder $R'(h')$ so that $R'(h'_1)\geq R'(h'_2)\geq\dots R'(h_m')$, thus the environment is closed under $G$. 
\end{remark}

\paragraph{Adversarial Case}
There are several ways of specifying adversarial versions of contextual bandits.  Some of those are amenable to our flipping process in the Algorithm described in this section, and others are not.  In particular, the flipping process introduces stochasticity into the perceived rewards, so algorithms that assume deterministic rewards (the context is drawn from some unknown distribution but the reward is picked by an adversary) in the environment will not apply directly (~\cite{auer2016algorithm,auer2002nonstochastic,beygelzimer2011contextual,bubeck2012best,seldin2017improved}).

Other adversarial contextual bandit algorithms are designed for environments with stochastic rewards.  We specifically focus on adversarial contextual bandits where the contexts are chosen by an adversary but the reward is drawn from a fixed (but unknown) conditional distribution given context. We allow the adversary to be aware of the action flipping. The analysis is similar to that of the stochastic bandit. 

Since the contexts are assigned by the adversary deterministically, we denote the context at time step $t$ of action $a$ as $c_{t,a}$. The rewards are stochastic and we denote the potential rewards as $\{R_t(0),R_t(1)\}$ and the reward function as $\mathbb{E}[R_t|A_t,c_t]$
\begin{lemma}
Given an adversarial contextual bandit, the context $c_{t,a}\in \Omega$ is assigned by an adversary. Assume the adversary has the knowledge of the stochastic mapping $G$. There are two sufficient conditions for Theorem~\ref{thm:wrapper} to hold. First, the context $c_{t,a}$ is allowed to evolve arbitrarily. Second,
the stochastic rewards $R_t$ belongs to one those described in Corollary~\ref{lemma:apdx_lemma_1}. 
\end{lemma}
\begin{proof}
With the knowledge of the stochastic mapping $G$, the adversary may generate a new assignment of contexts $c'_{t,a}$ different than the one generated without $G$. Since the context can evolve arbitrarily, the new assignment $c'_{t,a}$ is still in $\Omega$. And by proof in Lemma~\ref{lemma:apdx_lemma_1}, the potential rewards is closed under transformation $G$. Thus, the environment is closed under $G$. 
\end{proof}

\begin{remark}
SupLinUCB~\citep{chu2011contextual} and SupLinREL~\cite{auer2002using}, which are analyzed in Remark~\ref{rmk:apdx_rmk3}, allow the context vector to be chosen by an oblivious adversary (the adversary is not adaptive) and don't make assumptions on how the contexts evolve. In Remark~\ref{rmk:apdx_rmk3}, we already show that under $G$, the new reward function of both algorithms, $\mathbb{E}[R_t(G(a)|c_t]$, is still in the environment class. Therefore, in the adversarial scenario, the environment class is still closed under $G$.
\end{remark}

\subsection{Action Flipping Wrapper Algorithm in MDP Setting}
\label{subsec:apdx_mdp}
In this section, we prove that our action flipping strategy can also be applied to an MDP setting since our test statistic allows the features to depend on the full history.
We again drop $n$ for convenience. A MDP $M$ is defined with a set of finite states $\mathcal{S}$ and a set of  finite actions $A$. An environment for an MDP is defined by the initial state distribution $S_0\sim P_0$, the transition probability $P_{s,s'}^a$ and the reward which is a function of current state, action and next state, $R_t=r(S_t,A_t, S_{t+1})$.

Again we use potential outcome notation; this notation is coherent with the standard MDP notation and allows us to make the role of the stochastic transformation, $G$, clear.  
At time $t$, given the current state $S_t$, the algorithm selects the action $a$ and transits to the next state $S_t$ with transition probability $P_{s,s'}^a=P(S_{t+1}=s'|S_t=s,A_t=a)$. 
The observed reward is $R_{t}(A_t)$ and the expected reward given a state-action pair is $\mathbb{E}[R_{t}(A_t)|S_t=s, A_t=a]=\sum_{s'}P_{s,s'}^a r(s,a, s')$.

Recall that the set of environments is denoted by $\Omega$.  At state $S_t$, an algorithm $\mathcal{A}$ maps the history for each user up to time $t$: $H_{t} = (\{S_{j}, A_{j},R_{j}\}_{j=t}^{t-1},S_t)$ to a probability distribution over action space $A$. As before the wrapper algorithm makes the input algorithm $\mathcal{A}$ believe that it is in an environment more stochastic than it truly is (particularly the distribution of $S_{t+1}$ is more stochastic). Intuitively, if algorithm $\mathcal{A}$ is capable of achieving some rate in this more stochastic environment, then it will be optimal with respect to the clipped oracle. 
\begin{corollary}
Given $\pi_{\min},\ \pi_{\max}$ and an MDP algorithm $\mathcal{A}$, 
assume that algorithm $\mathcal{A}$ has an expected regret  $\mathcal{R}(T)$ for any MDP environment in  $\Omega$, 
with respect to an oracle $\mathcal{O}$.
Under stochastic transformation $G$, if there exists an environment  in
$\Omega$ that contains the new transition probability function:
$P_{s,s'}^{'a}=\left(\pi^a_{\min}\pi^{1-a}_{\max}P_{s,s'}^0+\pi^{1-a}_{\min}\pi^a_{\max} P_{s,s'}^1\right)$
then the wrapper algorithm will  (1) return a data set that  satisfies the desired power constraints 
and (2) have expected regret no larger than $\mathcal{R}(T)$ with respect to a clipped oracle $\mathcal{O}'$. 
\end{corollary}
\begin{proof}
The proof of satisfaction of power constraints follows as in Theorem~\ref{thm:apdx_thm4}.

\textbf{Regarding regret:}  
Under the wrapper algorithm, the action $A_t$ is transformed by the stochastic mapping $G$, which only impacts the next state $S_{t+1}$.
The new transition probability function $P^{'a}_{s,s'}$ can be written as $\left(\pi^a_{\min}\pi^{1-a}_{\max}P_{s,s'}^0+\pi^{1-a}_{\min}\pi^a_{\max} P_{s,s'}^1\right)$.  And by assumption there is an environment in 
$\Omega$ with this probability transition function.  
Recall that the reward is a deterministic function of the current state, the action and the next state. Further recall that $\mathcal{A}$ has regret rate no greater than $\mathcal{R}(T)$  with respect to an oracle $\mathcal{O}$ on the original environment.
Thus the expected reward of an oracle on this environment is  the same as  the expected reward of the wrapper algorithm applied to the oracle on the original environment, i.e. $ \mathbb{E}[R_{t}(G(\mathcal{O}))|S_t,A_t]=\mathbb{E}[R_{t}(\mathcal{O}')|S_t,A_t]$.
Thus, we can equivalently state that the algorithm resulting from transforming $A_t$ by $G$ has expected regret bound $\mathcal{R}(T)$ 
with respect to a clipped oracle $\mathcal{O}'$.
\end{proof}

\section{Descriptions of Algorithms}
\label{sec:apdx_algo}
Below, we provide pseudocode of all the algorithms we used for reference. All the algorithms listed below is for each user $n$ and we drop subscript $n$ for simplicity.
\subsection{Fixed Randomization with $\pi=0.5$}
\begin{algorithm}[H]
\caption{Fixed Randomization with $\pi = 0.5$}
\label{alg:apdx_baseline}
\begin{algorithmic}[1]
  \scriptsize
  \FOR{$t=1,2,\cdots,T$}
  \STATE $A_{t}\sim\texttt{Bern}(0.5)$
  \STATE Observe $R_{t}$
  \ENDFOR
\end{algorithmic}
\end{algorithm}
\subsection{ACTS}
\begin{algorithm}[H]
\caption{Clipped ACTS(Action Centered Thompson Sampling)}
\label{alg:apdx_acts}
\begin{algorithmic}[1]
  \STATE Input: $\sigma^2,\ \pi_{\min},\ \pi_{\max}$
  \STATE $b=0,V=I,\hat\delta = V^{-1}b, \hat\Sigma=\sigma^2V^{-1}$
  \FOR{$t=1,2,\cdots,T$}
    \STATE Observe $C_{t}$
  \IF {$1-\phi_{C_{t}^\intercal\hat\delta,C_{t}^\intercal\hat\Sigma C_{t}}(0)<\pi_{\min}$}
  \STATE $\pi_{t}=\pi_{\min}$
  \STATE $A_{t}\sim\texttt{Bern}(\pi_{t})$
  \ELSIF {$1-\phi_{C_{t}^\intercal\hat\delta,C_{t}^\intercal\hat\Sigma C_{t}}(0)>\pi_{\max}$}
  \STATE $\pi_{t}=\pi_{\max}$
  \STATE $A_{t}\sim\texttt{Bern}(\pi_{t})$
  \ELSE 
  \STATE $\tilde\delta\sim\mathcal{N}(C_{t}^\intercal\hat\delta,C_{t}^\intercal\hat\Sigma C_{t})$
  \STATE $A_{t}=\arg\max(0,C_{t}^\intercal\tilde\delta)$  
  \STATE $\pi_{t}=1-\phi_{C_{t}^\intercal\hat\delta,C_{t}^\intercal\hat\Sigma C_{t}}(0)$
  \ENDIF
  \STATE Observe $R_{t}$
  \STATE Update $V = V+(1-\pi_{t})\pi_{t}C_{t}C_{t}^\intercal, b = b+(A_{t}-\pi_{t})R_{t}C_{t},\hat\delta = V^{-1}b$
  \ENDFOR
\end{algorithmic}
\end{algorithm}
\subsection{BOSE}
\begin{algorithm}[H]
\caption{Clipped BOSE (Bandit Orthogonalized Semiparametric Estimation)}
\label{alg:apdx_bose}
\begin{algorithmic}[1]
  \footnotesize
  \STATE Input: $\pi_{\text{min}},\ \pi_{\text{max}},\ \eta$
  \STATE $b=0,V=I,\hat\delta = V^{-1}b$
  \FOR{$t=1,2,\cdots,T$}
    \STATE Observe $C_{t}$
  \IF {$C_{t}^\intercal\hat\delta>\eta C_{t}^\intercal V^{-1}C_{t}$}
  \STATE $\pi_{t}=\pi_{\max}$
  \ELSIF {$-C_{t}^\intercal\hat\delta>\eta C_{t}^\intercal V^{-1}C_{t}$}
  \STATE $\pi_{t}=\pi_{\min}$
  \ELSE 
  \STATE $\pi_{t}=0.5$
  \ENDIF
  \STATE $A_{t}\sim\texttt{Bern}(\pi_{t})$ and observe $R_{t}$
  \STATE Update $V = V+(A_{t}-\pi_{t})^2C_{t}C^\intercal_{t}, b = b+(A_{t}-\pi_{t})R_{t}C_{t},\hat\delta = V^{-1}b$
  \ENDFOR
\end{algorithmic}
\end{algorithm}
\subsection{linUCB}
\begin{algorithm}[H]
\caption{linUCB(linear Upper Confidence Bound)}
\label{alg:apdx_linucb}
\begin{algorithmic}[1]
  \STATE Input: $\pi_{\text{min}},\ \pi_{\text{max}},\ \eta$
  \STATE $b=0,V=I,\hat\theta = V^{-1}b$
  \FOR{$t=1,2,\cdots,T$}
    \STATE Observe $C_{t}(0),\ C_{t}(1)$
    \STATE For $a\in\{0,1\}$, compute $\mathcal{L}_{t,a}=C_{t}(a)\hat\theta+\eta\sqrt{C^\intercal_{t}(a)V^{-1}C_{t}(a)}$
    \STATE $A^*_t=\arg\max_a \mathcal{L}_{t,a}$
    \STATE $\pi_{t} =\pi_{\text{min}}^{1-A^*_t}\pi_{\text{max}}^{A^*_t}$
  \STATE $A_{t}\sim\texttt{Bern}(\pi_{t})$ and observe $R_{t}$
  \STATE Update $V = V+C_{t}C^\intercal_{t}, b = b+C_{t}R_{t},\hat\theta = V^{-1}b$
  \ENDFOR
\end{algorithmic}
\end{algorithm}

\section{Environments}
\label{sec:apdx_environments}
In this section, we describe the details of the simulated environments we used in the experiment. Recall that the reward function is defined as 
\[\mathbb{E}[R_{nt}|A_{nt},C_{nt}]=\gamma_{nt}+(A_{nt}-\pi_{nt})Z^{\intercal}_{nt}\delta_0\]
The marginal reward $\gamma_{nt}$ is approximated as $B^\intercal_{nt}\gamma_0$. To construct an environment, we need to specify the feature vectors $Z_{nt},\ B_{nt}$ and the vectors $\delta_0,\ \gamma_0$. We also need to specify a noise model for $\tilde{\epsilon}_{nt}$.
\subsection{Mobile Health Simulator}
\label{subsec:apdx_mobile_environments}
The mobile health simulator, which mimics the data generation process of a mobile application to increase users' physical activities, was originally developed in ~\citep{liao2015micro}. In this environment, the effect changes over time but is still independent across days. The noise terms are correlated and follows Gaussian AR(1) process.
The response to the binary action $A_t \in \{0,1\}$ is $R_t$, which is interpreted as $\sqrt{\text{Step count on day }t}$.  
\begin{eqnarray*}
  R_{nt} &=&  A_{nt} Z_{nt}^T \delta + \alpha(t) + \frac{\sigma(t)}{\sqrt{2}}\epsilon_{nt} \\
  \epsilon_{nt} &=& \phi\epsilon_{n,t-1}+e_{nt}\\
  e_{nt} &\sim& \mathcal{N}(0,1) \\
  \epsilon_0&\sim&\mathcal{N}\left(0,\frac{1}{1-\phi^2}\right)\\
    A_{nt}&\in&\{0,1\}
\end{eqnarray*}
Note $Var(\epsilon_{nt})=\frac{1}{1-\phi^2}$ for all $t$.  One can choose $\phi=1/\sqrt{2}$. The features are
\begin{equation}
  Z_{nt} = [ 1 , \frac{t-1}{45} , \left(\frac{t-1}{45}\right)^2 ]^\intercal
\end{equation}
The $\alpha(t)$ represents the square root of the step-count under no action $A_t=0$.  Let $\alpha(t)$ vary linearly from 125 at $t=0$ to 50 at $t=T$.
The $\sigma^2(t)$ is the residual variance in step count.  We set $\sigma(t) = 30$.  
For $\delta_0$, under null hypothesis, $\delta_0=\mathbf{0}$. Under alternate hypothesis, $\delta^{(0)} = 6$. There is no effect at $T=90$ and peak effect at $T=21$. By solving the system, we have $\delta^\intercal_0=[6.00\ ,-2.48\ ,-2.79],\ \bar\delta=\sum_{t=1}^T Z^\intercal_{nt} \delta_0\approx 1.6$. 

To construct a correct working model for the marginal reward, we let
\begin{eqnarray*}
B^\intercal_{nt}&=& [\alpha(t),\pi_{nt}Z^\intercal_{nt}]\\
\gamma_0&=&[1,\delta^\intercal_0]
\end{eqnarray*}
\subsection{Environmental Set-up for Type 1 error Experiment}
For all environments, to verify  Type 1 error is recovered, during simulation, we set $\delta_0=\mathbf{0}$ where $\mathbf{0}$ is a zero vector. When solving for $\pi_{\text{min}},\ \pi_{\text{max}}$, we used $\delta_0$ values specified in the above sections. 
\subsection{Semi-parametric Contextual Bandit(SCB)}
\label{subsec:apdx_lin_cb_environments}
In this environment, for each user $n$, at each round, a feature $Z_{nt}$ is independently drawn from a sphere with norm $0.4$. Additionally:
\begin{eqnarray*}
R_{nt}&=&f_{nt}+A_{nt}Z^\intercal_{nt}\delta_0+\epsilon_{nt}\\
  \delta^\intercal_0&=&[0.382,-0.100,0.065] \quad  (\|\delta_0\|=0.4)\\
  f_{nt}&=& \frac{1}{900}t-0.05\\  
  \epsilon_{nt}&\sim&\mathcal{N}(0,\sigma^2)\quad\text{where}\quad\sigma^2=0.25\\ 
\end{eqnarray*}
We can see that $\mathbb{E}[R_{nt}(0)|H_{nt}]=f_{nt}$. According to Section~\ref{apdx:pf_reward_fcn}, the marginal reward $\gamma_{nt}=f_{nt}+\pi_{nt}Z^\intercal_{nt}\delta_0$. To construct a correct working model for the marginal reward, we let 
\begin{eqnarray*}
B^\intercal_{nt}&=& [t,1,\pi_{nt}Z^\intercal_{nt}]\\
\gamma^\intercal_0&=&[\frac{1}{900},-0.05,\delta^\intercal_0]
\end{eqnarray*}
\subsection{Adversarial Semi-parametric Contextual Bandit(ASCB)}
\label{subsec:apdx_cb_environments}
The adversarial semi-parametric contextual bandit is similar to SCB except that in each round, $\gamma_{nt}$ is chosen by an adaptive adversary. We specifically used the adversary ($f_{nt}$ below) introduced in ~\citep{abbasiyadkori2018}. The environment is defined as follows:
\begin{eqnarray*} 
R_{nt}&=&f_{nt}+A_{nt}Z^\intercal_{nt}\delta_0+\epsilon_{nt}\\
  Z^\intercal_{nt}&=&[-0.5,0.3\cdot (-1)^t,(t/100)^2]\\
  \delta^\intercal_0&=&[0.2,0.2,0.2]\\
  f_{nt}&=&-\max(0,A_{nt}Z_{nt}^\intercal\delta)\\
  \epsilon_{nt}&\sim&\mathcal{N}(0,\sigma^2)\quad\text{where}\quad\sigma^2=0.25\\
\end{eqnarray*}
Similar to SCB (Section~\ref{subsec:apdx_lin_cb_environments}), $\gamma_{nt}=f_{nt}+\pi_{nt}Z^\intercal_{nt}\delta_0$. We let 
\begin{eqnarray*}
B^\intercal_{nt}&=& [-\max(0,A_{nt}Z_{nt}^\intercal\delta),\pi_{nt}Z^\intercal_{nt}]\\
\gamma_0&=&[1,\delta^\intercal_0]
\end{eqnarray*}

\subsection{Environmental Set-up for Robustness Test}
\label{subsec:apdx_robust_env}
\textit{Robustness test of mis-estimated treatment effect.}
To study the impact of the estimated effect size, we tested two different types of mis-estimation: underestimation and overestimation of the average treatment effect. For the experiment purpose,
the guessed size of each dimension $d$ is set as $\delta^{(d)}_{est} = \delta^{(d)}/1.1$ (underestimation) and $\delta^{(d)}_{est}= \delta^{(d)}\times1.1$ (overestimation) while the effect size of the simulation environment remains as $\delta_0=\delta$. 

\textit{Robustness test of mis-estimated noise variance.}
To study the impact of the estimated noise variance size, we tested two different types of noise mis-estimation for SCB and ASCB: underestimation and overestimation of the environment noise. For the experiment purpose,
the guessed size of the noise variance is set as $\sigma^2_{est} = \sigma^2_{est}/1.2$ (underestimation) and $\sigma^2_{est} = \sigma^2_{est}\times 1.2$ (overestimation) while the noise variance of the simulation environment remains as $\sigma^2$ specified above. For mobile health, we mimic the data pattern that during the weekends, the users' behaviors are more stochastic due to less motivation.
Specifically, we let the noise variance of the weekend be $1.5^2$ times larger than that of the weekdays. The estimated noise variance is calculated using the average variance over time $\sigma^2_{est}=\frac{1}{T}\sum_{t=1}^T \sigma^2_t$.

\textit{Robustness test of mis-specified marginal reward model.}
To test the robustness of the power against the working model of the marginal reward, for all environments, we used a bad approximation where $B_{nt}=1$. 

\textit{Robustness test of mis-specified treatment effect model.}
We conducted two types of environments to demonstrate the robustness of the power against the treatment effect model mis-specification. 
(1) In the first experiment, we consider the effect where the constructed model lies within a subspace of the true model. We suppose that the experts consider the last feature in $Z_{nt}$ as irrelevant and drops it during the experiment. (2) In the second experiment, we consider the situation where the true treatment effect model is a nonlinear function of $Z_{nt}$. Specifically, the true treatment effect model is defines the same as in ASCB (Section~\ref{subsec:apdx_cb_environments}) and we approximated it with $R_{nt}=-0.15+0.003t+\epsilon_{nt}$.
\section{Experiment Settings}
For all environments, we use $N=20$ subjects and $T=90$ trajectory length. We ran $1,000$ simulations in total.
In the regret minimization algorithm, $C_{nt}$ is set as $Z_{nt}$. 
\subsection{Identifying optimal hyperparameters}
\label{apdx:subsec_hyperparameter}
For all algorithms, the hyperparameters are chose by maximizing the average return over $1,000$ individuals.
The prior of the ACTS algorithm is set as $\mathcal{N}(0,\sigma^2_0)$ and $\sigma^2_0$ is chosen between $[0.05,0.5]$ for SCB and ASCB, and between $[50,150]$ for the mobile health simulator.
The parameter $\eta$ of BOSE is chosen between $[0.1,2.0]$ for SCB and ASCB, and between $[10,150]$ for the mobile health simulator. The hyperparameter $\eta$ of linUCB is chosen between $[0.01,0.25]$ for SCB and ASCB, and between $[10,100]$ for the mobile health simulator. (Note: in reality, we would not be able to repeatedly run experiments of 1000 individuals to find the optimal hyperparameters; we do this to give the baseline versions of the algorithms their best chance for success.) The optimal hyperparameters, that is, those that minimize empirical regret, are listed below:
\begin{table}[htbp]
  \centering 
  \caption{Optimal hyperparameter chosen for a given pair of an algorithm and an environment}  
  \label{table:hyper}  
  \begin{tabular}{|c|c|c|c|}  
    \hline  
    &SCB&ASCB&Mobile Health Simulator\\
    \hline
    ACTS($\sigma^2_0$)&$0.15$&$0.05$&$60$\\
    \hline
    BOSE($\eta$)&$0.2$&$0.2$&$120$\\
    \hline
    linUCB($\eta$)&$0.03$&$0.02$&$95$\\
    \hline
    \end{tabular}
\end{table}

\subsection{Solved $\pi_\text{min}$, $\pi_\text{max}$}
Table ~\ref{table:solved_pi} lists solved $\pi$ values given a pair of an environment and a guessed effect size as well as given a pair of an environment and a guessed noise variance. We see the smaller in magnitude of $\delta_{est}$ or the larger $\sigma^2$, the closer $\pi_{\min},\pi_{\max}$ are to 0.5, which results in more exploration. The larger in magnitude of $\delta_{est}$ or the smaller $\sigma^2$, the further $\pi_{\min},\pi_{\max}$ are from 0.5 exploration, which results in less exploration.
\begin{table}[htbp]
  \centering 
  \caption{Solved $\pi_\text{min}$,$\pi_\text{max}$ given a pair of an environment and a guessed effect size or given a pair of an environment and a guessed noise variance}  
  \label{table:solved_pi}  
  \begin{tabular}{|c|c|c|c|c|c|c|}  
    \hline  
    &\multicolumn{2}{c|}{$\delta_{est}<\delta$}&\multicolumn{2}{c|}{$\delta_{est}=\delta$}&\multicolumn{2}{c|}{$\delta_{est}>\delta$}\\
    \hline
    &$\pi_{\min}$&$\pi_{\max}$&$\pi_{\min}$&$\pi_{\max}$ &$\pi_{\min}$&$\pi_{\max}$\\
    \hline
    SCB&$0.288$&$0.712$&$0.216$&$0.784$&$0.168$&$0.832$\\
    \hline
    ASCB&$0.301$&$0.699$&$0.225$&$0.775$&$0.174$&$0.826$\\
    \hline
    Mobile Health&$0.335$&$0.665$&$0.243$&$0.757$&$0.187$&$0.813$\\
    \hline
    &\multicolumn{2}{c|}{$\sigma^2_{est}<\sigma^2$}&\multicolumn{2}{c|}{$\sigma^2_{est}=\sigma^2$}&\multicolumn{2}{c|}{$\sigma^2_{est}>\sigma^2$}\\
    \hline
    &$\pi_{\min}$&$\pi_{\max}$&$\pi_{\min}$&$\pi_{\max}$ &$\pi_{\min}$&$\pi_{\max}$\\
    \hline
    SCB&$0.170$&$0.830$&$0.216$&$0.784$&$0.284$&$0.716$\\
    \hline
    ASCB&$0.176$&$0.824$&$0.225$&$0.775$&$0.297$&$0.703$\\
    \hline
    &\multicolumn{2}{c|}{$\sigma^2_{est}\neq\sigma^2$}&\multicolumn{2}{c|}{$\sigma^2_{est}=\sigma^2$}&\multicolumn{2}{c|}{}\\
    \hline
    &$\pi_{\min}$&$\pi_{\max}$&$\pi_{\min}$&$\pi_{\max}$ &\multicolumn{2}{c|}{}\\
    \hline
    Mobile Health&$0.433$&$0.567$&$0.243$&$0.757$&\multicolumn{2}{c|}{}\\
    \hline
    \end{tabular}
\end{table}

\section{Additional Results}
\subsection{Results for Additional Benchmark Environments}
\label{subsec:res_benchmark}
In this section, we show that our approaches can be generalized to other settings. We consider one stochastic environment (semiparametric contextual bandit (SCB)) and one adversarial environment (adversarial semiparametric contextual bandit (ASCB)) for benchmark testing. SCB samples $Z_{nt}$ and $\delta_0$ uniformly from a sphere and $\epsilon_{nt}$ are i.i.d.. Our adversarial semiparametric (ASCB) setting is from \citet{Krishnamurthy2018}; it uses a non-parametric component in the reward to corrupt the information the learner receives. Details in Appendix~\ref{subsec:apdx_lin_cb_environments} and~\ref{subsec:apdx_cb_environments}.

\begin{figure}[h]
    \centering
    \subfigure[SCB]{
    \includegraphics[width=0.4\textwidth,valign=t]{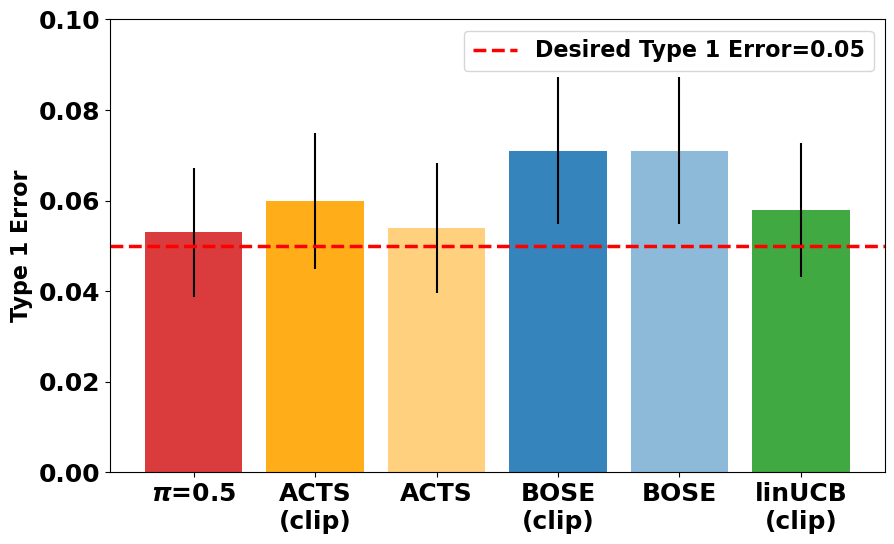}
    }
  \subfigure[ASCB]{
    \includegraphics[width=0.4\textwidth,valign=t]{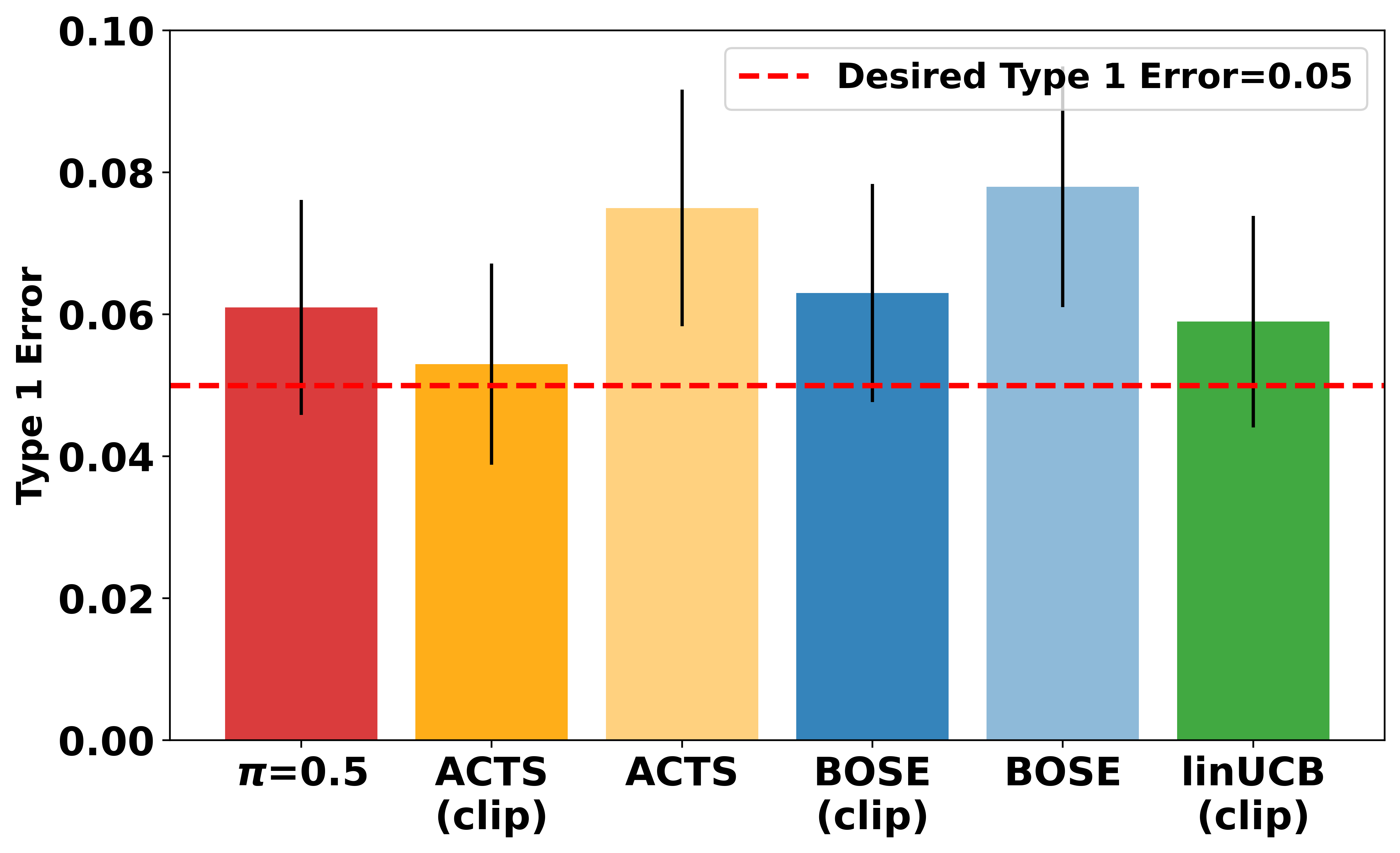}
    }
\caption{Type 1 error with 95\% Confidence Interval:  
We see some Type 1 errors are close
to $\alpha_0=0.05$ while 
some are a little larger than $0.05$, 
due to bias in our estimates of the covariance matrix.}
\label{fig:type_i}
\end{figure}

\emph{Type 1 Error.}  When there is no treatment effect, we see that benchmark environments also suffer from bias in ${\hat\Sigma}_\delta$ and results in Type 1 errors slightly higher than $0.05$ (In Figure~\ref{fig:type_i}, some bars are slightly higher than the red dashed line), suggesting that bias reduction is necessary for future work.  
\begin{figure}[h]
    \centering
    \subfigure[SCB]{
    \includegraphics[width=0.38\textwidth,valign=t]{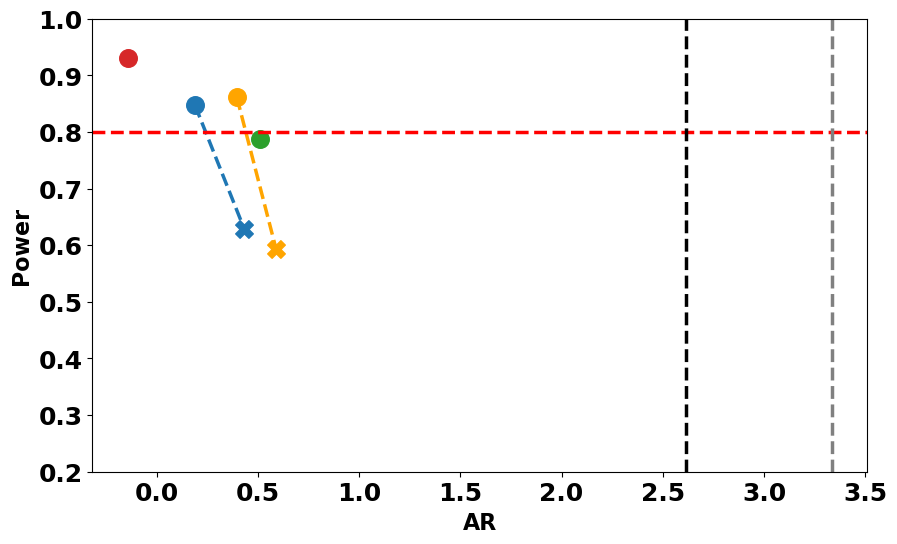}
    }
  \subfigure[ASCB]{
    \includegraphics[width=0.38\textwidth,valign=t]{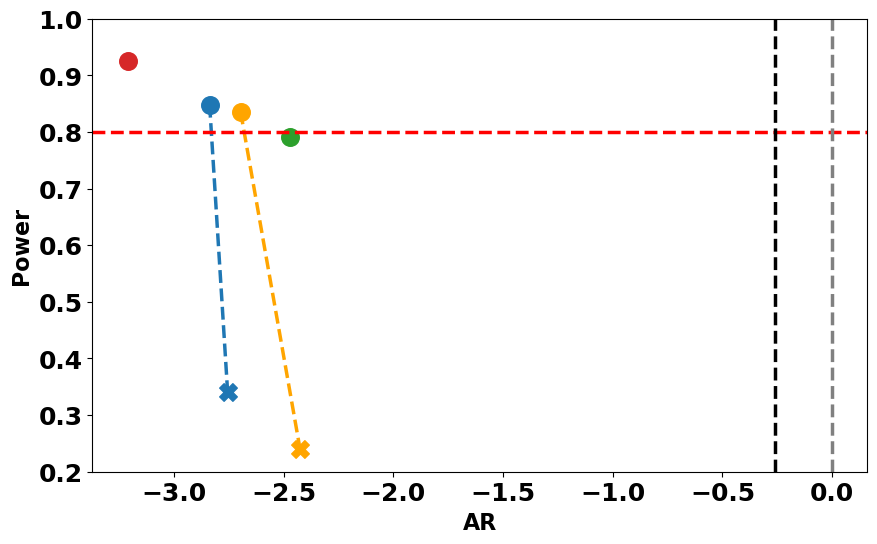}
    }
    \raisebox{3.6cm}{
     \includegraphics[width=0.18\textwidth,valign=t]{mlhc_fig/ar_vs_power_legend.png}}
    \caption{Average Return v.s. Resulting power: $x$-axis denotes average return and $y$-axis denotes the resulting power. Power tends to decrease as average return increases, though clipped linUCB preserves power with a stronger performance than the other baselines.}
     \label{fig:regret_vs_power}
\end{figure}

\emph{Power and Average Return.}
Similar to the mobile health simulator,  we also see the trade-off between the power and the average return in SCB and ASCB. Based on Figure~\ref{fig:regret_vs_power}, in both environments, Fixed Policy ($\pi$=0.5) achieves the highest power.
Comparing the powers of non-clipped algorithms to those of clipped algorithms, 
our clipping scheme achieves the desired power while the non-clipped algorithms fail 
especially in the harder environment (In ASCB, non-clipped algorithms are below the desired power level (even below $0.3$) while clipped ones are above). 
Clipped linUCB achieves the highest average rate while preserving the power guarantee.

\begin{figure}[htbp]
    \centering
    \subfigure[SCB]{
    \includegraphics[width=0.37\textwidth,valign=t]{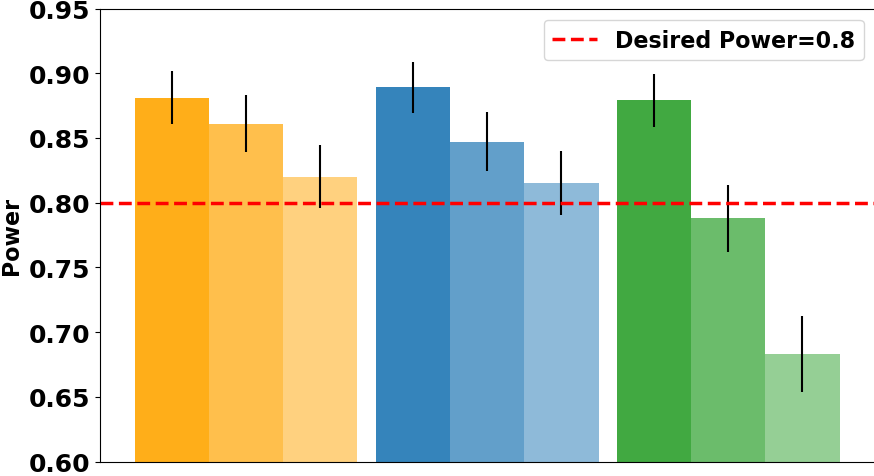}
    }
    \subfigure[ASCB]{
    \includegraphics[width=0.37\textwidth,valign=t]{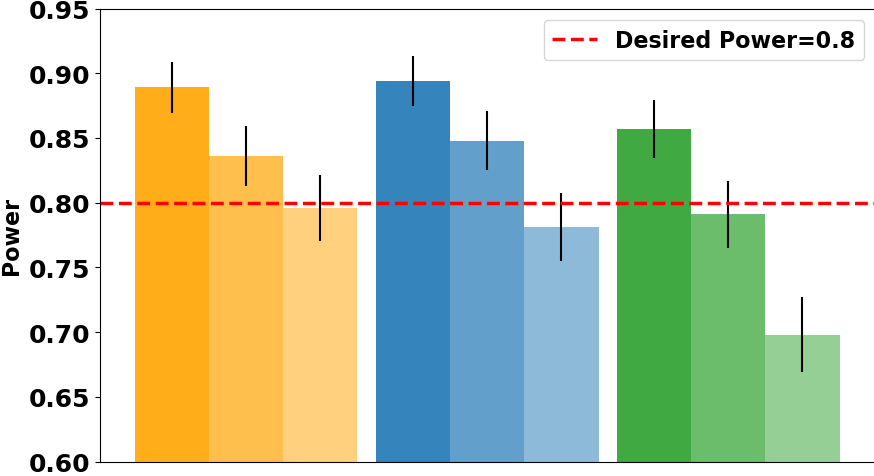}
    }
    \raisebox{3.1cm}{
    \includegraphics[width=0.17\textwidth,valign=t]{mlhc_fig/effect_legend.png}
    }
    \caption{Effect of mis-estimated treatment effect size on power: In general, when $Z_t\delta_{est}<Z_t\delta_0$, power is higher and when $Z_t\delta_{est}>Z_t\delta_0$, power is lower. ACTS and BOSE are more robust to effect mis-specification.}
    \label{fig:effect}
\end{figure}

\emph{Treatment Effect Size Mis-specification.} We consider the effect on the power when our guess of the effect size is overestimated ($Z_{nt}\delta_{est}>Z_{nt}\delta_0$) or underestimated ($Z_{nt}\delta_{est}<Z_{nt}\delta_0$). 
Similar to the mobile health simulator, in both cases, underestimation results in more exploration and higher power and vice versa (Figure~\ref{fig:effect}). Additionally,
linUCB is least robust to mis-estimated effect size as 
it drops the most when the effect size is underestimated with the resulting power still above $0.65$.

\begin{figure}[h]
    \centering
    \subfigure[SCB]{
    \includegraphics[width=0.37\textwidth,valign=t]{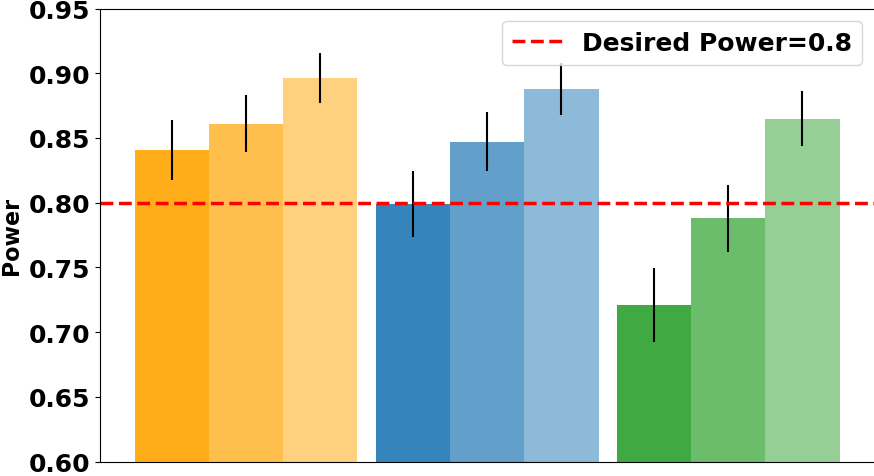}
    }
    \subfigure[ASCB]{
        \includegraphics[width=0.37\textwidth,valign=t]{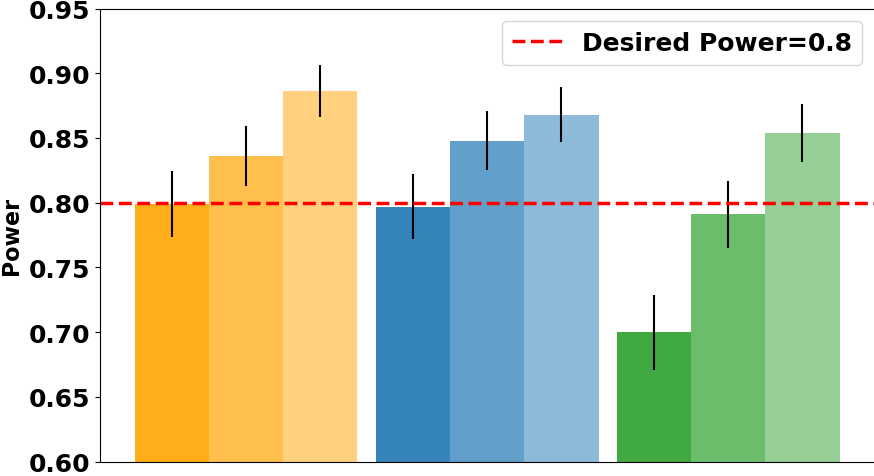}
    }
    \raisebox{3.1cm}{
        \includegraphics[width=0.18\textwidth,valign=t]{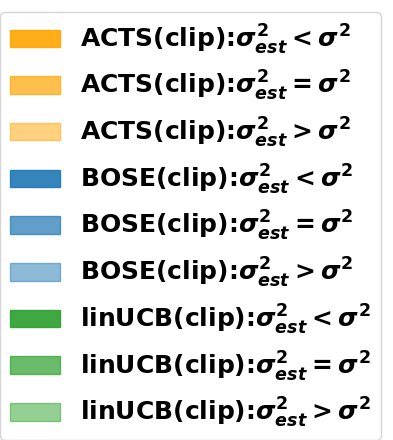}
    }
    \caption{Effect of mis-estimated noise model on power: In general, when $\sigma_{est}>\sigma$, power is higher and when $\sigma_{est}<\sigma$, power is lower. ACTS and BOSE are more robust to noise mis-estimation.}
    \label{fig:noise}
\end{figure}

\emph{Noise Model Mis-specification.} For SCB and ASCB, we test our approaches when the noise variance is overestimated ($\sigma^2_{est}> \sigma^2$) or underestimated ($\sigma^2_{est}< \sigma^2$). We show that overestimated noise variance results in more exploration because more information is needed in a noisy environment, and thus higher power, while underestimation results in less exploration and lower power, with worst case above $0.7$. The results are consistent to our discussion of Theorem~\ref{thm:theorem2}.

\begin{figure}[h]
    \centering
    \subfigure[SCB]{
    \includegraphics[width=0.37\textwidth,valign=t]{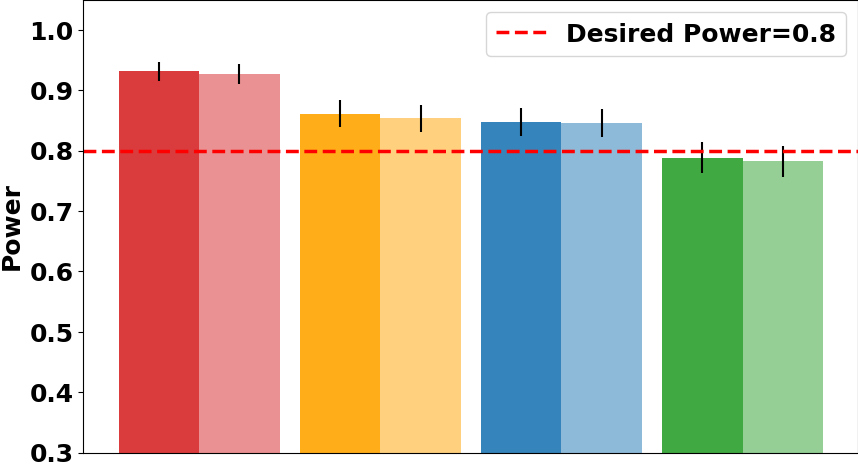}
    }
    ~
    \subfigure[ASCB]{
    \includegraphics[width=0.37\textwidth,valign=t]{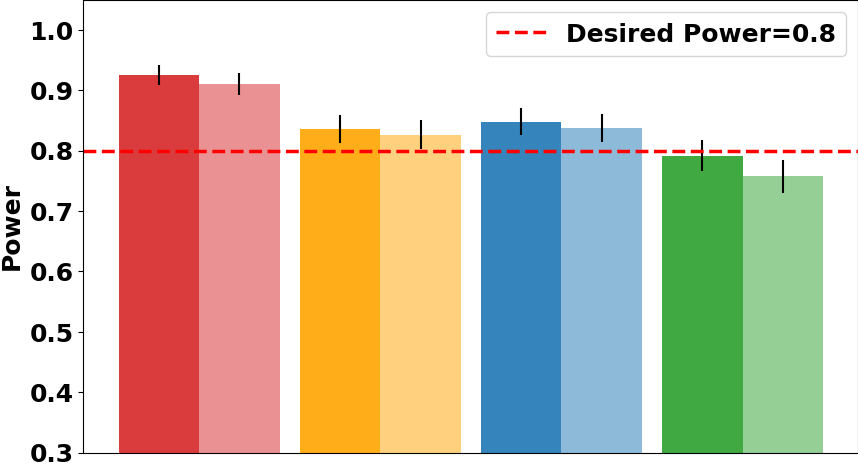}
    }
    \raisebox{3.1cm}{
    \includegraphics[width=0.17\textwidth,valign=t]{mlhc_fig/approx_model_legend.png}
    }
    \caption{Effect of mis-specified marginal reward model on power: Powers is robust to reward model mis-specification in SCB and ASCB where the bar heights are similar.}
    \label{fig:marginal_reward_approx}
\end{figure}

\begin{adjustbox}{minipage=0.82\paperwidth,margin=0pt,center}
    \centering
    \subfigure[SCB]{
    \includegraphics[width=0.2\paperwidth,valign=t]{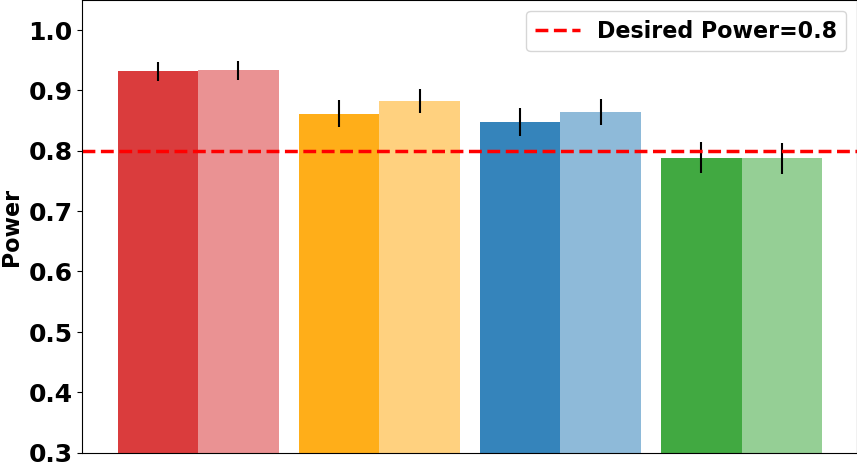}
    \label{fig:scb_treatmemt_effect_model_approx}
    }
    \subfigure[ASCB]{
    \includegraphics[width=0.2\paperwidth,valign=t]{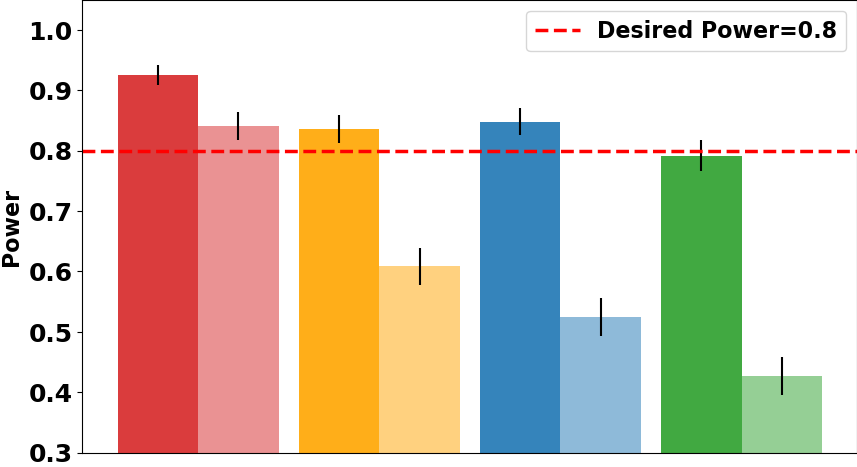}
    \label{fig:ascb_treatmemt_effect_model_approx}
    }
    \subfigure[Non-linear]{
    \includegraphics[width=0.2\paperwidth,valign=t]{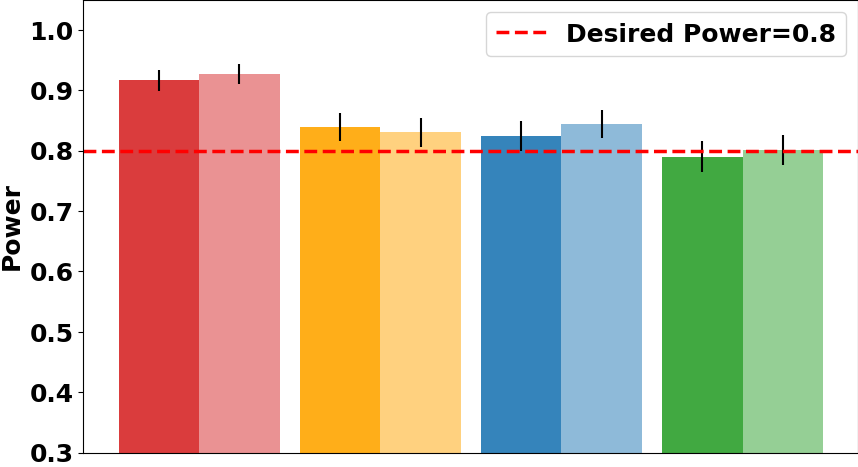}
        \label{fig:nonlinear_treatmemt_effect_model_approx}
    }
    \raisebox{2.45cm}{
    \includegraphics[width=0.12\paperwidth,valign=t]{mlhc_fig/effect_model_legend.png}}
    \captionof{figure}{Effect of mis-specified treatment effect model on power: Excluding a key feature can cause the power to decrease significantly with the worst case above $0.4$. When the feature dimension is correct but the features are incorrect, algorithms are still robust in terms of power with inflated resulting power.}
    \label{fig:treatment_effect_model_approx}
\end{adjustbox}

\emph{Marginal Reward Model Mis-specification.}
For both environments, we use $B_{nt}=1$ as a bad approximation of the marginal reward structure. The resulting powers are similar to those of correctly specified models (Figure~\ref{fig:marginal_reward_approx}). Thus, our methods are robust to marginal reward mis-specification in various settings. 

\emph{Treatment Effect Model Mis-specification.}  For SCB and ASCB, we consider the case where the constructed feature space is smaller than the true feature space. Excluding a key feature can have a big effect in a challenging environment: In Figure~\ref{fig:ascb_treatmemt_effect_model_approx}, for linUCB, power drops to around $0.4$.  We also consider the situation where the true treatment effect, which is a nonlinear function of features $Z_{nt}$, is approximated by a linear function. A different environment is built for this experiment (Appendix Section~\ref{sec:apdx_environments}). In Figure~\ref{fig:nonlinear_treatmemt_effect_model_approx}, we see that all the powers are similar to those when the model is correctly specified.

Based on the results of the robustness experiments, we see that although clipped
linUCB performs the best in term of the average return (Figure~\ref{fig:regret_vs_power}), it is the
least robust in terms of various model mis-specifications (Figure~\ref{fig:effect},~\ref{fig:noise},~\ref{fig:marginal_reward_approx} and~\ref{fig:treatment_effect_model_approx})).

\emph{Regrets with respect to the Clipped Oracle.} In both environments, the regret of clipped algorithms with respect to a clipped oracle is on the same scale as the regret of non-clipped algorithms with respect to a non-clipped oracle (Figure~\ref{fig:regret_vs_power}).

\begin{figure}[h]
    \centering
    \subfigure[SCB]{
    \includegraphics[width=0.38\textwidth,valign=t]{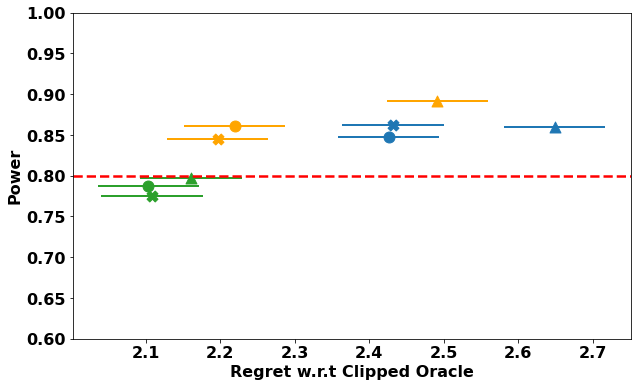}
    }
  \subfigure[ASCB]{
    \includegraphics[width=0.38\textwidth,valign=t]{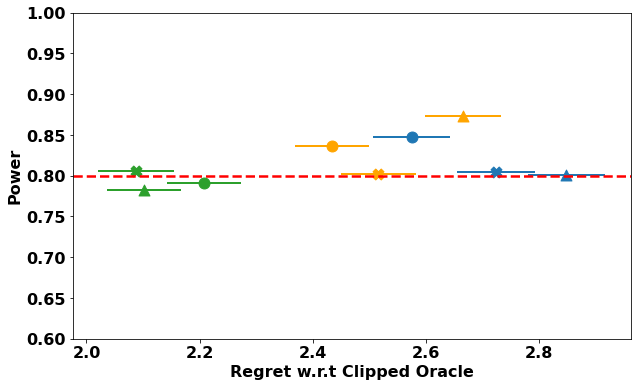}
  }
    \raisebox{3.5cm}{\includegraphics[width=0.18\textwidth,valign=t]{mlhc_fig/wrapper_legend.png}}
    \caption{Regret w.r.t clipped oracle v.s. Resulting power with different wrapper algorithms: $x$-axis is regret with respect to clipped oracle and $y$-axis is the resulting power. In SCB, ASCB, probability clipping and  data dropping works similarly in terms of regret and power. Mostly, action flipping works the worst in terms of regret but results in high power.}
     \label{fig:regret_vs_power_wrapper}
\end{figure}

\emph{Comparison of Wrapper Algorithms.} The power guarantee is preserved for all wrapper algorithms (Figure~\ref{fig:regret_vs_power_wrapper}). In more general settings, action flipping has a clear disadvantage comparing to the other two.
For ACTS, BOSE in environments SCB, ASCB, action flipping results in most power and most regret as we have more exploration due to forced stochasticity and a smaller perceived treatment effect in the modified environment (unlike dropping).  

\subsection{Type 1 Error}
\begin{table}[H]
\centering
 \caption{Type 1 error (Figure~\ref{fig:mb_type_i},~\ref{fig:type_i}) with 2 standard error$(\hat\alpha_0\pm 2\sqrt{\hat\alpha_0 (1-\hat\alpha_0)/S}$ where $S=1000$). We see some Type 1 errors are close to $\alpha_0=0.05$ while some are larger than $0.05$ but not significantly. }
  \label{table:type_i_err}  

  \begin{tabular}{|c|c|c|c|c|c|}  
    \hline  
    \multicolumn{6}{|c|}{ASCB}\\ 
    \hline
    Fix $\pi=0.5$&ACTS&ACTS (clip)&BOSE&BOSE (clip)&linUCB(clip)\\
    \hline  
  $0.060 \pm 0.015$&$0.074 \pm 0.017$&$0.054 \pm 0.014$&$0.078 \pm 0.017$&$0.063 \pm 0.015$&$0.060 \pm 0.015$\\
    \hline  
     \multicolumn{6}{|c|}{SCB}\\ 
    \hline
    Fix $\pi=0.5$&ACTS&ACTS (clip)&BOSE&BOSE (clip)&linUCB(clip)\\
    \hline  
  $0.053 \pm 0.014$&$0.054 \pm 0.014$&$0.06 \pm 0.015$&$0.071 \pm 0.016$&$0.071 \pm 0.016$&$0.058 \pm 0.015$\\
    \hline  
     \multicolumn{6}{|c|}{Mobile Health Simulator}\\ 
    \hline
    Fix $\pi=0.5$&ACTS&ACTS (clip)&BOSE&BOSE (clip)&linUCB(clip)\\
    \hline  
    $0.072 \pm 0.016$&$0.075 \pm 0.017$&$0.061 \pm 0.015$&$0.062 \pm 0.015$&$0.062 \pm 0.015$&$0.072 \pm 0.016$\\
    \hline  
  \end{tabular}
\end{table}

\subsection{Power, Average Return \& Regrets}
\begin{table}[htbp]
\footnotesize
  \centering 
\begin{adjustbox}{minipage=0.8\paperwidth,margin=0pt,center}
  \begin{tabular}{|c|c|c|c|c|c|c|}  
    \hline  
    \multicolumn{7}{|c|}{ASCB}\\ 
    \hline
  & Fix $\pi=0.5$&ACTS&ACTS (clip)&BOSE&BOSE (clip)&linUCB(clip)\\
    \hline  
    power &$0.925 \pm 0.017$&$0.241 \pm 0.027$&$0.836 \pm 0.023$&$0.340 \pm 0.030$&$0.848 \pm 0.023$&$0.791 \pm 0.026$\\
    \hline
    AR&$-3.210 \pm 0.067$&$-2.425 \pm 0.066$&$-2.696 \pm 0.066$&$-2.755 \pm 0.071$&$-2.837 \pm 0.068$&$-2.470 \pm 0.066$\\
    \hline
    $\textit{reg}$&$3.210 \pm 0.067$&$2.425 \pm 0.066$&$-$&$2.755 \pm 0.071$&$-$&$-$\\
    \hline
    $\textit{reg}_c$&$2.949 \pm 0.067$&$-$&$2.434 \pm 0.066$&$-$&$2.575 \pm 0.068$&$2.208 \pm 0.066$\\
    \hline
     \multicolumn{7}{|c|}{SCB}\\ 
     \hline
    power &$0.931 \pm 0.016$&$0.594 \pm 0.031$&$0.861 \pm 0.022$&$0.628 \pm 0.031$&$0.847 \pm 0.023$&$0.788 \pm 0.026$\\ 
     \hline
    AR &$-0.143 \pm 0.067$&$0.590 \pm 0.068$&$0.397 \pm 0.068$&$0.434 \pm 0.070$&$0.190 \pm 0.068$&$0.513 \pm 0.068$\\ 
    \hline
    $\textit{reg}$&$3.479 \pm 0.067$&$2.747 \pm 0.068$&$-$&$2.903 \pm 0.070$&$-$&$-$\\ 
 \hline
$\textit{reg}_c$&$2.759 \pm 0.067$&$-$&$2.219 \pm 0.068$&$-$&$2.426 \pm 0.068$&$2.104 \pm 0.068$\\ 
\hline
     \multicolumn{7}{|c|}{Mobile Health Simulator}\\ 
    \hline
    power &$0.911 \pm 0.018$&$0.390 \pm 0.031$&$0.789 \pm 0.026$&$0.667 \pm 0.030$&$0.901 \pm 0.019$&$0.797 \pm 0.025$\\ 
    \hline
    AR$(\times 10^3)$&$8.089 \pm 0.010$&$8.271 \pm 0.009$&$8.204 \pm 0.010$&$8.109 \pm 0.010$&$8.094 \pm 0.010$&$8.201 \pm 0.010$\\ 
    \hline
    $\textit{reg}(\times 10^3)$&$0.223 \pm 0.010$&$0.041 \pm 0.009$&$-$&$0.204 \pm 0.010$&$-$&$-$\\ 
    \hline
    $\textit{reg}_c(\times 10^3)$&$0.117 \pm 0.010$&$-$&$0.003 \pm 0.010$&$-$&$0.112 \pm 0.010$&$0.005 \pm 0.010$\\  
    \hline
  \end{tabular}
  \end{adjustbox}
    \caption{Resulting power, average return (AR), the regret with respect to the standard oracle ($\textit{reg}$) and the regret with respect to the clipped oracle ($\textit{reg}_c$) with 2 standard error (Figure~\ref{fig:mb_regret_vs_power},~\ref{fig:regret_vs_power}). With probability clipping, the correct power $\beta_0=0.80$ is recovered while without clipping, sufficient power is not guaranteed. There is a trade-off between the average return and the resulting power. The regrets of the clipped algorithms converge as expected with respect to the clipped oracle.}  
     \label{table:power}  

\end{table}

\subsection{Robustness Analysis}
In this section, we list the resulting power of the robustness experiments against various model mis-specifications in tables.

\begin{table}[htbp] 
  \centering 
  \scriptsize
  \begin{adjustbox}{minipage=0.97\paperwidth,margin=0pt,center}
  \begin{tabular}{|c|c|c|c|c|c|c|c|c|c|}  
    \hline  
    &\multicolumn{3}{|c|}{ASCB}&\multicolumn{3}{|c|}{SCB}&\multicolumn{3}{|c|}{Mobile Health}\\ 
    \hline
  &$\delta_{est}<\delta$&$\delta_{est}=\delta$&$\delta_{est}>\delta$&$\delta_{est}<\delta$&$\delta_{est}=\delta$&$\delta_{est}>\delta$&$\delta_{est}<\delta$&$\delta_{est}=\delta$&$\delta_{est}>\delta$\\
    \hline  
    ACTS&$0.889 \pm 0.020$&$0.836 \pm 0.023$&$0.796 \pm 0.025$&$0.881 \pm 0.020$&$0.861 \pm 0.022$&$0.820 \pm 0.024$&$0.862 \pm 0.022$&$0.789 \pm 0.026$&$0.724 \pm 0.028$\\
    \hline  
    BOSE&$0.894 \pm 0.019$&$0.848 \pm 0.023$&$0.781 \pm 0.026$&$0.889 \pm 0.020$&$0.847 \pm 0.023$&$0.815 \pm 0.025$&$0.918 \pm 0.017$&$0.901 \pm 0.019$&$0.841 \pm 0.023$\\
    \hline  
    linUCB&$0.857 \pm 0.022$&$0.791 \pm 0.026$&$0.698 \pm 0.029$&$0.879 \pm 0.021$&$0.788 \pm 0.026$&$0.683 \pm 0.029$&$0.879 \pm 0.021$&$0.797 \pm 0.025$&$0.726 \pm 0.028$\\
    \hline  
  \end{tabular}
  \end{adjustbox}
  \caption{Resulting power with $2$ standard error with mis-estimated treatment effect size (Figure~\ref{fig:mb_treatment},~\ref{fig:effect}) where $\delta_{est}$ denotes the estimated treatment effect size. In general, the power is lower when $\delta_{est}>\delta$ and higher when $\delta_{est}<\delta$. The power is robust against mis-estimated treatment effect with most powers above $0.7$.}
  \label{table:robustness_treatment_effect}  
\end{table}
\vspace*{-0.5cm}
\begin{table}[htbp]
  \centering 
  \footnotesize
  \begin{adjustbox}{minipage=0.9\paperwidth,margin=0pt,center}
  \begin{tabular}{|c|c|c|c|c|c|c|c|c|}  
    \hline  
    &\multicolumn{3}{|c|}{ASCB}&\multicolumn{3}{|c|}{SCB}&\multicolumn{2}{|c|}{Mobile Health}\\ 
    \hline
  &$\sigma_{est}>\sigma$&$\sigma_{est}=\sigma$&$\sigma_{est}<\sigma$&$\sigma_{est}>\sigma$&$\sigma_{est}=\sigma$&$\sigma_{est}<\sigma$&$\sigma_{est}\neq\sigma$&$\sigma_{est}=\sigma$\\
    \hline  
    ACTS&$0.883 \pm 0.020$&$0.836 \pm 0.022$&$0.799 \pm 0.025$&$0.896 \pm 0.019$&$0.861 \pm 0.022$&$0.841 \pm 0.023$&$0.802 \pm 0.025$&$ 0.789 \pm 0.025$\\
    \hline  
    BOSE&$0.868 \pm 0.021$&$0.843 \pm 0.023$&$0.797 \pm 0.025$&$0.888 \pm 0.020$&$0.844 \pm 0.023$&$0.794 \pm 0.026$&$0.801 \pm 0.025$&$0.901 \pm 0.019$\\
    \hline  
    linUCB&$0.854 \pm 0.022$&$0.793 \pm 0.026$&$0.7 \pm 0.029$&$0.865 \pm 0.022$&$0.788 \pm 0.026$&$0.721 \pm 0.028$&$0.824 \pm 0.025$&$0.793 \pm 0.026$\\
    \hline  
  \end{tabular}
  \end{adjustbox}
    \caption{Resulting power with $2$ standard error with mis-estimated noise (Figure~\ref{fig:mb_noise},~\ref{fig:noise}) variance where $\sigma_{est}$ denotes the estimated noise. In general, the power is lower when $\sigma_{est}<\sigma$ and higher when $\sigma_{est}>\sigma$. The power is robust against noise variance as the worst case is still above $0.7$ (given by linUCB in ASCB). }  
  \label{table:robustness_noise}  
\end{table}

\vspace*{-0.5cm}
\begin{table}[htbp]
  \centering 
     \begin{adjustbox}{minipage=0.8\paperwidth,margin=0pt,center}
  \begin{tabular}{|c|c|c|c|c|c|c|}  
    \hline  
    &\multicolumn{2}{|c|}{ASCB}&\multicolumn{2}{|c|}{SCB}&\multicolumn{2}{|c|}{Mobile Health}\\ 
    \hline
  &True $B_{nt}$&$B_{nt}=1$ &True $B_{nt}$&$B_{nt}=1$ &True $B_{nt}$&$B_{nt}=1$\\
    \hline  
      Fixed $\pi=0.5$&$0.925 \pm 0.017$&$0.910 \pm 0.018$&$0.931 \pm 0.016$&$0.926 \pm 0.017$&$0.911 \pm 0.018$&$0.887 \pm 0.020$\\
     \hline
    ACTS&$0.836 \pm 0.023$&$0.826 \pm 0.024$&$0.861 \pm 0.022$&$0.853 \pm 0.022$&$0.789 \pm 0.026$&$0.792 \pm 0.026$\\
    \hline  
    BOSE&$0.848 \pm 0.023$&$0.837 \pm 0.023$&$0.847 \pm 0.023$&$0.846 \pm 0.023$&$0.901 \pm 0.019$&$0.885 \pm 0.020$\\
    \hline  
    linUCB&$0.791 \pm 0.026$&$0.757 \pm 0.027$&$0.788 \pm 0.026$&$0.782 \pm 0.026$&$0.797 \pm 0.025$&$0.818 \pm 0.024$\\
    \hline  
  \end{tabular}
  \end{adjustbox}
    \caption{Resulting power with $2$ standard error with mis-specified marginal reward model (Figure~\ref{fig:mb_marginal_reward_approx},~\ref{fig:marginal_reward_approx}). Powers is robust to
reward model mis-specification as most resulting powers are close to $0.8$.}  
  \label{table:robustness_marginal_reward_model}  
\end{table}

\vspace*{-0.5cm}
\begin{table}[!htbp]
  \centering 
     \begin{adjustbox}{minipage=0.8\paperwidth,margin=0pt,center}
     \footnotesize
  \begin{tabular}{|c|c|c|c|c|c|c|c|}  
    \hline  
    &\multicolumn{3}{|c|}{ASCB}&\multicolumn{2}{|c|}{SCB}&\multicolumn{2}{|c|}{Mobile Health}\\ 
    \hline
  &True $Z_{nt}$&Drop&Nonlinear&True $Z_{nt}$&Drop&True $Z_{nt}$&Drop\\
    \hline  
      Fixed $\pi=0.5$&$0.925 \pm 0.017$&$0.840 \pm 0.023$&$0.926\pm0.016$&$0.931 \pm 0.016$&$0.933 \pm 0.016$&$0.911 \pm 0.018$&$0.927 \pm 0.016$\\
     \hline
    ACTS&$0.836 \pm 0.023$&$0.608 \pm 0.031$&$0.83 \pm 0.024$&$0.861 \pm 0.022$&$0.882 \pm 0.020$&$0.789 \pm 0.026$&$0.516 \pm 0.032$\\
    \hline  
    BOSE&$0.848 \pm 0.023$&$0.524 \pm 0.032$&$0.847 \pm 0.023$&$0.844 \pm 0.023$&$0.864 \pm 0.022$&$0.901 \pm 0.019$&$0.709 \pm 0.029$\\
    \hline  
    linUCB&$0.791 \pm 0.026$&$0.427 \pm 0.031$&$0.801 \pm 0.025$&$0.788 \pm 0.026$&$0.787 \pm 0.026$&$0.797 \pm 0.025$&$0.398 \pm 0.031$\\
    \hline  
  \end{tabular}
  \end{adjustbox}
    \caption{Resulting power with $2$ standard error with mis-specified treatment effect model (Figure~\ref{fig:mb_treatment_model_approx},~\ref{fig:treatment_effect_model_approx}). We see that the power is robust against both types of treatment effect mis-specification. linUCB is the least robust with the worst resulting power above $0.4$.}  
  \label{table:robustness_treatment_effect_model}  
\end{table}

 \subsection{Comparison of Wrapper Algorithms}
 The full results of action-flipping/ data-dropping/ probability-clipping wrapper algorithms are listed in Table~\ref{table:wrapper_ar}.
  \centering 
  \begin{longtable}{|c|c|c|c|c|} 
    \hline  
    \multicolumn{5}{|c|}{SCB}\\ 
    \hline
    &ACTS&ACTS (flip)&ACTS (drop)&ACTS (clip)\\
    \hline  
    power&$0.594 \pm 0.031$&$0.892 \pm 0.020$&$0.845 \pm 0.023$&$0.860 \pm 0.022$\\\hline
    AR  &$0.590 \pm 0.069$&$0.125 \pm 0.068$&$0.420 \pm 0.067$&$0.390 \pm 0.068$\\ \hline
$\emph{reg}$  &$2.747 \pm 0.069$&-&-&-\\ \hline
$\emph{reg}_c$ &-&$2.492 \pm 0.068$&$2.197 \pm 0.067$&$2.219 \pm 0.068$\\ \hline
    &BOSE&BOSE (flip)&BOSE (drop)&BOSE (clip)\\\hline
  power&$0.628 \pm 0.030$&$0.86 \pm 0.020$&$0.863 \pm 0.022$&$0.848 \pm 0.023$\\\hline
    AR  &$0.434 \pm 0.070$&$-0.033 \pm 0.067$&$0.179 \pm 0.069$&$0.190 \pm 0.068$\\ \hline
$\emph{reg}$  &$2.903 \pm 0.070$&-&-&-\\ \hline
$\emph{reg}_c$ &-&$2.649 \pm 0.067$&$2.437 \pm 0.069$&$2.426 \pm 0.068$\\\hline
    &linUCB&linUCB(flip)&linUCB(drop)&linUCB(clip)\\\hline 
    power& -&$0.797 \pm 0.025$&$0.775 \pm 0.025$&$0.788 \pm 0.026$\\\hline
    AR  &$1.243 \pm 0.069$&$0.455 \pm 0.068$&$0.508 \pm 0.068$&$0.513 \pm 0.068$\\ \hline
$\emph{reg}$  &$2.093 \pm 0.069$&-&-&-\\ \hline
$\emph{reg}_c$ &-&$2.161 \pm 0.068$&$2.109 \pm 0.068$&$2.104 \pm 0.068$\\\hline
     \multicolumn{5}{|c|}{ASCB}\\ 
    \hline
 &ACTS&ACTS (flip)&ACTS (drop)&ACTS (clip)\\
 \hline  
 power& $0.241 \pm 0.027$&$0.873 \pm 0.022$&$0.802 \pm 0.023$&$0.836 \pm 0.023$\\ \hline
 AR  &$-2.425 \pm 0.069$&$-2.947 \pm 0.067$&$-2.778 \pm 0.068$&$-2.696 \pm 0.068$\\ \hline
$\emph{reg}$  &$2.425 \pm 0.069$&-&-&-\\ \hline
$\emph{reg}_c$ &-&$2.666 \pm 0.067$&$2.516 \pm 0.068$&$2.434 \pm 0.068$\\\hline
    &BOSE&BOSE (flip)&BOSE (drop)&BOSE (clip)\\\hline
    power& $0.34 \pm 0.030$&$0.801 \pm 0.023$&$0.804 \pm 0.023$&$0.844 \pm 0.023$\\\hline
     AR  &$-2.755 \pm 0.072$&$-3.110 \pm 0.067$&$-2.985 \pm 0.068$&$-2.837 \pm 0.068$\\ \hline
$\emph{reg}$  &$2.755 \pm 0.072$&-&-&-\\ \hline
$\emph{reg}_c$ &-&$2.848 \pm 0.067$&$2.724 \pm 0.068$&$2.575 \pm 0.068$\\\hline
    &linUCB&linUCB(flip)&linUCB(drop)&linUCB(clip)\\\hline  
    power&-&$0.783 \pm 0.025$&$0.806 \pm 0.026$&$0.791 \pm 0.025$\\\hline
    AR  &$-1.655 \pm 0.066$&$-2.364 \pm 0.066$&$-2.349 \pm 0.067$&$-2.470 \pm 0.066$\\ \hline
$\emph{reg}$  &$1.655 \pm 0.066$&-&-&-\\ \hline
$\emph{reg}_c$ &-&$2.102 \pm 0.066$&$2.087 \pm 0.067$&$2.208 \pm 0.066$\\\hline
  \multicolumn{5}{|c|}{Mobile Health Simulator}\\ \hline
 &ACTS&ACTS (flip)&ACTS (drop)&ACTS (clip)\\
 \hline  
 power& $0.39 \pm 0.031$&$0.819 \pm 0.023$&$0.801 \pm 0.026$&$0.789 \pm 0.025$\\\hline
 AR($\times10^3$)  &$8.271 \pm 0.004$&$8.185 \pm 0.004$&$8.206 \pm 0.004$&$8.204 \pm 0.004$\\ \hline
$\emph{reg}(\times10^3)$  &$0.041 \pm 0.004$&-&-&-\\ \hline
$\emph{reg}_c(\times10^3)$ &-&$0.020 \pm 0.004$&$-0.036 \pm 0.004$&$0.025 \pm 0.004$\\\hline
     &BOSE&BOSE (flip)&BOSE (drop)&BOSE (clip)\\\hline
     power&$0.667 \pm 0.030$&$0.858 \pm 0.022$&$0.856 \pm 0.022$&$0.901 \pm 0.019$\\\hline
      AR($\times10^3$)  &$8.106 \pm 0.004$&$8.100 \pm 0.004$&$8.095 \pm 0.004$&$8.097 \pm 0.004$\\ \hline
$\emph{reg}(\times10^3)$  &$0.203 \pm 0.004$&-&-&-\\ \hline
$\emph{reg}_c(\times10^3)$ &-&$0.105 \pm 0.004$&$0.111 \pm 0.004$&$0.109 \pm 0.004$\\\hline
    &linUCB&linUCB(flip)&linUCB(drop)&linUCB(clip)\\\hline  
    power&-&$0.794 \pm 0.026$&$0.817 \pm 0.024$&$0.793 \pm 0.026$\\\hline
     AR($\times10^3$) &$8.295 \pm 0.004$&$8.189 \pm 0.004$&$8.192 \pm 0.004$&$8.201 \pm 0.004$\\ \hline
$\emph{reg}(\times10^3)$  &$0.017 \pm 0.004$&-&-&-\\ \hline
$\emph{reg}_c(\times10^3)$ &-&$0.016\pm 0.004$&$0.014 \pm 0.004$&$0.006 \pm 0.004$\\ \hline
  \caption{Average Return, $\textit{reg}$, $\textit{reg}_{c}$ with 2 standard errors: All wrapper algorithms achieve good regret rate with slightly different trade-offs given the situation.}  
  \label{table:wrapper_ar}  
  \end{longtable}

\end{document}